%% file: main.tex
\documentclass[twoside,11pt]{article}

\usepackage[preprint]{jmlr2e}

\usepackage{lastpage}

\usepackage{doi}

\input{macros}
\numberwithin{equation}{section}

\jmlrheading{0}{2026}{1-\pageref{LastPage}}{--/--; Revised --/--}{--/--}{00-0000}{Eviatar Bach, Ricardo Baptista, Jochen Br\"ocker, Bohan Chen and Andrew Stuart}

\ShortHeadings{Learning Probabilistic Filters with Strictly Proper Scoring Rules}{Bach, Baptista, Br\"ocker, Chen, and Stuart}
\firstpageno{1}

\begin{document}

\title{Learning Probabilistic Filters with Strictly Proper Scoring Rules}

\author{\name Eviatar Bach \email eviatarbach@protonmail.com \\
       \addr Department of Meteorology, University of Reading\\
       Brian Hoskins Building, Reading, RG6 6ET, UK\\
       \addr Department of Mathematics and Statistics, University of Reading\\
       Pepper Lane, Reading, RG6 6AX, UK\\
       \addr National Centre for Earth Observation\\
       Brian Hoskins Building, University of Reading, Reading, RG6 6ET, UK
       \AND
       \name Ricardo Baptista \email r.baptista@utoronto.ca\\
       \addr Department of Statistical Sciences, University of Toronto\\
       700 University Ave., Toronto, ON M5G 1Z5, Canada
       \AND
       \name Jochen Br\"ocker \email j.broecker@reading.ac.uk\\
       \addr Department of Meteorology, University of Reading\\
       Brian Hoskins Building, Reading, RG6 6ET, UK\\
       \addr Department of Mathematics and Statistics, University of Reading\\
       Pepper Lane, Reading, RG6 6AX, UK
       \AND
       \name Bohan Chen \email bhchen@caltech.edu\\
       \addr The Computing + Mathematical Sciences Department, California Institute of Technology\\
       1200 E. California Blvd., Pasadena, CA 91125, USA
       \AND
       \name Andrew Stuart \email astuart@caltech.edu\\
       \addr The Computing + Mathematical Sciences Department, California Institute of Technology\\
       1200 E. California Blvd., Pasadena, CA 91125, USA}

\editor{}

\maketitle

\begin{abstract}
Bayesian filtering of partially and noisily observed dynamical systems seeks to infer the evolving conditional distribution of the state of a dynamical system given observations, in an online fashion. This Bayesian filtering distribution is rarely available as a supervised learning target. However, one can often use the forecast model to generate synthetic trajectories, with corresponding synthetic observations. We introduce the proper scoring ensemble filter (PSEF), an ensemble data assimilation method trained using only synthetic trajectories. The analysis step is represented as a permutation-equivariant, transformer-based map. Training is based on strictly proper scoring rules---with the energy score used in our implementation---so that probabilistic accuracy is rewarded over the whole probability distribution. Under a realizability assumption, the population mean-field objective is minimized by the true Bayesian filtering distribution. Our methodology allows the same learned parameters to be shared between different ensemble sizes, subject to an ensemble-dependent fine-tuning. Numerical experiments show that the learned filter accurately approximates challenging filtering distributions, including highly non-Gaussian and multi-modal posteriors, and achieves stronger performance in data assimilation tasks than classical methods or learning-based methods with mean-squared-error objectives.
\end{abstract}

\begin{keywords}
data assimilation, Bayesian filtering distribution, ensemble filter, proper scoring rules, amortized Bayesian inference
\end{keywords}

\input{sections/introduction}

\input{sections/background}

\input{sections/learning_the_true_filter}

\input{sections/experiments}

\input{sections/conclusions}

\acks{AMS is supported by a Department of Defense (DoD) Vannevar Bush Faculty Fellowship (award N00014-22-1-2790),
which also supported BC and RB; AMS was also supported by NSF award AGS1835860 and by the Resnick Sustainability Institute; RB was also supported by a Von Karman Instructorship at Caltech. The work of EB and JB was partially funded by the AUSPICE project, UKRI2165. EB and JB also received support from the Centre for the Mathematics of Planet Earth at the University of Reading. We thank Chris Oates for helpful discussions.

\medskip\noindent\textbf{Declaration on the use of generative AI.}
During the preparation of this manuscript, the authors used ChatGPT for editorial suggestions and to improve the clarity, readability, and presentation of the writing. The tool was not used to generate research ideas, theoretical results, proofs, experiments, data, or conclusions. The authors reviewed and edited all AI-assisted output and take full responsibility for the content of the manuscript.
}

\appendix
\input{sections/appendix}

\bibliography{references}

\end{document}

%% file: macros.tex
\usepackage{amsmath}  %
\usepackage{amssymb}  %
\usepackage{bm}       %
\usepackage{graphicx} %
\usepackage{siunitx}  %
\usepackage{float}    %
\usepackage{caption}  %
\usepackage{subcaption} %
\usepackage{enumitem} %
\usepackage{booktabs} %
\usepackage{verbatim} %
\usepackage{multimedia}
\usepackage{xcolor}
\usepackage{etoolbox}
\usepackage{algorithm}
\usepackage{algorithmic}
\usepackage{multirow}
\usepackage{stmaryrd}
\usepackage[size=tiny]{todonotes}
\usepackage{tikz-cd}
\usepackage{thmtools}
\usepackage{thm-restate}
\usepackage{bbm}
\usepackage{xstring}

\usepackage{pifont} %

\DeclareMathOperator{\Law}{Law}

\DeclareMathOperator*{\trace}{tr}

\DeclareMathOperator*{\argmin}{arg\,min}

\renewcommand{\hat}{\widehat}
\renewcommand{\bar}{\overline}

\newcommand{\normal}{\mathsf{N}}

\newcommand{\sA}{\mathsf{A}}
\newcommand{\E}{\mathbb{E}}

\newcommand{\fF}{\mathfrak{F}}

\newcommand{\sS}{\mathsf{S}}

\newcommand{\sT}{\mathsf{T}}
\newcommand{\sF}{\mathsf{F}}

\newcommand{\pp}{\mathfrak{p}}

\newcommand{\fB}{\mathfrak{B}}
\newcommand{\fA}{\mathfrak{A}}
\newcommand{\fT}{\mathfrak{T}}

\newcommand{\sP}{\mathsf{P}}
\newcommand{\sQ}{\mathsf{Q}}
\newcommand{\sB}{\mathsf{B}}

\renewcommand{\epsilon}{\varepsilon}

\newcommand{\rd}{\mathrm{d}}

\newcommand{\vd}{v^{\dagger}}
\newcommand{\yd}{y^{\dagger}}

\newcommand{\cP}{\mathcal{P}}

\newcommand{\cB}{\mathcal{B}}

\newcommand{\bbR}{\mathbb{R}}
\newcommand{\bbP}{\mathbb{P}}

\newcommand{\bbM}{\mathbb{M}}

\newcommand{\bbN}{\mathbb{N}}
\newcommand{\bbE}{\mathbb{E}}
\newcommand{\bbZ}{\mathbb{Z}}

\newcommand{\placeholder}{\mathord{\color{black!33}\bullet}}%

\definecolor{light}{rgb}{0.5, 0.5, 0.5}

\definecolor{mypink1}{rgb}{0.858, 0.188, 0.478}

\definecolor{darkred}{rgb}{0.5,0,0}
\definecolor{darkgreen}{rgb}{0,0.75,0}

\newtheorem{dataassumption}{Data Assumption}
\newtheorem{assumption}{Assumption}

\newcommand{\hv}{\hat{v}}

\newcommand{\hy}{\hat{y}}

\newcommand{\hvn}{\hat{v}^{(n)}}

\newcommand{\hyn}{\hat{y}^{(n)}}

\newcommand{\vn}{v^{(n)}}

\newcommand{\dst}{d_\text{ST}}

\newcommand{\DoublingAngleGrid}[2]{%
    \resizebox{\textwidth}{!}{%
        \begin{tabular}{%
            >{\centering\arraybackslash}m{0.08\textwidth}
            >{\centering\arraybackslash}m{0.08\textwidth}
            >{\centering\arraybackslash}m{0.06\textwidth}
            >{\centering\arraybackslash}m{0.08\textwidth}
            >{\centering\arraybackslash}m{0.08\textwidth}
            >{\centering\arraybackslash}m{0.08\textwidth}
            >{\centering\arraybackslash}m{0.08\textwidth}
            >{\centering\arraybackslash}m{#2}
            >{\centering\arraybackslash}m{0.08\textwidth}
            >{\centering\arraybackslash}m{0.08\textwidth}
            >{\centering\arraybackslash}m{0.08\textwidth}
            >{\centering\arraybackslash}m{0.08\textwidth}}
        \toprule
        \multirow{2}{*}{Stage} & Truth & \multirow{2}{*}{ML $N$} & EtE & {\scriptsize CorrTerms} & EtE & {\scriptsize CorrTerms} & Bench. & \multirow{2}{*}{EnKF} & \multirow{2}{*}{ESRF} & \multirow{2}{*}{IEnKF} & \multirow{2}{*}{SMF} \\
        & BPF & & + \textbf{ES} & + \textbf{ES} & + NL2 & + NL2 & $N$ & & & & \\
        \midrule
        \begin{tabular}{@{}c@{}}Predictive\\ Density\end{tabular}
        & \includegraphics[width=\linewidth, trim=0.5cm 0.5cm 0.5cm 0.5cm, clip]{#1/truth_n10_prior.png}
        & $N$=10
        & \includegraphics[width=\linewidth, trim=0.5cm 0.5cm 0.5cm 0.5cm, clip]{#1/ml_ete_es_n10_prior.png}
        & \includegraphics[width=\linewidth, trim=0.5cm 0.5cm 0.5cm 0.5cm, clip]{#1/ml_corrterms_es_n10_prior.png}
        & \includegraphics[width=\linewidth, trim=0.5cm 0.5cm 0.5cm 0.5cm, clip]{#1/ml_ete_nl2_n10_prior.png}
        & \includegraphics[width=\linewidth, trim=0.5cm 0.5cm 0.5cm 0.5cm, clip]{#1/ml_corrterms_nl2_n10_prior.png}
        & $N$=300
        & \includegraphics[width=\linewidth, trim=0.5cm 0.5cm 0.5cm 0.5cm, clip]{#1/benchmark_enkf_n300_ml_n10_prior.png}
        & \includegraphics[width=\linewidth, trim=0.5cm 0.5cm 0.5cm 0.5cm, clip]{#1/benchmark_esrf_n300_ml_n10_prior.png}
        & \includegraphics[width=\linewidth, trim=0.5cm 0.5cm 0.5cm 0.5cm, clip]{#1/benchmark_ienkf_n300_ml_n10_prior.png}
        & \includegraphics[width=\linewidth, trim=0.5cm 0.5cm 0.5cm 0.5cm, clip]{#1/benchmark_smf_n300_ml_n10_prior.png} \\
        \begin{tabular}{@{}c@{}}\textcolor{red!45}{Filtering}\\ Density\end{tabular}
        & \includegraphics[width=\linewidth, trim=0.5cm 0.5cm 0.5cm 0.5cm, clip]{#1/truth_n10_post.png}
        & $N$=10
        & \includegraphics[width=\linewidth, trim=0.5cm 0.5cm 0.5cm 0.5cm, clip]{#1/ml_ete_es_n10_post.png}
        & \includegraphics[width=\linewidth, trim=0.5cm 0.5cm 0.5cm 0.5cm, clip]{#1/ml_corrterms_es_n10_post.png}
        & \includegraphics[width=\linewidth, trim=0.5cm 0.5cm 0.5cm 0.5cm, clip]{#1/ml_ete_nl2_n10_post.png}
        & \includegraphics[width=\linewidth, trim=0.5cm 0.5cm 0.5cm 0.5cm, clip]{#1/ml_corrterms_nl2_n10_post.png}
        & $N$=300
        & \includegraphics[width=\linewidth, trim=0.5cm 0.5cm 0.5cm 0.5cm, clip]{#1/benchmark_enkf_n300_ml_n10_post.png}
        & \includegraphics[width=\linewidth, trim=0.5cm 0.5cm 0.5cm 0.5cm, clip]{#1/benchmark_esrf_n300_ml_n10_post.png}
        & \includegraphics[width=\linewidth, trim=0.5cm 0.5cm 0.5cm 0.5cm, clip]{#1/benchmark_ienkf_n300_ml_n10_post.png}
        & \includegraphics[width=\linewidth, trim=0.5cm 0.5cm 0.5cm 0.5cm, clip]{#1/benchmark_smf_n300_ml_n10_post.png} \\
        \addlinespace
        \begin{tabular}{@{}c@{}}Predictive\\ Density\end{tabular}
        & \includegraphics[width=\linewidth, trim=0.5cm 0.5cm 0.5cm 0.5cm, clip]{#1/truth_n30_prior.png}
        & \begin{tabular}{@{}c@{}}$N$=30\\ Training\end{tabular}
        & \includegraphics[width=\linewidth, trim=0.5cm 0.5cm 0.5cm 0.5cm, clip]{#1/ml_ete_es_n30_prior.png}
        & \includegraphics[width=\linewidth, trim=0.5cm 0.5cm 0.5cm 0.5cm, clip]{#1/ml_corrterms_es_n30_prior.png}
        & \includegraphics[width=\linewidth, trim=0.5cm 0.5cm 0.5cm 0.5cm, clip]{#1/ml_ete_nl2_n30_prior.png}
        & \includegraphics[width=\linewidth, trim=0.5cm 0.5cm 0.5cm 0.5cm, clip]{#1/ml_corrterms_nl2_n30_prior.png}
        & $N$=1000
        & \includegraphics[width=\linewidth, trim=0.5cm 0.5cm 0.5cm 0.5cm, clip]{#1/benchmark_enkf_n1000_ml_n30_prior.png}
        & \includegraphics[width=\linewidth, trim=0.5cm 0.5cm 0.5cm 0.5cm, clip]{#1/benchmark_esrf_n1000_ml_n30_prior.png}
        & \includegraphics[width=\linewidth, trim=0.5cm 0.5cm 0.5cm 0.5cm, clip]{#1/benchmark_ienkf_n1000_ml_n30_prior.png}
        & \includegraphics[width=\linewidth, trim=0.5cm 0.5cm 0.5cm 0.5cm, clip]{#1/benchmark_smf_n1000_ml_n30_prior.png} \\
        \begin{tabular}{@{}c@{}}\textcolor{red!45}{Filtering}\\ Density\end{tabular}
        & \includegraphics[width=\linewidth, trim=0.5cm 0.5cm 0.5cm 0.5cm, clip]{#1/truth_n30_post.png}
        & \begin{tabular}{@{}c@{}}$N$=30\\ Training\end{tabular}
        & \includegraphics[width=\linewidth, trim=0.5cm 0.5cm 0.5cm 0.5cm, clip]{#1/ml_ete_es_n30_post.png}
        & \includegraphics[width=\linewidth, trim=0.5cm 0.5cm 0.5cm 0.5cm, clip]{#1/ml_corrterms_es_n30_post.png}
        & \includegraphics[width=\linewidth, trim=0.5cm 0.5cm 0.5cm 0.5cm, clip]{#1/ml_ete_nl2_n30_post.png}
        & \includegraphics[width=\linewidth, trim=0.5cm 0.5cm 0.5cm 0.5cm, clip]{#1/ml_corrterms_nl2_n30_post.png}
        & $N$=1000
        & \includegraphics[width=\linewidth, trim=0.5cm 0.5cm 0.5cm 0.5cm, clip]{#1/benchmark_enkf_n1000_ml_n30_post.png}
        & \includegraphics[width=\linewidth, trim=0.5cm 0.5cm 0.5cm 0.5cm, clip]{#1/benchmark_esrf_n1000_ml_n30_post.png}
        & \includegraphics[width=\linewidth, trim=0.5cm 0.5cm 0.5cm 0.5cm, clip]{#1/benchmark_ienkf_n1000_ml_n30_post.png}
        & \includegraphics[width=\linewidth, trim=0.5cm 0.5cm 0.5cm 0.5cm, clip]{#1/benchmark_smf_n1000_ml_n30_post.png} \\
        \addlinespace
        \begin{tabular}{@{}c@{}}Predictive\\ Density\end{tabular}
        & \includegraphics[width=\linewidth, trim=0.5cm 0.5cm 0.5cm 0.5cm, clip]{#1/truth_n100_prior.png}
        & $N$=100
        & \includegraphics[width=\linewidth, trim=0.5cm 0.5cm 0.5cm 0.5cm, clip]{#1/ml_ete_es_n100_prior.png}
        & \includegraphics[width=\linewidth, trim=0.5cm 0.5cm 0.5cm 0.5cm, clip]{#1/ml_corrterms_es_n100_prior.png}
        & \includegraphics[width=\linewidth, trim=0.5cm 0.5cm 0.5cm 0.5cm, clip]{#1/ml_ete_nl2_n100_prior.png}
        & \includegraphics[width=\linewidth, trim=0.5cm 0.5cm 0.5cm 0.5cm, clip]{#1/ml_corrterms_nl2_n100_prior.png}
        & $N$=3000
        & \includegraphics[width=\linewidth, trim=0.5cm 0.5cm 0.5cm 0.5cm, clip]{#1/benchmark_enkf_n3000_ml_n100_prior.png}
        & \includegraphics[width=\linewidth, trim=0.5cm 0.5cm 0.5cm 0.5cm, clip]{#1/benchmark_esrf_n3000_ml_n100_prior.png}
        &
        & \includegraphics[width=\linewidth, trim=0.5cm 0.5cm 0.5cm 0.5cm, clip]{#1/benchmark_smf_n3000_ml_n100_prior.png} \\
        \begin{tabular}{@{}c@{}}\textcolor{red!45}{Filtering}\\ Density\end{tabular}
        & \includegraphics[width=\linewidth, trim=0.5cm 0.5cm 0.5cm 0.5cm, clip]{#1/truth_n100_post.png}
        & $N$=100
        & \includegraphics[width=\linewidth, trim=0.5cm 0.5cm 0.5cm 0.5cm, clip]{#1/ml_ete_es_n100_post.png}
        & \includegraphics[width=\linewidth, trim=0.5cm 0.5cm 0.5cm 0.5cm, clip]{#1/ml_corrterms_es_n100_post.png}
        & \includegraphics[width=\linewidth, trim=0.5cm 0.5cm 0.5cm 0.5cm, clip]{#1/ml_ete_nl2_n100_post.png}
        & \includegraphics[width=\linewidth, trim=0.5cm 0.5cm 0.5cm 0.5cm, clip]{#1/ml_corrterms_nl2_n100_post.png}
        & $N$=3000
        & \includegraphics[width=\linewidth, trim=0.5cm 0.5cm 0.5cm 0.5cm, clip]{#1/benchmark_enkf_n3000_ml_n100_post.png}
        & \includegraphics[width=\linewidth, trim=0.5cm 0.5cm 0.5cm 0.5cm, clip]{#1/benchmark_esrf_n3000_ml_n100_post.png}
        &
        & \includegraphics[width=\linewidth, trim=0.5cm 0.5cm 0.5cm 0.5cm, clip]{#1/benchmark_smf_n3000_ml_n100_post.png} \\
        \bottomrule
        \end{tabular}%
    }%
}

\newcommand{\IEnKFCell}[4]{%
\ifnum#1=1
\multicolumn{1}{@{}p{0pt}@{}}{}%
\else
\fbox{\includegraphics[width=\linewidth]{figures/lorenz63_grid/timestep#2_#3_IEnKF_N100_#4.pdf}}%
\fi
}

\newcommand{\LorenzRowLabel}[2]{%
\begin{tabular}{@{}c@{}}\scriptsize Cycle #1\\[-1pt]\scriptsize dim $#2$\end{tabular}%
}

\newcommand{\LorenzGridRows}[3]{%
\LorenzRowLabel{#1}{1,2}
& \fbox{\includegraphics[width=\linewidth]{figures/lorenz63_grid/timestep#1_#2_Ground_Truth_PFN1000000_12_r00_c00.pdf}}
& \fbox{\includegraphics[width=\linewidth]{figures/lorenz63_grid/timestep#1_#2_EtE_ES_N100_12_r00_c01.pdf}}
& \fbox{\includegraphics[width=\linewidth]{figures/lorenz63_grid/timestep#1_#2_CorrTerms_ES_N100_12_r00_c02.pdf}}
& \fbox{\includegraphics[width=\linewidth]{figures/lorenz63_grid/timestep#1_#2_EtE_NL2_N100_12_r00_c03.pdf}}
& \fbox{\includegraphics[width=\linewidth]{figures/lorenz63_grid/timestep#1_#2_CorrTerms_NL2_N100_12_r00_c04.pdf}}
& \fbox{\includegraphics[width=\linewidth]{figures/lorenz63_grid/timestep#1_#2_ESRF_N100_12_r00_c05.pdf}}
& \fbox{\includegraphics[width=\linewidth]{figures/lorenz63_grid/timestep#1_#2_EnKF_N100_12_r00_c06.pdf}}
& \IEnKFCell{#3}{#1}{#2}{12_r00_c07}
& \fbox{\includegraphics[width=\linewidth]{figures/lorenz63_grid/timestep#1_#2_SMF_N100_12_r00_c08.pdf}}
\\
\LorenzRowLabel{#1}{1,3}
& \fbox{\includegraphics[width=\linewidth]{figures/lorenz63_grid/timestep#1_#2_Ground_Truth_PFN1000000_13_r01_c00.pdf}}
& \fbox{\includegraphics[width=\linewidth]{figures/lorenz63_grid/timestep#1_#2_EtE_ES_N100_13_r01_c01.pdf}}
& \fbox{\includegraphics[width=\linewidth]{figures/lorenz63_grid/timestep#1_#2_CorrTerms_ES_N100_13_r01_c02.pdf}}
& \fbox{\includegraphics[width=\linewidth]{figures/lorenz63_grid/timestep#1_#2_EtE_NL2_N100_13_r01_c03.pdf}}
& \fbox{\includegraphics[width=\linewidth]{figures/lorenz63_grid/timestep#1_#2_CorrTerms_NL2_N100_13_r01_c04.pdf}}
& \fbox{\includegraphics[width=\linewidth]{figures/lorenz63_grid/timestep#1_#2_ESRF_N100_13_r01_c05.pdf}}
& \fbox{\includegraphics[width=\linewidth]{figures/lorenz63_grid/timestep#1_#2_EnKF_N100_13_r01_c06.pdf}}
& \IEnKFCell{#3}{#1}{#2}{13_r01_c07}
& \fbox{\includegraphics[width=\linewidth]{figures/lorenz63_grid/timestep#1_#2_SMF_N100_13_r01_c08.pdf}}
\\
\LorenzRowLabel{#1}{2,3}
& \fbox{\includegraphics[width=\linewidth]{figures/lorenz63_grid/timestep#1_#2_Ground_Truth_PFN1000000_23_r02_c00.pdf}}
& \fbox{\includegraphics[width=\linewidth]{figures/lorenz63_grid/timestep#1_#2_EtE_ES_N100_23_r02_c01.pdf}}
& \fbox{\includegraphics[width=\linewidth]{figures/lorenz63_grid/timestep#1_#2_CorrTerms_ES_N100_23_r02_c02.pdf}}
& \fbox{\includegraphics[width=\linewidth]{figures/lorenz63_grid/timestep#1_#2_EtE_NL2_N100_23_r02_c03.pdf}}
& \fbox{\includegraphics[width=\linewidth]{figures/lorenz63_grid/timestep#1_#2_CorrTerms_NL2_N100_23_r02_c04.pdf}}
& \fbox{\includegraphics[width=\linewidth]{figures/lorenz63_grid/timestep#1_#2_ESRF_N100_23_r02_c05.pdf}}
& \fbox{\includegraphics[width=\linewidth]{figures/lorenz63_grid/timestep#1_#2_EnKF_N100_23_r02_c06.pdf}}
& \IEnKFCell{#3}{#1}{#2}{23_r02_c07}
& \fbox{\includegraphics[width=\linewidth]{figures/lorenz63_grid/timestep#1_#2_SMF_N100_23_r02_c08.pdf}}
}

\newcommand{\LorenzGridRowsNoIEnKF}[2]{%
\LorenzRowLabel{#1}{1,2}
& \fbox{\includegraphics[width=\linewidth]{figures/lorenz63_grid/timestep#1_#2_Ground_Truth_PFN1000000_12_r00_c00.pdf}}
& \fbox{\includegraphics[width=\linewidth]{figures/lorenz63_grid/timestep#1_#2_EtE_ES_N100_12_r00_c01.pdf}}
& \fbox{\includegraphics[width=\linewidth]{figures/lorenz63_grid/timestep#1_#2_CorrTerms_ES_N100_12_r00_c02.pdf}}
& \fbox{\includegraphics[width=\linewidth]{figures/lorenz63_grid/timestep#1_#2_EtE_NL2_N100_12_r00_c03.pdf}}
& \fbox{\includegraphics[width=\linewidth]{figures/lorenz63_grid/timestep#1_#2_CorrTerms_NL2_N100_12_r00_c04.pdf}}
& \fbox{\includegraphics[width=\linewidth]{figures/lorenz63_grid/timestep#1_#2_ESRF_N100_12_r00_c05.pdf}}
& \fbox{\includegraphics[width=\linewidth]{figures/lorenz63_grid/timestep#1_#2_EnKF_N100_12_r00_c06.pdf}}
& \fbox{\includegraphics[width=\linewidth]{figures/lorenz63_grid/timestep#1_#2_SMF_N100_12_r00_c08.pdf}}
\\
\LorenzRowLabel{#1}{1,3}
& \fbox{\includegraphics[width=\linewidth]{figures/lorenz63_grid/timestep#1_#2_Ground_Truth_PFN1000000_13_r01_c00.pdf}}
& \fbox{\includegraphics[width=\linewidth]{figures/lorenz63_grid/timestep#1_#2_EtE_ES_N100_13_r01_c01.pdf}}
& \fbox{\includegraphics[width=\linewidth]{figures/lorenz63_grid/timestep#1_#2_CorrTerms_ES_N100_13_r01_c02.pdf}}
& \fbox{\includegraphics[width=\linewidth]{figures/lorenz63_grid/timestep#1_#2_EtE_NL2_N100_13_r01_c03.pdf}}
& \fbox{\includegraphics[width=\linewidth]{figures/lorenz63_grid/timestep#1_#2_CorrTerms_NL2_N100_13_r01_c04.pdf}}
& \fbox{\includegraphics[width=\linewidth]{figures/lorenz63_grid/timestep#1_#2_ESRF_N100_13_r01_c05.pdf}}
& \fbox{\includegraphics[width=\linewidth]{figures/lorenz63_grid/timestep#1_#2_EnKF_N100_13_r01_c06.pdf}}
& \fbox{\includegraphics[width=\linewidth]{figures/lorenz63_grid/timestep#1_#2_SMF_N100_13_r01_c08.pdf}}
\\
\LorenzRowLabel{#1}{2,3}
& \fbox{\includegraphics[width=\linewidth]{figures/lorenz63_grid/timestep#1_#2_Ground_Truth_PFN1000000_23_r02_c00.pdf}}
& \fbox{\includegraphics[width=\linewidth]{figures/lorenz63_grid/timestep#1_#2_EtE_ES_N100_23_r02_c01.pdf}}
& \fbox{\includegraphics[width=\linewidth]{figures/lorenz63_grid/timestep#1_#2_CorrTerms_ES_N100_23_r02_c02.pdf}}
& \fbox{\includegraphics[width=\linewidth]{figures/lorenz63_grid/timestep#1_#2_EtE_NL2_N100_23_r02_c03.pdf}}
& \fbox{\includegraphics[width=\linewidth]{figures/lorenz63_grid/timestep#1_#2_CorrTerms_NL2_N100_23_r02_c04.pdf}}
& \fbox{\includegraphics[width=\linewidth]{figures/lorenz63_grid/timestep#1_#2_ESRF_N100_23_r02_c05.pdf}}
& \fbox{\includegraphics[width=\linewidth]{figures/lorenz63_grid/timestep#1_#2_EnKF_N100_23_r02_c06.pdf}}
& \fbox{\includegraphics[width=\linewidth]{figures/lorenz63_grid/timestep#1_#2_SMF_N100_23_r02_c08.pdf}}
}

%% file: sections/introduction.tex
\section{Introduction}
\label{sec:introduction}

The \emph{Bayesian filter} is defined by the sequential updating of the probability distribution of a hidden state of a dynamical system, given partial and noisy observations. \emph{Filtering algorithms,} which approximate the Bayesian filter,
are central to the field of data assimilation (DA).
Classical sequential methods, such as particle filters (PFs) and the ensemble Kalman filter (EnKF), form the backbone of modern applications aimed at approximating the filtering problem, but face significant limitations: 
PFs suffer from weight degeneracy in high dimensions, requiring a number of particles that scales exponentially with the state dimension \citep{snyder2008obstacles,bengtsson2008curse}; on the other hand
the EnKF, which weights all particles equally and hence does not suffer from weight degeneracy, is designed to accurately approximate the true filter only when it is close to Gaussian \citep{le_gland_large_2011,mandel2011convergence,carrillo2024mean,calvello_ensemble_2025}.  In recent years, data-driven approaches have emerged as a promising avenue to overcome these barriers. However, learning the true filtering distribution remains fundamentally difficult. The primary bottleneck lies in the unavailability of ground-truth filtering distributions for training. Consequently, supervised learning of the Bayesian posterior has historically been restricted to low-dimensional model problems where the ground truth is accessible.

In response to these challenges, we propose the proper scoring ensemble filter (PSEF), a machine learning framework designed to approximate the true Bayesian filtering distribution. Unlike traditional approaches that rely on explicit physical assumptions or hand-crafted correction terms, the PSEF learns a parameterized analysis operator on probability measures whose particle-wise implementation maps a forecast ensemble to an analysis ensemble. The centerpiece of our approach is a training objective based on \emph{strictly proper scoring rules}. This formulation enables a supervised learning paradigm that is both theoretically grounded and practically realizable: it allows the model to learn solely from ground-truth state--observation trajectories that are available through simulation, avoiding the need to access the (typically intractable) true filtering distribution during training.

We provide a population-level theoretical foundation for the PSEF, based on the mean-field (infinite particle) limit of the proposed filter \citep{calvello_ensemble_2025}. Under a realizability assumption, the resulting proper scoring rule objective is minimized by filters that recover the Bayesian filtering distribution. Under regularity and ergodicity assumptions, we also relate the mean-field loss arising from a single state--observation trajectory, evaluated at a realizable exact-filter parameter, to the limiting population objective.  This clarifies the statistical target of the proposed objective by separating the 
mean-field, population limit and theory from the finite-ensemble, empirical, optimization problem deployed in practice.
Furthermore, we implement a permutation-equivariant architecture capable of handling varying ensemble sizes, which facilitates a computationally efficient training strategy: pretraining on small ensembles followed by lightweight fine-tuning on larger target ensembles. The practical viability of this theoretically grounded approach is demonstrated through a series of comprehensive numerical experiments on partially and noisily observed chaotic dynamical systems.

In this section, we review the relevant literature in Subsection~\ref{ssec:literature_review} and establish the mathematical notation used throughout the paper in Subsection~\ref{ssec:NO}. We summarize our primary contributions in Subsection~\ref{ssec:CO}. Finally, we provide an illustrative nonlinear filtering example in Subsection~\ref{ssec:an_illustrative_example} to demonstrate that a strictly proper scoring rule objective is useful for learning non-Gaussian filtering distributions.

\subsection{Literature Review}
\label{ssec:literature_review}

This subsection reviews the literature relevant to our proposed machine-learning-based filter.

\subsubsection{Classical data assimilation}
\label{sssec:classical_da}

Bayesian data assimilation alternates prediction using the dynamical model with analysis conditioned on observations; this area is reviewed in works including \citet{jazwinski1970stochastic,doucet2001introduction,law2015data,reich2015probabilistic,asch2016data,carrassi2018data,evensen2022data}.

The Kalman filter is exact for linear--Gaussian models \citep{Kalman1960new,ho_bayesian_1964}, while the EnKF approximates forecast statistics using ensembles \citep{evensen1994sequential,burgers1998analysis,anderson2001ensemble,houtekamer2016review}. Square-root variants avoid perturbed observations \citep{bishop2001adaptive,tippett_ensemble_2003}; inflation and localization regularize finite ensembles \citep{anderson1999monte,anderson2007adaptive,hunt2007efficient}; and iterative or optimization-based updates address nonlinearity \citep{zupanski2005maximum,sakov2012iterative}. Despite convergence guarantees in Gaussian or near-Gaussian settings \citep{mandel2011convergence,calvello_ensemble_2025}, the EnKF does not generally recover the Bayesian filter as the ensemble grows \citep{le_gland_large_2011}. Non-Gaussian filtering is particularly relevant at convection-resolving scales \citep{hu_progress_2023}.

Particle filters consistently approximate nonlinear Bayesian filters \citep{doucet2001introduction,gordon1993novel}, but suffer weight degeneracy in high dimensions \citep{snyder2008obstacles,bengtsson2008curse}. Alternatives include particle flows \citep{daum2011particle,van2019particle}, transport maps and normalizing flows \citep{spantini2022coupling,chipilski2025exact}, and optimal transport \citep{al2023nonlinear,bocquet2024bridging}. These approaches often solve a new inference problem at each assimilation time; here, we learn a reusable analysis operator.

\subsubsection{Machine learning for data assimilation and probabilistic training}
\label{sssec:ml_da_probabilistic}

Machine learning augments DA through learned model corrections, observation operators, latent representations, variational solvers, and surrogate dynamics \citep{bach2024inverse,fablet2021learning,frerix2021variational,chattopadhyay_towards_2022,chattopadhyay_deep_2023,bach_interface_2026,arcucci_convergence_2026}. Directly learned filters range from prescribed parametric forms \citep{hoang_simple_1994,brocker_probabilistic_2009,mallia-parfitt_assessing_2016,levine_framework_2022,luk_learning_2024} to flexible neural updates for chaotic systems \citep{mccabe_learning_2021,boudier_data_2023,bocquet_accurate_2024,hammoud_data_2024,bach_learning_2026}.

Such filters amortize inference across observations. Some learn both prediction and analysis \citep{boudier_data_2023,allen_end--end_2025}; others retain known dynamics and learn corrections, gains, inflation, localization, or ensemble updates \citep{luk_learning_2024,bocquet_accurate_2024,vishny_high-dimensional_2024,bach_learning_2026}. Deterministic training losses primarily assess point estimates. Our distributional objective is related to operator learning for maps from observations to posterior distributions \citep{nelsen2025operator}.

Probabilistic training includes variational inference \citep{luk_learning_2024} and proper scoring rules, whose expectation is minimized by the distribution that generated the data \citep{gneiting2007strictly}. Examples include logarithmic, continuous ranked probability, energy, and variogram scores \citep{gneiting2007strictly,scheuerer2015variogram}. Scoring rules use realized states as verification targets, making them suitable when simulated trajectories, but not filtering distributions, are available.

Scoring-rule training has been used for generative time-series forecasting \citep{pacchiardi_probabilistic_2024}, AIFS-CRPS \citep{lang2026aifs}, joint forecasts learned from marginal objectives \citep{alet2025skillful}, and FourCastNet~3 \citep{bonev2025fourcastnet}. These works learn forecast distributions. Scoring-rule training of posterior generators is also studied by \citet{pacchiardi_likelihood_free_2022}; here, we perform amortized simulation-based inference by composing the learned updates within a recursive forecast--analysis filter.

The book \citet[Section~10.3]{bach2024inverse} establishes the conditional scoring-rule principle and its application to realizable filtering operators; \citet{bach_learning_2026} develops EnKF corrections formulated on measures and implemented with transformers, trained for point estimation. Building on these results, this work develops an end-to-end particle analysis map and sequential energy-score training, compares architectures and losses, and analyzes the long-time behavior of the mean-field objectives.

\subsubsection{Attention-based operators on empirical measures}
\label{sssec:attention_measure_operator}

Permutation symmetry is natural for exchangeable ensembles \citep{brocker_concept_2011}. Deep sets and set transformers provide invariant or equivariant architectures \citep{zaheer2017deep,lee2019set}, with applications to ensemble updates, neural filtering, and postprocessing \citep{hohlein_postprocessing_2024,zhou_bi-eqno_2024,bach_learning_2026}.

\citet{geshkovski2023emergence,geshkovski2025mathematical} study attention dynamics on particles and measures, while \citet{calvello2025continuum} formulate continuum attention for neural operators. Related operator-learning formulations address mean-field games \citep{huang2025unsupervised}, optimal transport \citep{cole2026context}, and amortized conditioning of joint densities \citep{tsimpos2026one}. These perspectives motivate our attention-based analysis operator on empirical measures.

\subsection{Notation}\label{ssec:NO}

We write $\bbN=\{1,2,\ldots\}$, $\bbZ_+=\{0,1,2,\ldots\}$, $\bbR=(-\infty,\infty)$, and $\bbR_+=[0,\infty)$ for the sets of positive integers, nonnegative integers, real numbers, and nonnegative reals, respectively. For $N\in\bbN$, let $[N]=\{1,\ldots,N\}$. For a set $A$, define $\mathcal{U}([N];A) := A^{[N]}$, the set of sequences indexed by $[N]$ with values in $A$. We define the set of all nonempty finite sequences in $A$ by $\mathcal{U}_F(A) := \bigcup_{N=1}^{\infty}\mathcal{U}([N];A).$

For vectors $x,y\in\bbR^d$, the Euclidean inner product and norm are $\langle x,y\rangle=x^\top y$ and $\|x\|=\|x\|_2$, respectively, and for an invertible matrix $C\in \bbR^{d\times d}$ the weighted Euclidean norm is denoted as $\|x\|_C = \sqrt{x^\top C^{-1} x}$. For a matrix $A\in\bbR^{d\times d}$, $\trace(A)$ is the trace and $A^\top$ is the transpose. We write $I_d$ for the $d\times d$ identity matrix.

Let $(\Omega, d_\Omega)$ be a metric space equipped with its Borel $\sigma$-algebra $\cB$. We let $\cP(\Omega)$ denote the set of all probability measures on $(\Omega,\cB)$, and we denote by $\cP_k(\Omega)$ the subset of measures $\mu \in \cP(\Omega)$ with finite $k$-th moments ($k \in (0,\infty)$), meaning that $\int_\Omega d_\Omega(x, x_0)^k \,\mathrm{d}\mu(x) < \infty$
for any reference point $x_0 \in \Omega$. A probability space is denoted by $(\Omega,\cB,\bbP)$, and expectations under $\bbP$ are denoted by 
$\bbE^{\bbP}[\cdot]$; we drop explicit mention of $\bbP$ when it is clear
from the context what measure is intended in the expectation. For a random variable $v$, we write $\Law(v)$ for its distribution. If $T:\Omega\to\Omega'$ is a measurable map and $\pi\in\cP(\Omega)$, then the pushforward $T_\#\pi\in\cP(\Omega')$ is defined by $T_\#\pi(B)=\pi(T^{-1}(B))$ for all measurable $B\subseteq\Omega'$. If $z \sim \pi$ then $T(z) \sim T_\#\pi.$ For $\tau\in\Omega$, $\delta_\tau$ denotes the Dirac measure at $\tau$.  We write $\normal(m,C)$ for the Gaussian distribution on $\bbR^d$ with mean $m$ and covariance $C$. 
We use the abbreviation a.s.\ to denote \emph{almost surely}.
For simplicity we assume that all probability measures defined on $\bbR^d$ are absolutely continuous with respect to the Lebesgue measure (so that they have probability density functions), or that they comprise a convex combination of Dirac masses; the former is a useful simplifying assumption for the underlying
stochastic dynamics models, and the latter enables us to describe the
empirical measure defined by ensemble methods. When discussing the filtering
distribution itself, which is a random object, we will need probability measures on infinite dimensional spaces and Lebesgue measure will not be relevant.
We will refer to probability measures and probability density functions interchangeably, depending on the setting.  

For probability measures, we denote the $p$-Wasserstein metric (induced by $d_\Omega$) for $p \ge 1$ by $W_p$, and the total variation (TV) distance by $\|\cdot\|_{\mathrm{TV}}$. For a measurable function $f$ and a measure $\mu$, we write $\|f\|_{L^p(\mu)}=\big(\int |f|^p\,\mathrm{d}\mu\big)^{1/p}$ for $1\le p<\infty$. We use a sans-serif font (e.g., $\mathsf{B}$) for operators acting on the space of probability measures or the space of functions; we also use this convention for operations on $\mathcal{U}([N];\bbR^{d})$. If an operator $\sB$ on probability measures is defined as the pushforward by a map, we denote the associated transport map with a Fraktur font (e.g., $\mathfrak{B}$), i.e., $\sB(\cdot) = \fB_\sharp (\cdot)$.

In the context of DA, we use $v$ to denote the system states and $y$ to denote the observations. The superscript $\dagger$ indicates true values, and the subscript $j\in\bbZ^+$ denotes the time. For ensemble methods, we use the parenthesized superscript $(n)$, with $n \in \bbN,$ for the ensemble index.
When training DA methods we use $\theta\in\Theta$ to denote trainable parameters, where $\Theta \subseteq \bbR^p$ is the corresponding parameter space.

\subsection{Contributions and Overview}
\label{ssec:CO}

The central contribution of this work is a learning framework for Bayesian filtering that targets the full filtering distribution rather than only a point estimate of the hidden state. This perspective is important because the filtering distribution is the natural object for uncertainty quantification, yet it is rarely available as a supervised learning target in realistic data assimilation problems. Our approach shows that distributional learning can nevertheless be made practical by combining ensemble filtering, strictly proper scoring rules, and permutation-equivariant neural operators.

Our main contributions are summarized as follows.

\begin{enumerate}
    \item \textbf{An amortized distributional learning framework for ensemble approximation of the Bayesian filter.}
    We introduce the proper scoring ensemble filter (PSEF), a learned ensemble data assimilation method that amortizes the analysis step over forecast priors and observation realizations. 
    
    \item \textbf{Sequential energy-score training without filtering-distribution targets.}
    We implement sequential ensemble training using strictly proper scoring rules, with the energy score used in our implementation. This objective can be evaluated using only simulated trajectories of the state--observation process, avoiding the need to know or store the true filtering distribution during training.

    \item \textbf{Theoretical analysis of the filtering objective.}
    We derive the finite-ensemble empirical objective used in computation. At a fixed realizable exact-filter parameter, we show that the finite-trajectory and population mean-field objectives share the same long-time limit under regularity and ergodicity assumptions. 

    \item \textbf{An end-to-end, permutation-equivariant analysis operator on ensembles.}
    We implement the learned analysis step as a transformer-based particle-wise map acting on individual forecast particles, empirical forecast measures and observations. This map is invariant to permutations of its ensemble argument and induces a permutation-equivariant ensemble update. This construction respects the exchangeability of ensemble members and allows the same parameterization to be evaluated across different ensemble sizes. 

    \item \textbf{Empirical evidence across different settings.}
    We evaluate the proposed framework on diagnostic and chaotic data assimilation problems. Comparing the two analysis architectures under both energy-score and mean-based training separates their effects on distributional accuracy. The results show that energy-score-trained filters approximate complex filtering distributions in settings where Kalman-type methods and mean-based learning objectives struggle.
\end{enumerate}

The remainder of the paper is organized as follows.
Section~\ref{sec:background} introduces the mathematical background, including the Bayesian filtering formulation, mean-field filters, and proper scoring rules.
Section~\ref{sec:learning_true_filter} develops the proposed PSEF framework, establishes the population-level recovery of the true filtering distribution, derives the finite-ensemble training objective, and presents the neural analysis map and training strategy.
Section~\ref{sec:numerical_exp} evaluates the method on linear--Gaussian, doubling-angle, Lorenz~'63, and Lorenz~'96 filtering problems, comparing against non-ML benchmark filters and learning-based baselines under both distributional and calibration metrics.
Finally, Section~\ref{sec:conclusion} summarizes the findings, discusses limitations, and outlines future directions.

\subsection{An Illustrative Doubling-Angle Example}
\label{ssec:an_illustrative_example}

To illustrate PSEF, we begin with the one-dimensional doubling-angle model:
\begin{subequations}\label{eq:doubling}
\begin{align}
    &\vd_{j+1} = 2 \vd_j + \xi_j ~~(\bmod~1),\quad &&\xi_j\sim \normal(0,\sigma_v^2),\quad j\in\mathbb{Z}_+,\\
    &\yd_{j+1} = \cos (2\pi \vd_{j+1}) + \eta_{j+1},\quad &&\eta_{j+1}\sim \normal(0,\sigma_y^2),\quad j\in\mathbb{Z}_+.
\end{align}
\end{subequations}
This doubling map is chaotic and ergodic \citep{halmos_lectures_1956}, and the cosine observation cannot distinguish $\vd$ from $1-\vd$, inducing a bimodal filtering distribution. Full experimental settings are given in Subsection~\ref{ssec:doublin_angle_model}.

A bootstrap particle filter (BPF) with $10^6$ particles provides an estimate of the ground truth filtering distribution (which is not used in training). We compare two learning-based filters, one the PSEF trained with the energy score (ES), and one with a deterministic mean-based objective (normalized L2: NL2); the two filters are otherwise identical. These two filters are compared against four non-ML benchmarks: the ensemble Kalman filter (EnKF), ensemble square-root filter (ESRF), iterative ensemble Kalman filter (IEnKF), and stochastic map filter (SMF). The learned filters are trained at $N=30$ and results displayed at $N\in\{10,30,100\}$, whereas the benchmark filters use $N\in\{300,1000,3000\}$. No inflation or localization is used for these large benchmark ensembles, and the remaining SMF map parameters are specified in Appendix~\ref{app_subsec:grid_search}.

Figure~\ref{fig:doubling1d_filtering} shows that the BPF reference is strongly bimodal. PSEF follows both modes across all three displayed ensemble sizes, whereas the NL2-trained filter and the four benchmark filters do not. This comparison illustrates the benefit of training with a probabilistic score over the tested mean-based objective in this example, and shows the PSEF's capacity to estimate non-Gaussian filtering distributions, unlike the benchmark methods.

\begin{figure}[H]
    \centering
    \resizebox{\textwidth}{!}{
        \begin{tabular}{
            >{\centering\arraybackslash}m{0.106\textwidth}
            >{\centering\arraybackslash}m{0.10\textwidth}
            >{\centering\arraybackslash}m{0.106\textwidth}
            >{\centering\arraybackslash}m{0.106\textwidth}
            >{\centering\arraybackslash}m{0.10\textwidth}
            >{\centering\arraybackslash}m{0.106\textwidth}
            >{\centering\arraybackslash}m{0.106\textwidth}
            >{\centering\arraybackslash}m{0.106\textwidth}
            >{\centering\arraybackslash}m{0.106\textwidth}
        }
        \toprule
        Truth & ML $N$ & ES & NL2 & \shortstack{Benchmark\\$N$} & EnKF & ESRF & IEnKF & SMF \\
        \midrule

        \includegraphics[width=\linewidth, trim=0.5cm 0.5cm 0.5cm 0.5cm, clip]{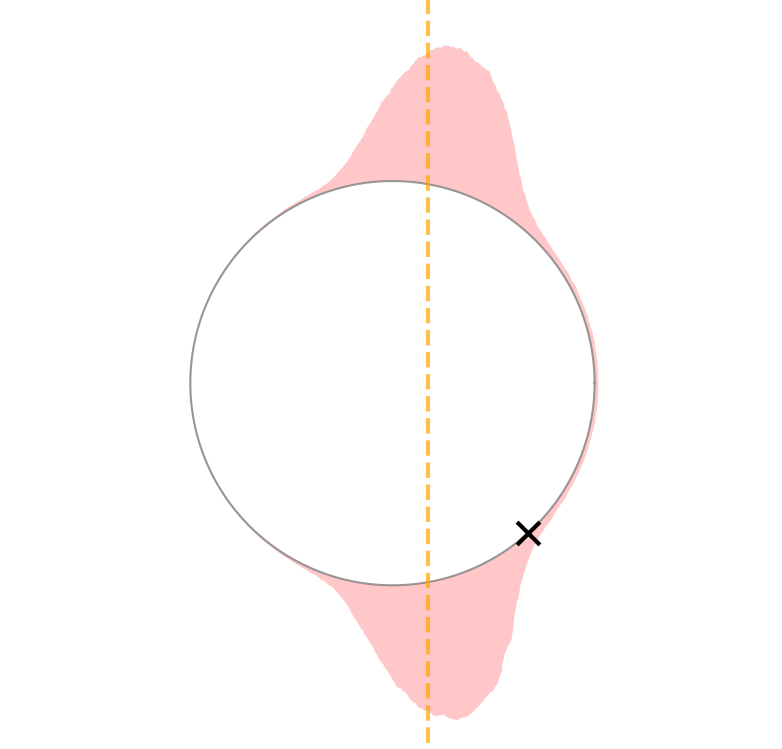} &
        $N=10$ &
        \includegraphics[width=\linewidth, trim=0.5cm 0.5cm 0.5cm 0.5cm, clip]{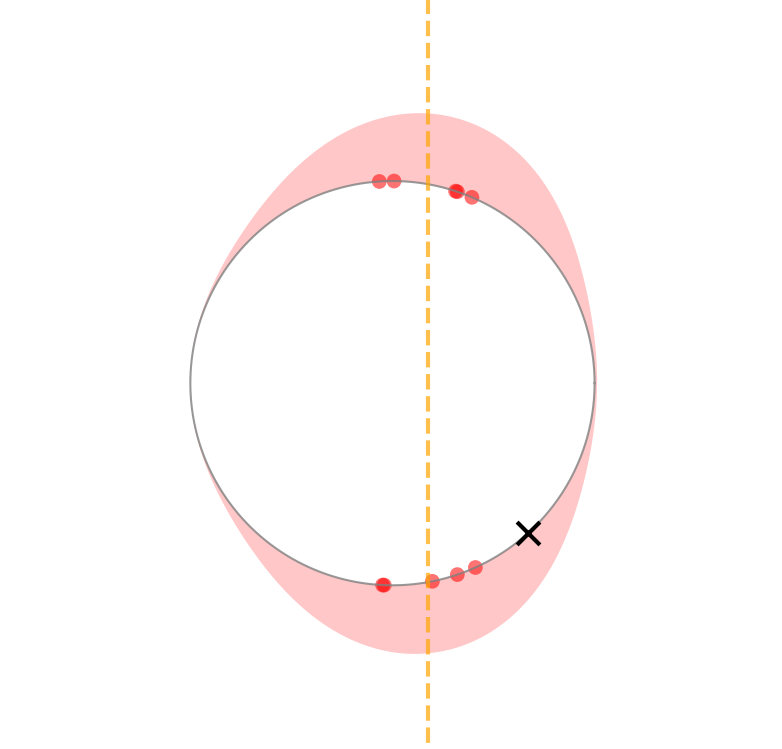} &
        \includegraphics[width=\linewidth, trim=0.5cm 0.5cm 0.5cm 0.5cm, clip]{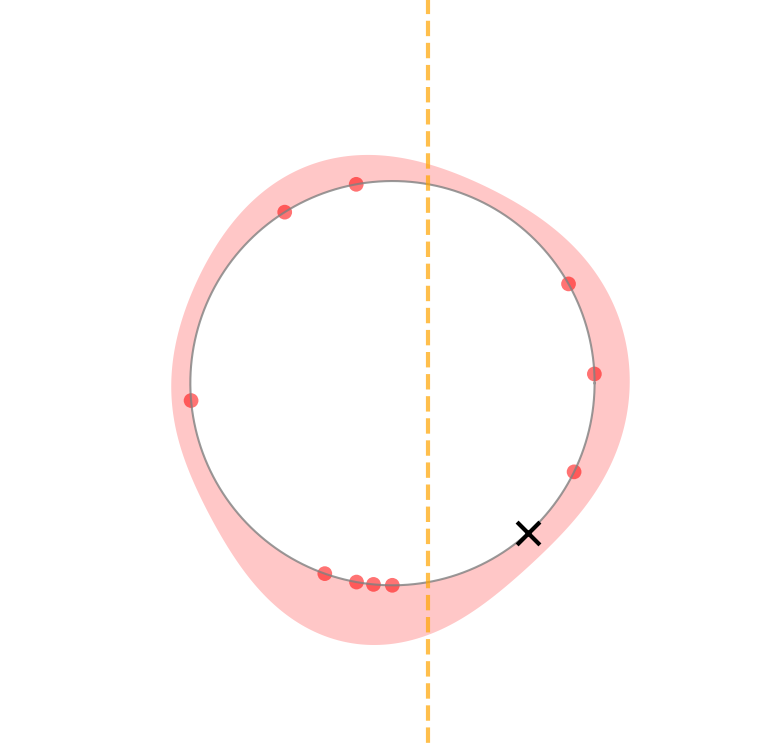} &
        $N=300$ &
        \includegraphics[width=\linewidth, trim=0.5cm 0.5cm 0.5cm 0.5cm, clip]{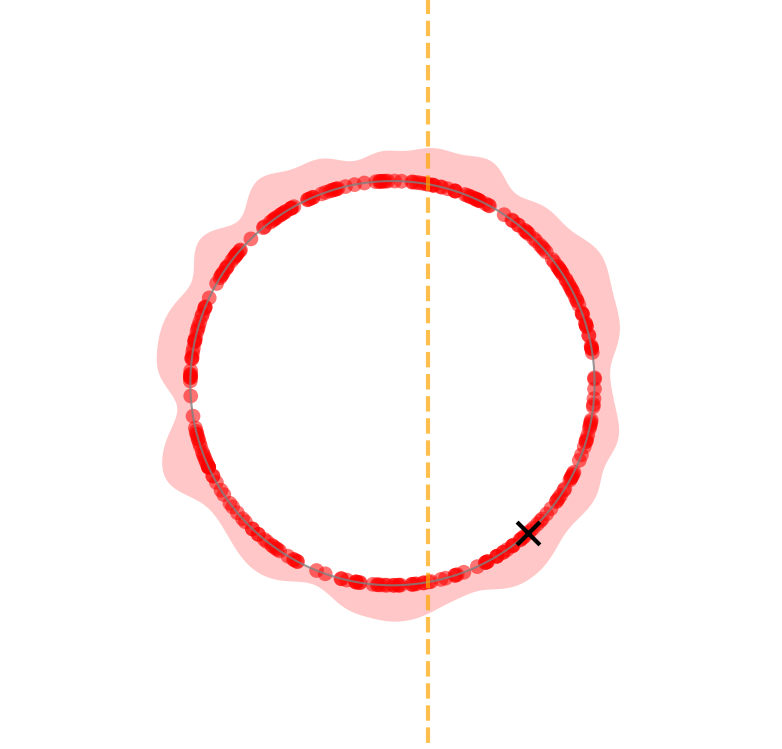} &
        \includegraphics[width=\linewidth, trim=0.5cm 0.5cm 0.5cm 0.5cm, clip]{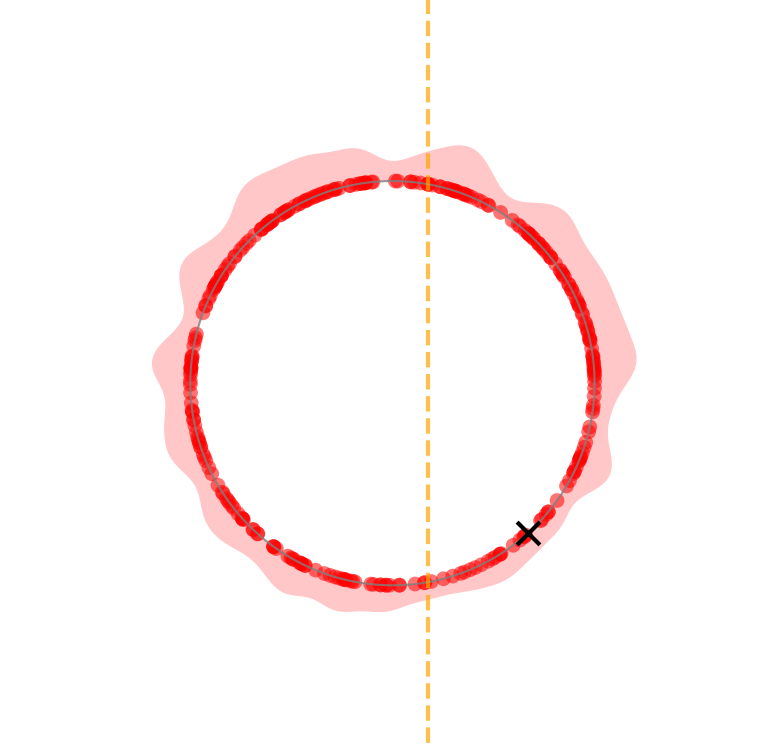} &
        \includegraphics[width=\linewidth, trim=0.5cm 0.5cm 0.5cm 0.5cm, clip]{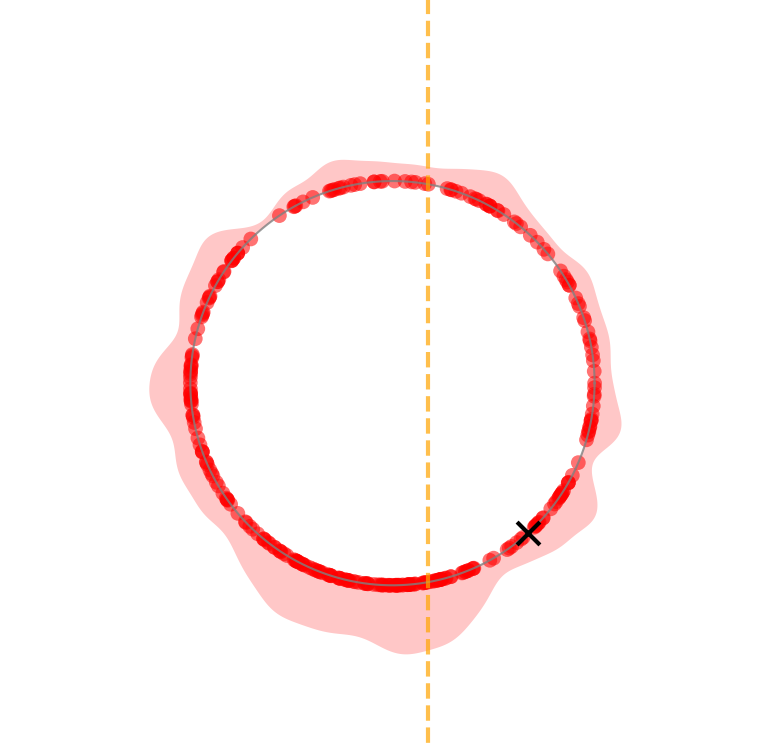} &
        \includegraphics[width=\linewidth, trim=0.5cm 0.5cm 0.5cm 0.5cm, clip]{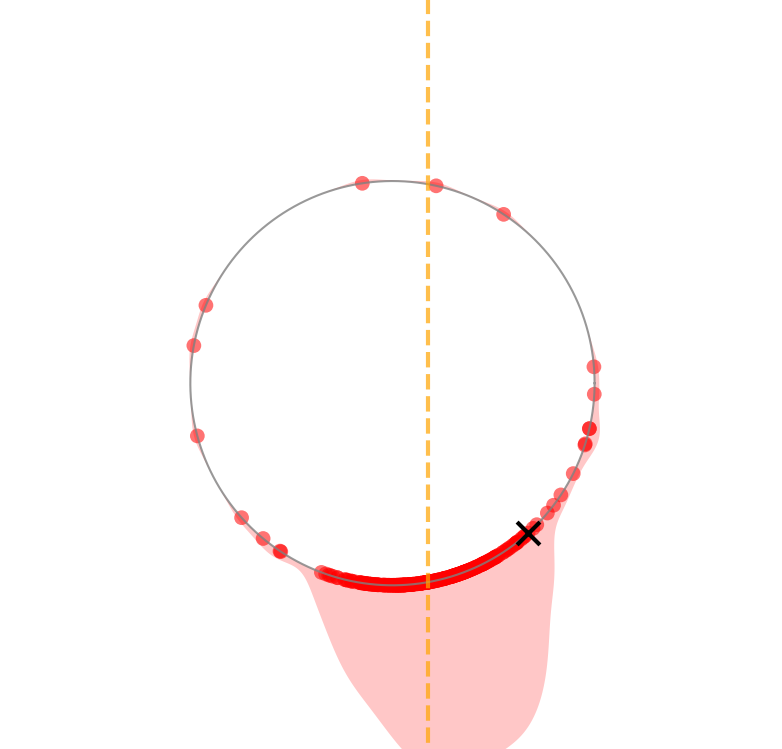} \\
        \addlinespace

        \includegraphics[width=\linewidth, trim=0.5cm 0.5cm 0.5cm 0.5cm, clip]{figures/doubling1d_step200/truth_n10_post.png} &
        $N=30$ \newline Training &
        \includegraphics[width=\linewidth, trim=0.5cm 0.5cm 0.5cm 0.5cm, clip]{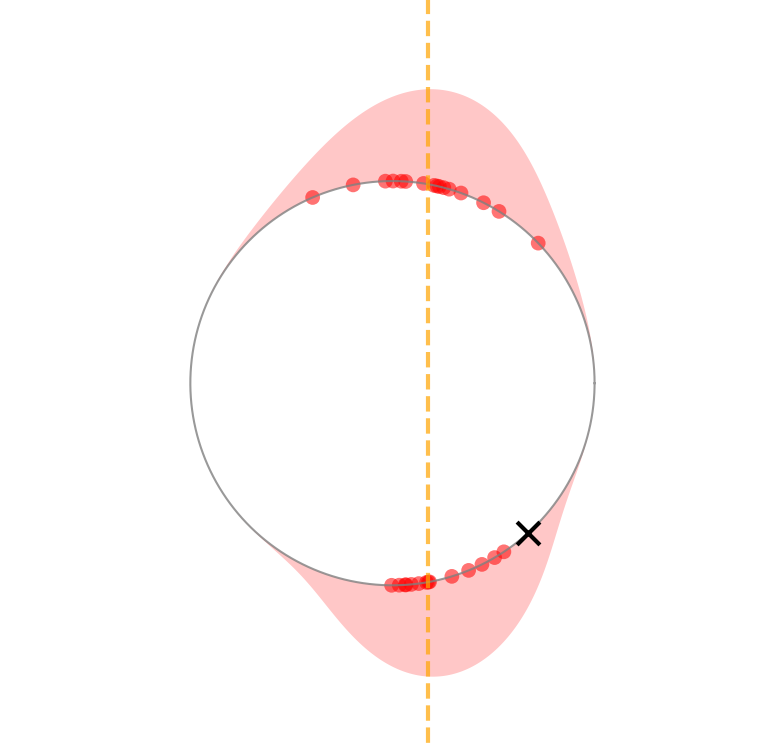} &
        \includegraphics[width=\linewidth, trim=0.5cm 0.5cm 0.5cm 0.5cm, clip]{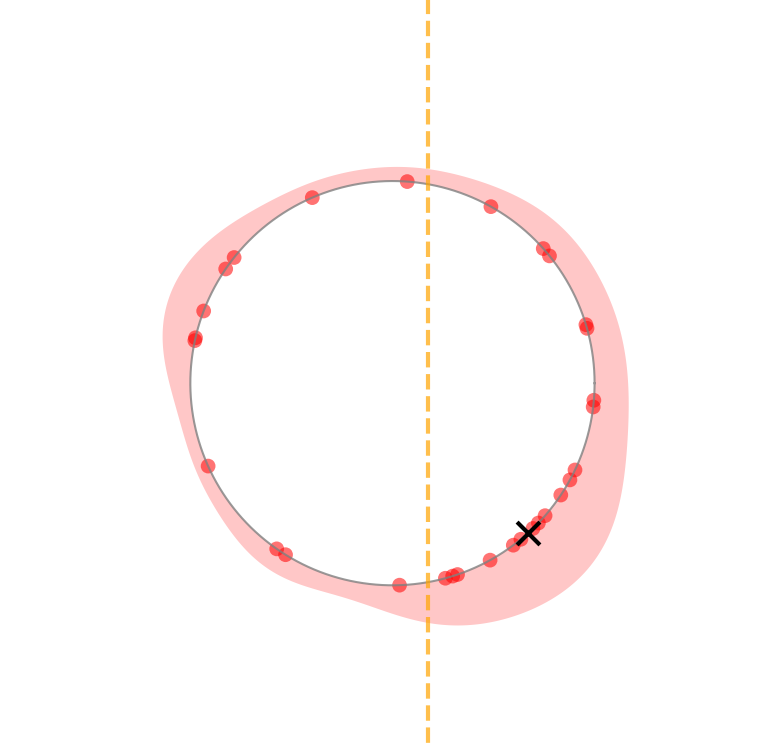} &
        $N=1000$ &
        \includegraphics[width=\linewidth, trim=0.5cm 0.5cm 0.5cm 0.5cm, clip]{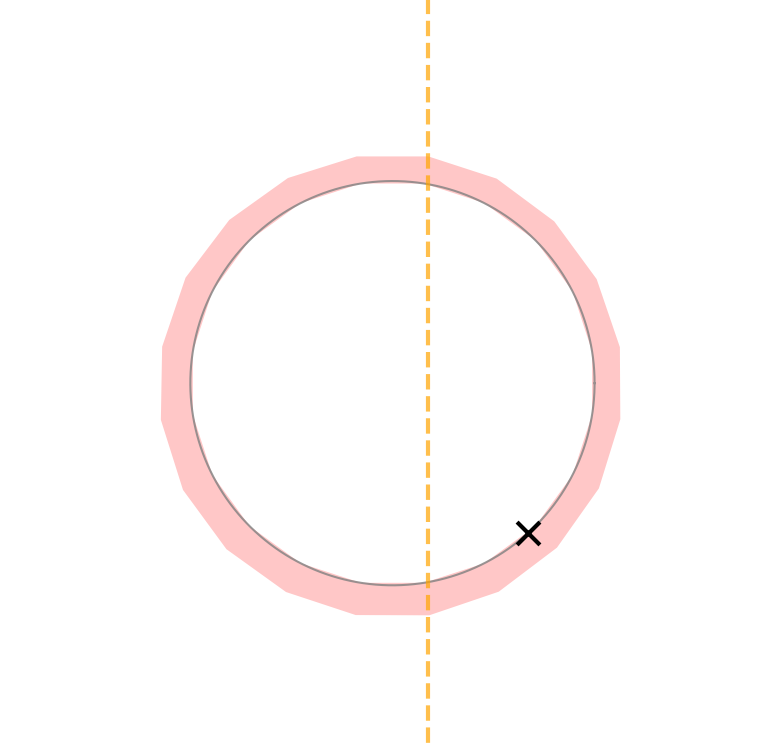} &
        \includegraphics[width=\linewidth, trim=0.5cm 0.5cm 0.5cm 0.5cm, clip]{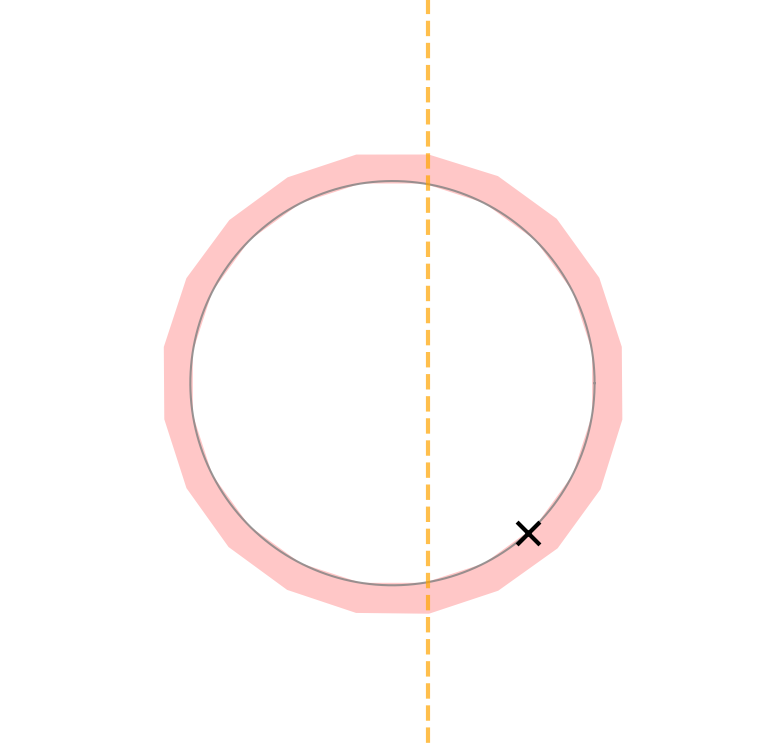} &
        \includegraphics[width=\linewidth, trim=0.5cm 0.5cm 0.5cm 0.5cm, clip]{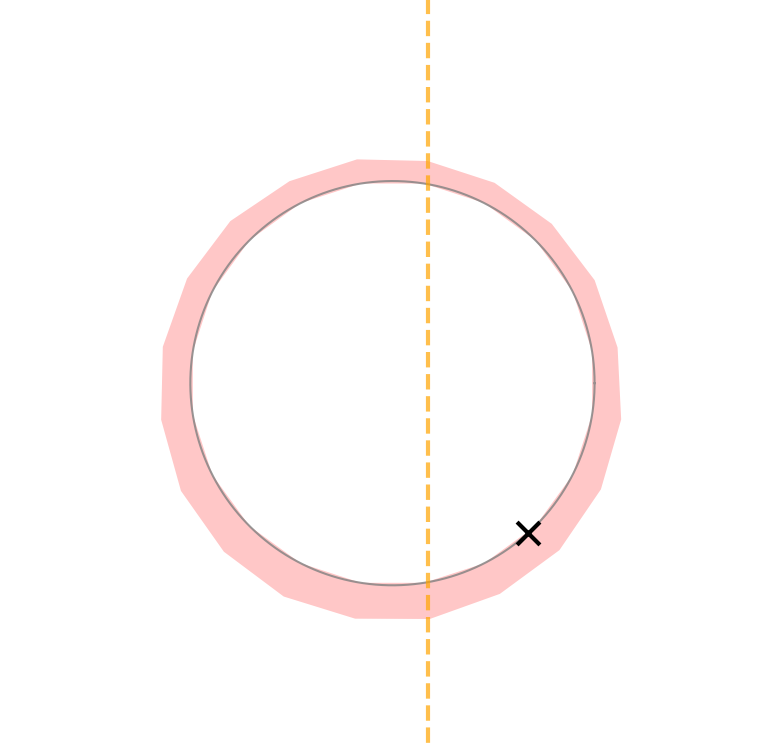} &
        \includegraphics[width=\linewidth, trim=0.5cm 0.5cm 0.5cm 0.5cm, clip]{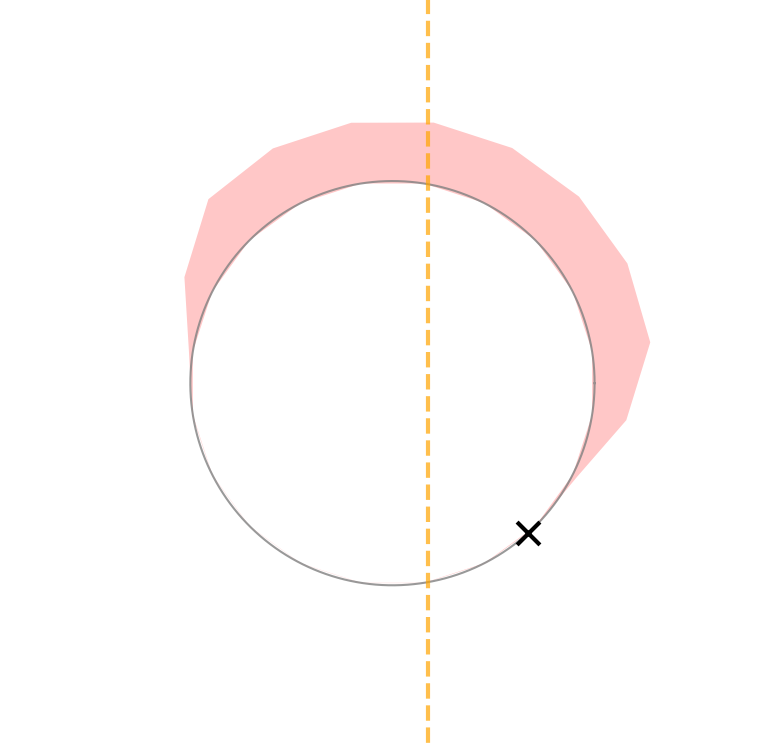} \\
        \addlinespace

        \includegraphics[width=\linewidth, trim=0.5cm 0.5cm 0.5cm 0.5cm, clip]{figures/doubling1d_step200/truth_n10_post.png} &
        $N=100$ &
        \includegraphics[width=\linewidth, trim=0.5cm 0.5cm 0.5cm 0.5cm, clip]{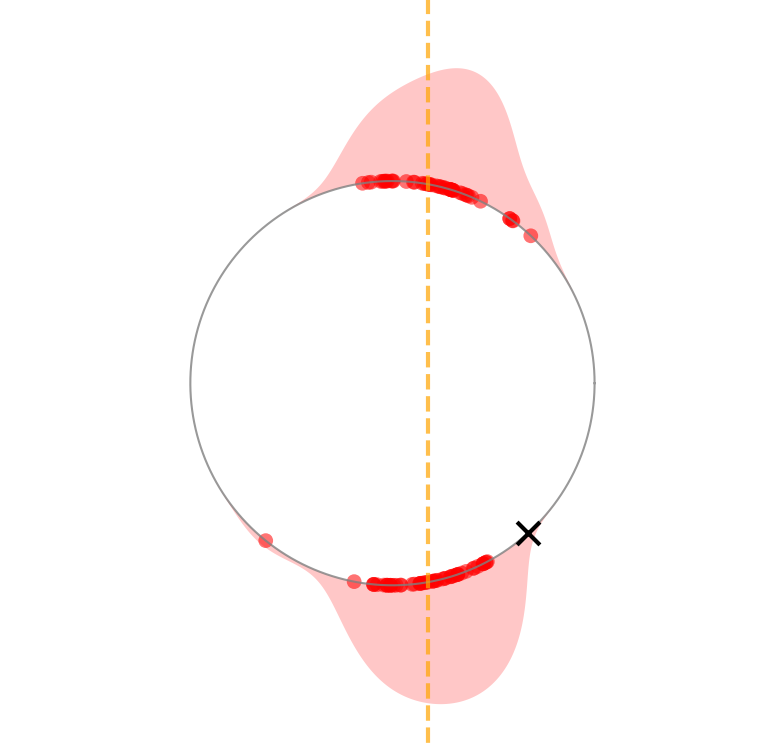} &
        \includegraphics[width=\linewidth, trim=0.5cm 0.5cm 0.5cm 0.5cm, clip]{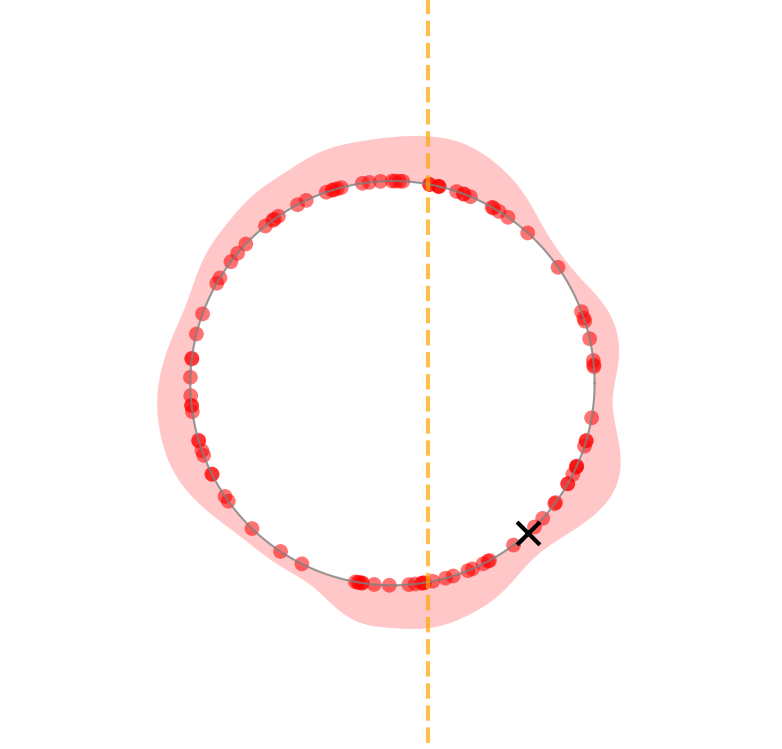} &
        $N=3000$ &
        \includegraphics[width=\linewidth, trim=0.5cm 0.5cm 0.5cm 0.5cm, clip]{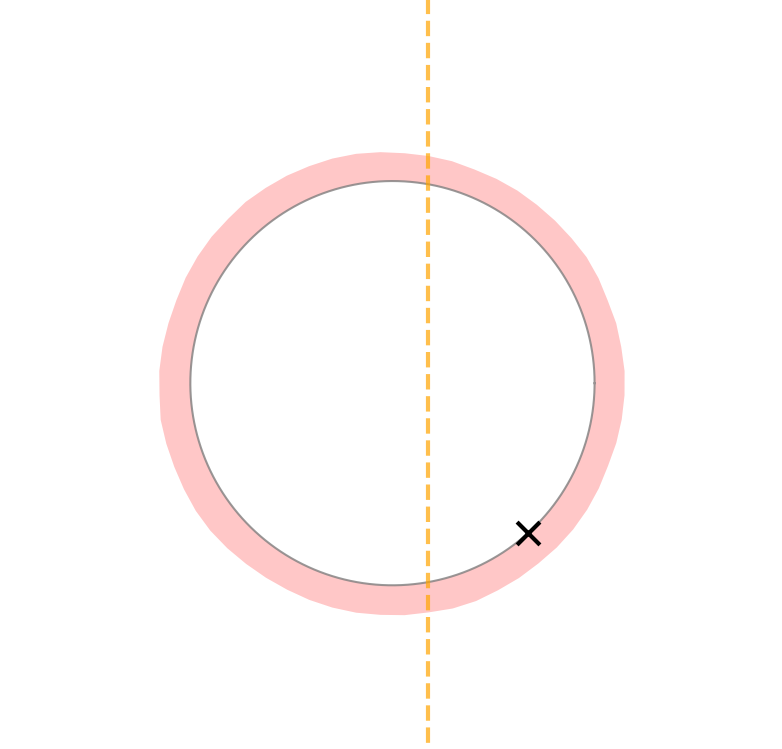} &
        \includegraphics[width=\linewidth, trim=0.5cm 0.5cm 0.5cm 0.5cm, clip]{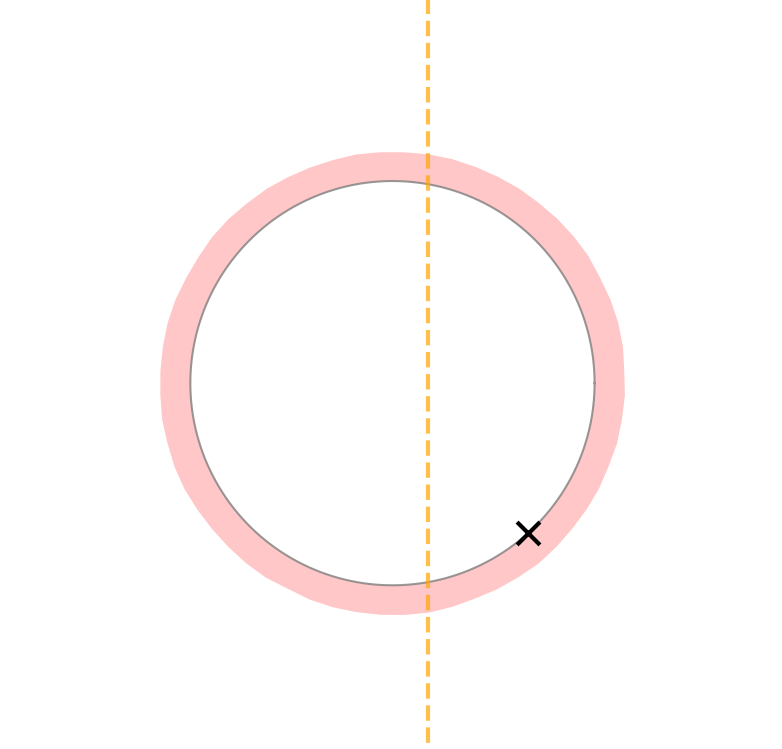} &
        \multicolumn{1}{c}{\textemdash} &
        \includegraphics[width=\linewidth, trim=0.5cm 0.5cm 0.5cm 0.5cm, clip]{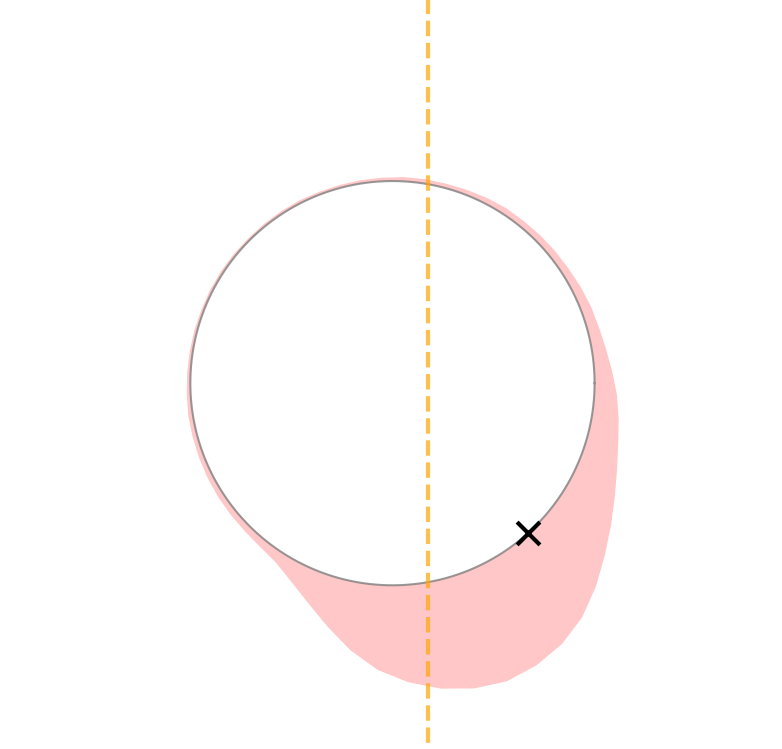} \\
        \bottomrule
        \end{tabular}
    }

    \medskip

    \includegraphics[width=0.7\textwidth, trim=0.6cm 1.5cm 0.6cm 1.5cm, clip]{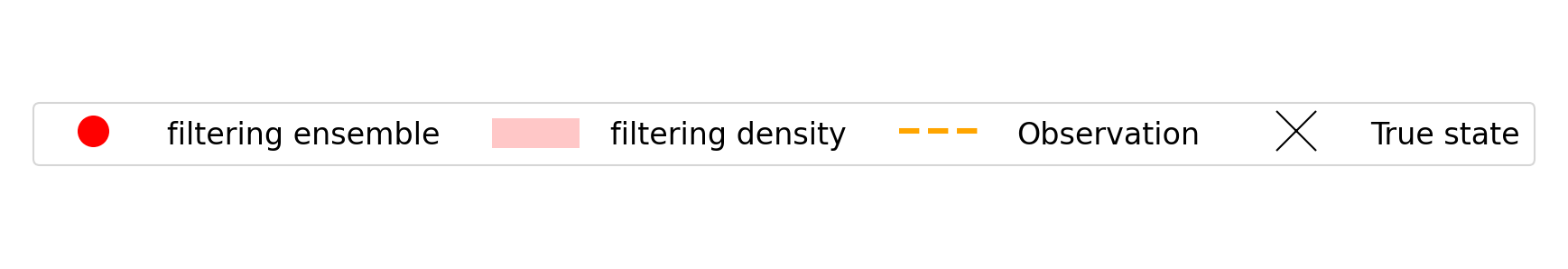}

    \caption{Filtering distributions at assimilation cycle 200 for the doubling-angle model \eqref{eq:doubling}. States are mapped to the unit circle; the vertical line marks the observation. The BPF reference uses $10^6$ particles. Learned filters use $N\in\{10,30,100\}$ (trained at $N=30$); benchmark filters use $N\in\{300,1000,3000\}$, with IEnKF at $N=3000$ omitted for cost. PSEF captures both modes; the NL2-trained model and benchmark filters do not.}
    \label{fig:doubling1d_filtering}
\end{figure}

%% file: sections/background.tex
\section{Background}\label{sec:background}

In this section, we review the concepts essential for our work. We begin by formulating the filtering problem within sequential data assimilation (DA) in Subsection~\ref{subsec:filtering_problem}. Next, we discuss mean-field filters, a relevant class of methods for this problem, in Subsection~\ref{subsec:mean_field_filters}. Finally, we introduce the theory of proper scoring rules, which underpins our proposed loss function, in Subsection~\ref{subsec:scoring_rules}. 

\subsection{The Filtering Problem}\label{subsec:filtering_problem}

The DA problem setting is based on the state--observation model defined by
\begin{subequations}\label{eq:DA_setting}
\begin{align}
\text{Dynamics model:}~~&v_{j+1}^\dag = \Psi(v_j^\dag) + \xi_j^\dag, \quad \xi_j^\dag \sim \normal(0, \Sigma) \ \text{i.i.d.},\quad v_0^\dag \sim \normal(m_0, C_0),\quad  j \in \bbZ_+, \label{eq:dynamic_model}\\
\text{Observation model:}~~&y_{j+1}^\dag = h(v_{j+1}^\dag) + \eta_{j+1}^\dag, \quad \eta_{j+1}^\dag \sim \normal(0, \Gamma) \ \text{i.i.d.},\quad j \in \bbZ_+. \label{eq:data_model}
\end{align}
\end{subequations}
The mapping $\Psi:\bbR^{d_v}\to\bbR^{d_v}$ is the deterministic forecast model, and $h:\bbR^{d_v}\to\bbR^{d_y}$ is the observation operator. The state variables $v_j^\dag\in\bbR^{d_v}$ and the process noises $\xi_j^\dag\in\bbR^{d_v}$ share the state dimension $d_v$, while the observations $y_j^\dag\in\bbR^{d_y}$ and the observation noises $\eta_j^\dag\in\bbR^{d_y}$ share the observation dimension $d_y$. We assume mutual independence between the initial condition and the two noise sequences, and between the noises themselves:
\begin{equation}
    \{\,\xi_j^\dag\,\}_{j\in\bbZ_+}\perp v_0^\dag,\qquad
\{\,\eta_j^\dag\,\}_{j\in\bbZ_+}\perp v_0^\dag,\qquad
\{\,\xi_j^\dag\,\}_{j\in\bbZ_+}\perp \{\,\eta_j^\dag\,\}_{j\in\bbZ_+}.
\end{equation}

Throughout this paper, a superscript $\dagger$ marks the true state or observation and their driving noises, distinguishing them from algorithmic variables introduced later, which appear without $\dagger$. While we work with autonomous dynamics and additive Gaussian errors in~\eqref{eq:DA_setting}, extensions are common in applications, such as time-varying $\Psi,h,\Sigma,\Gamma$, multiplicative or non-Gaussian noises, and infinite-dimensional state spaces for systems governed by partial differential equations. These extensions are not pursued in this paper to keep the presentation focused. Map $\Psi$ can be interpreted as the discrete-time flow obtained by integrating an underlying continuous-time model over a fixed interval of length $\Delta t$ between
observations; generalization to unequally spaced observations requires the generalization of our work to time-varying flow maps $\Psi_j.$

Let $Y_j^\dag := \{y_1^\dag,\dots,y_j^\dag\}$ be the observations accumulated up to time $j$. The \emph{filtering distribution} at time $j$ is the conditional law of the state given all data up to time $j$; we denote it by $\pi_j\in\cP(\bbR^{d_v})$. 
\begin{definition}[Filtering distribution]\label{def:filtering_distribution}
For each $j\in\bbZ_+$, the \emph{filtering distribution} is
\[
\pi_j := \Law\bigl(v_j^\dag \mid Y_j^\dag\bigr)\in\cP(\bbR^{d_v}),
\]
that is, the conditional probability measure of the state $v_j^\dag$ given the observations $Y_j^\dag$. 
For more background on filtering, see \cite{law2015data,reich2015probabilistic,sanz-alonso_inverse_2023}.
\end{definition}

\begin{definition}[Filtering problem]\label{def:filtering_problem}
Given the state--observation model \eqref{eq:DA_setting} and the data stream $\{y_j^\dag\}_{j\ge1}$, the \emph{filtering problem} is to approximate, and update sequentially in $j$, the filtering distributions $\{\pi_j\}_{j\ge1}$.
\end{definition}

To focus our presentation we make assumptions concerning the long-term statistical behavior of the stochastic dynamics model \eqref{eq:dynamic_model}. This behavior is determined by the interplay between the deterministic dynamics $\Psi$ and the stochastic forcing quantified by $\Law(\xi)$. Together, the pair $(\Psi, \Law(\xi))$ induces a discrete-time Markov chain on the state space $\bbR^{d_v}$, characterized by the transition kernel $\mathcal{K}: \bbR^{d_v} \times \cB(\bbR^{d_v}) \to [0,1]$, defined as
\begin{equation}
    \mathcal{K}(v, A) := \bbP(\Psi(v) + \xi \in A), \quad \xi \sim \normal(0, \Sigma).
\end{equation}
We assume the existence of an \emph{invariant measure} $\mu_\infty \in \cP(\bbR^{d_v})$ for this kernel.

\begin{definition}[Invariant measure]\label{def:stationarity}
A probability measure $\mu_\infty$ is \emph{invariant} (or stationary) with respect to the transition kernel $\mathcal{K}$ (and thus the pair $(\Psi, \Sigma)$) if
\begin{equation}
    \mu_\infty(A) = \int_{\bbR^{d_v}} \mathcal{K}(v, A) \, \mu_\infty(dv) \quad \text{for all } A \in \cB(\bbR^{d_v}).
\end{equation}
If $v_0^\dag \sim \mu_\infty$, the process $\{v_j^\dag\}_{j\in\bbZ_+}$ is stationary.
\end{definition}

\begin{definition}[Ergodicity]\label{def:ergodicity}
The stochastic process $\{v_j^\dag\}_{j\ge0}$ generated by the dynamics model \eqref{eq:dynamic_model} is \emph{ergodic} with respect to the invariant measure $\mu_\infty$ if, for any $\mu_\infty$-integrable test function $\phi: \bbR^{d_v} \to \bbR$, the Birkhoff average converges almost surely with respect to the law of the full stochastic trajectory:
\begin{equation}
\label{eq:birkhoff}
    \lim_{N \to \infty} \frac{1}{N} \sum_{j=0}^{N-1} \phi(v_j^\dag) = \bbE^{v \sim \mu_\infty}[\phi(v)], \quad \text{a.s.}
\end{equation}
\end{definition}
In this work, we assume the dynamics model is ergodic. Crucially for our learning-based approach, this assumption guarantees that temporal averages converge to ensemble expectations, thereby justifying the use of time-averaged losses to approximate the expected risk over the invariant distribution.

\subsection{Mean-Field Filters}
\label{subsec:mean_field_filters}

Sequential DA proceeds by evolving the filtering distribution through a prediction–analysis cycle driven by the state--observation model \eqref{eq:DA_setting}. 
Based on the filtering distribution $\pi_j$, we introduce the one-step predictive law for the state and its joint extension to include the next observation:
\begin{subequations}
\begin{align}
    \hat{\pi}_{j+1}
    &:= \Law\bigl(v_{j+1}^\dagger \mid Y_j^\dagger\bigr)
    \in \cP(\bbR^{d_v}), \label{eq:measure_hpi}\\
    \rho_{j+1}
    &:= \Law\bigl((v_{j+1}^\dagger,y_{j+1}^\dagger) \mid Y_j^\dagger\bigr)
    \in \cP(\bbR^{d_v+d_y}). \label{eq:measure_rho}
\end{align}
\end{subequations}
When these laws admit densities, we use the same symbols for the corresponding density functions.
The goal is to transport $\pi_j$ to $\pi_{j+1}$ by first evolving from $\pi_j$ to the predictive distribution $\hat{\pi}_{j+1}$ under the dynamics model \eqref{eq:dynamic_model}, then extending to the joint state–observation distribution $\rho_{j+1}$ via the observation model \eqref{eq:data_model}, and finally conditioning at the observation $y_{j+1}^\dagger$ to obtain $\pi_{j+1}$.

\paragraph{Probability-map evolution}
The update $\pi_j \to \pi_{j+1}$ can be expressed with two linear operators, $\sP$ (forecast) and $\sQ$ (likelihood extension), followed by a nonlinear conditioning operator $\sB$:
\begin{subequations}\label{eq:prob_map_evo}
\begin{align}
    &\text{Prediction (linear):} & \hat{\pi}_{j+1} &= \sP \pi_j, \label{eq:pred}\\
    &\text{Extend to state--observation space (linear):} & \rho_{j+1} &= \sQ \hat{\pi}_{j+1}, \label{eq:extend}\\
    &\text{Analysis (nonlinear):} & \pi_{j+1} &= \sB(\rho_{j+1}; y_{j+1}^\dagger). \label{eq:analysis}
\end{align}
\end{subequations}
The initial filtering distribution is denoted by $\pi_0$; in our experiments, unless otherwise specified, we take $\pi_0 = \normal(m_0, C_0)$.
For the stochastic dynamics $v_{j+1}=\Psi(v_j)+\xi_j$ with $\xi_j\sim\normal(0,\Sigma)$,
\begin{equation}
\label{eq:operator_P}
    \sP \pi(v) = \frac{1}{\sqrt{(2\pi)^{d_v} \det \Sigma}} 
    \int \exp\!\left(-\tfrac{1}{2} \|v - \Psi(u)\|_{\Sigma}^2 \right) \pi(u) \, \mathrm{d}u.
\end{equation}
For the observation model $y=h(v)+\eta$ with $\eta\sim\normal(0,\Gamma)$,
\begin{equation}
\label{eq:operator_Q}
    \sQ \pi(v,y) = \frac{1}{\sqrt{(2\pi)^{d_y} \det \Gamma}} 
    \exp\!\left(-\tfrac{1}{2} \|y - h(v)\|_\Gamma^2 \right) \pi(v).
\end{equation}
The conditioning operator $\sB$ maps a joint state--observation law and an observed value to the conditional state law. When $\rho_{j+1}$ admits a density, denoted by the same symbol, $\sB(\rho_{j+1};y_{j+1}^\dagger)$ is the probability measure whose density is
\begin{equation}
\label{eq:operator_B}
    \sB(\rho_{j+1};y_{j+1}^\dagger)(v) = 
    \frac{\rho_{j+1}(v,y_{j+1}^\dagger)}
    {\int_{\bbR^{d_v}} \rho_{j+1}(u,y_{j+1}^\dagger)\,\mathrm{d}u}.
\end{equation}
We may then summarize~\eqref{eq:prob_map_evo} as
\begin{equation}
\label{eq:q_update0}
    \pi_j
    =
    \sB\bigl(\sQ\sP \pi_{j-1};y_j^\dag\bigr),
    \qquad
    j\in[J].
\end{equation}
This is a nonlinear Markov process in the sense that the probability distribution
$\pi_{j}$ is found from a nonlinear operator applied to $\pi_{j-1}.$
\paragraph{Mean-field state-space evolution}
The nonlinear operator $\sB$ \eqref{eq:operator_B} cannot typically be computed in closed form, and numerical computations are often unfeasible for large $d_v$ due to the high-dimensional integral involved. To implement the analysis step indirectly, a mean-field state-space evolution is introduced whose law effects the desired conditioning: 
\begin{subequations}
\label{eq:meta_state}
\begin{align}
    &\text{Prediction:} &\quad \hat v_{j+1} &= \Psi(v_j) + \xi_j, \quad \xi_j\sim\normal(0,\Sigma), \label{eq:state_pred}\\
    &\text{Extend to observation:} &\quad \hat y_{j+1} &= h(\,\hat v_{j+1}) + \eta_{j+1}, \quad \eta_{j+1}\sim\normal(0,\Gamma), \label{eq:state_extend}\\
    &\text{Analysis:} &\quad v_{j+1} &= \fB(\hat v_{j+1},\hat y_{j+1};y_{j+1}^\dagger,\rho_{j+1}). \label{eq:state_analysis}
\end{align}
\end{subequations}
Here, the initial state is drawn according to $v_0\sim \normal(m_0, C_0)$. By construction of \eqref{eq:state_pred}–\eqref{eq:state_extend} and the definitions \eqref{eq:operator_P}–\eqref{eq:operator_Q}, we have
$\rho_{j+1}=\Law\bigl((\hat v_{j+1},\hat y_{j+1})\bigr)$. 

In \eqref{eq:state_analysis}, $\fB$ is a vector-valued map acting on the joint state–observation space for the analysis step, parameterized by the observation $y^\dag_{j+1}$ and by the joint measure $\rho_{j+1}$:
\begin{equation}
    \fB(\,\placeholder\,,\,\placeholder\,;\,y_{j+1}^\dagger,\rho_{j+1}): (v,y) \mapsto \fB(v,y;\,y_{j+1}^\dagger,\rho_{j+1})\in\bbR^{d_v}.
\end{equation}
The analysis operator $\sB$ acting on the probability measures is induced by $\fB$ via pushforward:
\begin{equation}
    \sB(\rho_{j+1};y_{j+1}^\dagger)=
\fB(\,\placeholder\,,\,\placeholder\,;\,y_{j+1}^\dagger,\rho_{j+1})_\sharp \,\rho_{j+1}.
\end{equation}
With this construction, if $\pi_j=\Law(v_j)$, then $\pi_{j+1}=\Law(v_{j+1})$. $\fB$ is a transport map, transforming samples from the joint state--observation distribution $\rho_{j+1}$ to ones from the filtering distribution $\pi_{j+1}$. Since it depends on $\rho_{j+1}$, it is a mean field map.

\paragraph{An example: the mean-field ensemble Kalman filter}
A concrete instance of the measurable mapping $\fB$ is the analysis step of the mean-field ensemble Kalman filter \citep{calvello_ensemble_2025}, obtained by restricting $\fB$ to the class of affine maps. Keeping \eqref{eq:state_pred}–\eqref{eq:state_extend} unchanged, define
\begin{equation}
\label{eq:kf_affine_map}
v_{j+1} \;=\; \fB_{\mathrm{KF}}(\hat v_{j+1},\hat y_{j+1};y_{j+1}^\dagger,\rho_{j+1})
\;=\; \hat v_{j+1} + K_{j+1}\bigl(y_{j+1}^\dagger - \hat y_{j+1}\bigr),
\end{equation}
with gain
\begin{equation}
\label{eq:kf_gain_general}
K_{j+1} \;=\; \hat C^{vh}_{j+1}\bigl(\hat C^{hh}_{j+1}+\Gamma\bigr)^{-1},
\end{equation}
where the predictive moments are taken under $\hat\pi_{j+1}$:
\begin{subequations}\label{eq:kf_predictive_moments}
\begin{align}
\bar v_{j+1} &= \bbE^{\hat v\sim\hat\pi_{j+1}}\hat v, 
\qquad
\bar h_{j+1} = \bbE^{\hat v\sim\hat\pi_{j+1}} h(\,\hat v\,),\\
\hat C^{vh}_{j+1} &= \bbE^{\hat v\sim\hat\pi_{j+1}}\!\left[(\hat v-\bar v_{j+1})\otimes\bigl(h(\,\hat v\,)-\bar h_{j+1}\bigr)\right],\\
\hat C^{hh}_{j+1} &= \bbE^{\hat v\sim\hat\pi_{j+1}}\!\left[\bigl(h(\,\hat v\,)-\bar h_{j+1}\bigr)\otimes\bigl(h(\,\hat v\,)-\bar h_{j+1}\bigr)\right].
\end{align}
\end{subequations}
Equivalently, one may write $K_{j+1}=\hat C^{vh}_{j+1}\bigl(\hat C^{yy}_{j+1}\bigr)^{-1}$ with $\hat C^{yy}_{j+1}=\hat C^{hh}_{j+1}+\Gamma$, making the dependence on $\rho_{j+1}$ explicit through the synthetic observation $\hat y_{j+1}$. In the linear–Gaussian setting (if both $\Psi$ and $h$ are linear), this affine choice yields the Bayesian posterior, and the law of $v_j$ is governed by the classical Kalman filter \citep{calvello_ensemble_2025}.

\subsection{Proper Scoring Rules}\label{subsec:scoring_rules}

Probabilistic scoring rules assess how well a forecast distribution
approximates a verification distribution that is only observed through 
a single sample. Let $\mu\in\cP(\bbR^d)$ denote the forecast distribution and $\pp\in\cP(\bbR^d)$ denote the verification distribution. A scoring rule is a function $\sS:\cP(\bbR^d)\times \bbR^d \to \bbR$,
which assigns a real-valued score to a realized sample $z\sim \pp$ given the forecast $\mu$. We adopt the convention that scores are negatively oriented: smaller values indicate better forecasts. To compare $\mu$ and $\pp$ in expectation, we define the expected score
\begin{equation}
\bar{\sS}(\mu,\pp)
:= \mathbb{E}^{z\sim \pp}\bigl[\sS(\mu,z)\bigr]
= \int_{\mathbb{R}^d} \sS(\mu,z)\pp(\rd z) .
\end{equation}

\begin{definition}[Proper / Strictly proper scoring rule]\label{def:strictly_proper_scoring_rule}
A scoring rule $\sS:\cP(\bbR^d)\times\bbR^d\to\bbR$ is called proper if for every verification distribution $\pp\in\cP(\bbR^d)$,
\begin{equation}
\label{eq:proper_scoring_rule}
    \bar{\sS}(\pp,\pp)\;\le\;\bar{\sS}(\mu,\pp)\quad\text{for all } \mu\in\cP(\bbR^d).
\end{equation}
Furthermore, $\sS$ is called strictly proper if it is proper and, moreover, for every $\pp\in\cP(\bbR^d)$,
\[
\bar{\sS}(\mu,\pp) \;=\; \bar{\sS}(\pp,\pp)
\quad\Longrightarrow\quad
\mu=\pp .
\]
Equivalently, for each fixed $\pp$, the unique global minimizer of $\bar{\sS}(\cdot,\pp)$ over $\cP(\bbR^d)$ is attained at $\mu=\pp$.
\end{definition}

A strictly proper scoring rule makes the quantity
\begin{equation}
 \label{eq:sdiv}
\bar{\sS}(\mu,\pp)-\bar{\sS}(\pp,\pp)
\end{equation}
a statistical divergence between $\mu$ and $\pp$.
The abstract framework above admits many concrete choices of $\sS$. In what follows, we focus on the energy score (Definition~\ref{def:energy_score}), which is built from pairwise Euclidean distances.

\begin{restatable}[Energy score]{definition}{definitionEnergyScore}\label{def:energy_score}
Fix $\beta\in(0,2)$. To ensure the score is finite, the domain is restricted to $\cP_\beta(\bbR^d)\times \bbR^d$. The \emph{energy score} (ES) $\sS^\mathrm{ES}_{\beta}: \cP_\beta(\bbR^d)\times \bbR^d\to\bbR_+$ associated with a forecast $\mu\in\cP_\beta(\bbR^d)$ and a realization $z\in\bbR^d$ is
\begin{equation}
\label{eq:ES}
\sS^\mathrm{ES}_{\beta}(\mu,z)
:= \mathbb{E}^{x\sim\mu}\bigl[\|x-z\|^{\beta}\bigr]
-\frac{1}{2}\,\mathbb{E}^{x,x'\sim\mu}\bigl[\|x-x'\|^{\beta}\bigr].
\end{equation}
For $\beta\in(0,2)$, $\sS^\mathrm{ES}_{\beta}$ is strictly proper \citep{gneiting2007strictly}. %
\end{restatable}

In the one-dimensional case $d=1$ with $\beta=1$, the energy score reduces to the commonly used continuous ranked probability score (CRPS) \citep{bach2024inverse}. Beyond the energy score and CRPS, the family of strictly proper scoring rules is extensive. We refer to \cite{gneiting2014probabilistic, jordan2019evaluating} for comprehensive reviews. A classical example is the logarithmic score, defined as $\sS^{\mathrm{LS}}(\mu, z) = -\log p_\mu(z)$, which is local, depending solely on the probability assigned to the realized event $z$, ignoring the distance to other probable outcomes \citep{du_beyond_2021}. In contrast, the energy score belongs to the class of non-local or distance-sensitive scoring rules \citep{szekely2013energy}. This distance-based approach is generalized by the class of kernel scores, obtained by replacing Euclidean distances with a negative definite kernel \citep{gneiting2007strictly,muandet2017kernel}. Kernel scores quantify the discrepancy between the forecast and observation via embeddings in a reproducing kernel Hilbert space. Their expectation is mathematically equivalent to the squared maximum mean discrepancy \citep{gretton2012kernel}.

Critically, kernel scores, which include the energy score and CRPS, since they can be written using expectations over the forecast distribution, can be evaluated directly on empirical measures \citep{bach2024inverse}, as we will later see in Subsection~\ref{ssec:methods}. This makes such scores easily computable for evaluating ensemble forecasts, which motivate their use herein. Other scores, such as the logarithmic score, may require density estimation to be evaluated on empirical measures \citep{brocker_ensemble_2008}.

%% file: sections/learning_the_true_filter.tex
\section{Learning to Approximate the True Filter with Strictly Proper Scoring Rules}\label{sec:learning_true_filter}
In this section, we develop a framework for learning the data assimilation analysis operator using a probabilistic loss with synthetic true trajectories but without access to the true filtering distribution. Subsection~\ref{subsec:fd_and_psr} shows that strictly proper scoring rules allow the filtering distributions to be recovered using only the true states as training targets. We then introduce a unified family of scoring-rule objectives indexed by the time horizon, the number of independent training trajectories, and the ensemble size. When the number of training trajectories is infinite we refer to the population limit; when the number of ensemble members is infinite we refer to the mean-field limit. Subsection~\ref{ssec:bridge_pop_empirical} connects,
under regularity and ergodicity assumptions, the population and finite-trajectory mean-field objectives through a consistency result. Finally, Subsection~\ref{subsec:architecture_training} describes the neural particle analysis map and its practical training procedure. Auxiliary results used in the proofs are collected in Appendix~\ref{app:lemmas}.

\subsection{Filtering Distribution and Strictly Proper Scoring Rules}\label{subsec:fd_and_psr}
For a strictly proper scoring rule $\sS$, we show that the true filtering distributions are minimizers of a loss based on $\sS$ while using only the true states as the training data (Propositions~\ref{prop:filter_score} and \ref{prop:opt_score_to_filter}). 

\begin{proposition}\label{prop:filter_score}
Consider a strictly proper scoring rule
$\sS:\cP(\bbR^{d_v})\times\bbR^{d_v}\to\bbR$
(Definition~\ref{def:strictly_proper_scoring_rule}), and assume that all expected scores below are finite. Let $\pi_j := \Law\bigl(v_j^\dag \mid Y_j^\dag\bigr)$ denote the filtering distribution, viewed as a random probability measure determined by the observation history $Y_j^\dag$. Let $q(Y_j^\dag)$ be any $Y_j^\dag$-measurable random probability measure on $\bbR^{d_v}$. Then, with $\bbE^{(v_j^\dag,Y_j^\dag)}$ denoting expectation under the joint dynamics--observation model \eqref{eq:DA_setting}, we have
\begin{equation}
\label{eq:prop1_ineq}
\bbE^{(v_j^\dag,Y_j^\dag)}
\left[
    \sS\bigl(q(Y_j^\dag),v_j^\dag\bigr)
\right]
\geq
\bbE^{(v_j^\dag,Y_j^\dag)}
\left[
    \sS(\pi_j,v_j^\dag)
\right].
\end{equation}
Equality holds, almost surely with respect to the distribution of $Y_j^\dag$,
if and only if $q(Y_j^\dag)=\pi_j$.
\end{proposition}

\begin{proof}
Let $\bar{\sS}(\mu, \pp) := \bbE^{y \sim \pp}[\sS(\mu, y)]$ denote the expected score.
Since $(v_j^\dag,Y_j^\dag)\sim \bbP(v_j^\dag,Y_j^\dag) = \bbP(Y_j^\dag)\bbP(v_j^\dag\mid Y_j^\dag)$, and $\pi_j =\Law\bigl(v_j^\dag \mid Y_j^\dag\bigr)$,
we have
\begin{subequations}\label{eq:conditional_score_identity}
\begin{equation*}
\begin{aligned}
&\bbE^{(v_j^\dag,Y_j^\dag)} \bigl[ \sS\bigl(q(\cdot; Y_j^\dagger), v_j^\dagger\bigr) \bigr]
= \bbE^{Y_j^\dag} \left[ \bbE^{v_j^\dag | Y_j^\dag} \bigl[ \sS\bigl(q(\cdot; Y_j^\dagger), v_j^\dagger\bigr) \bigr] \right] 
= \bbE^{Y_j^\dag} \left[ \bar{\sS}\bigl(q(\cdot; Y_j^\dagger), \pi_j\bigr) \right],\\
&\bbE^{(v_j^\dag,Y_j^\dag)} \bigl[ \sS(\pi_j, v_j^\dagger) \bigr]
= \bbE^{Y_j^\dag} \left[ \bbE^{v_j^\dag | Y_j^\dag} \bigl[ \sS\bigl(\pi_j, v_j^\dagger\bigr) \bigr] \right] = \bbE^{Y_j^\dag} \left[ \bar{\sS}(\pi_j, \pi_j) \right].
\end{aligned}
\tag{\theparentequation}
\end{equation*}
\end{subequations}
The inequality \eqref{eq:prop1_ineq} and the equality condition follow the Definition~\ref{def:strictly_proper_scoring_rule} of the strictly proper scoring rule $\sS$.
\end{proof}

\paragraph{A unified notation for scoring-rule objectives}
Consider a parameterized analysis operator $\sB_\theta:\cP(\bbR^{d_v+d_y})\times\bbR^{d_y}\to\cP(\bbR^{d_v})$, $\theta \in \Theta$, approximating the conditioning operator $\sB$ in \eqref{eq:operator_B}. For an observation history $Y_j^\dag$, we write $q_j^\theta(Y_j^\dag)$ for the random probability measure produced by the parameterized mean-field filter at time $j$ with parameters $\theta$. The recursion is initialized by
\begin{equation}
\label{eq:q_initialization}
    q_0^\theta(Y_0^\dag)
    =
    \pi_0
    =
    \Law(v_0^\dag),
    \qquad
    Y_0^\dag=\varnothing.
\end{equation}
Then, imposing the prediction--analysis structure of the true filter \eqref{eq:q_update0}, the parameterized distribution at time $j$ is
\begin{equation}
\label{eq:q_update}
    q_j^\theta(Y_j^\dag)
    =
    \sB_\theta\bigl(\sQ\sP q_{j-1}^\theta(Y_{j-1}^\dag);y_j^\dag\bigr),
    \qquad
    j\in[J].
\end{equation}
Hence, the parameterized analysis operator $\sB_\theta$ gives a sequence of approximations of the filtering distribution $\{q_j^\theta(Y_j^\dag)\}_{j=1}^J$ by evolving \eqref{eq:q_update}. The same structure is also used in \cite{luk_learning_2024,bach_learning_2026}. 

The Proposition~\ref{prop:filter_score} and the preceding discussion about the mean-field
ensemble filter, suggests minimization of the following objective function,
in the population limit:
\begin{equation}
    \label{eq:basic_obj}
    \mathcal{L}_{J,\infty}(\theta)=\frac{1}{J}
\sum_{j=1}^J
\bbE^{(v_j^\dag,Y_j^\dag)}
\left[
\sS\bigl(q_j^{\theta}(Y_j^\dag),v_j^\dag\bigr)
\right].
\end{equation}
In practice we cannot work in the population limit.
Our training setting is based on the following Data Assumption~\ref{da:training_data}.

\begin{dataassumption}\label{da:training_data}
For a training length $J\in\bbN$, we consider $M\in\bbN$ independent training trajectories $\{(V_{J,m}^\dag,Y_{J,m}^\dag)\}_{m=1}^M$ generated according to \eqref{eq:DA_setting}, where
\begin{equation}
\label{eq:training_trajectory_m}
    V_{J,m}^\dag=\{v_{0,m}^\dag,\ldots,v_{J,m}^\dag\},
    \qquad
    Y_{J,m}^\dag=\{y_{1,m}^\dag,\ldots,y_{J,m}^\dag\}.
\end{equation}
Here $v_{0,m}^\dag\sim q_\infty$, $y_0^\dag$ is not needed, and $q_\infty$ is the invariant measure of the forward dynamical model; see Appendix~\ref{app:ergodicity_state} for the existence of $q_\infty$.
\end{dataassumption}

Nor, in practice, can we work in the mean-field limit. We use a unified notation $\mathcal{L}_{J,M,N}(\theta)$ for the scoring-rule objective, where $J$ is the time horizon, $M$ is the number of independent training trajectories used to approximate the data expectation, and $N$ is the ensemble size. Throughout the finite-horizon definitions, $J\in\bbN$, while $M,N\in\bbN\cup\{\infty\}$.

For $N\in\bbN$, let $q_{j,m}^{\theta,N}(Y_{j,m}^\dag)$ denote the empirical filtering measure generated by the finite-particle recursion introduced below in \eqref{eq:learn_mf}. For $N=\infty$, we set $q_{j,m}^{\theta,\infty}(Y_{j,m}^\dag)=q_j^\theta(Y_{j,m}^\dag).$
Then objective function $\mathcal{L}_{J,M,N}(\theta)$ is defined by
\begin{equation}
\label{eq:unified_loss_JMN}
\mathcal{L}_{J,M,N}(\theta)
=
\begin{cases}
\displaystyle
\frac{1}{JM}
\sum_{m=1}^M
\sum_{j=1}^J
\sS\bigl(q_{j,m}^{\theta,N}(Y_{j,m}^\dag),v_{j,m}^\dag\bigr),
& M\in\bbN,\ N\in\bbN,
\\
\displaystyle
\frac{1}{J}
\sum_{j=1}^J
\bbE^{(v_j^\dag,Y_j^\dag)}
\left[
\sS\bigl(q_j^{\theta,N}(Y_j^\dag),v_j^\dag\bigr)
\right],
& M=\infty,\ N\in\bbN,
\\
\displaystyle
\frac{1}{JM}
\sum_{m=1}^M
\sum_{j=1}^J
\sS\bigl(q_{j,m}^{\theta}(Y_{j,m}^\dag),v_{j,m}^\dag\bigr),
& M\in\bbN,\ N=\infty,
\\
\displaystyle
\frac{1}{J}
\sum_{j=1}^J
\bbE^{(v_j^\dag,Y_j^\dag)}
\left[
\sS\bigl(q_j^{\theta}(Y_j^\dag),v_j^\dag\bigr)
\right],
& M=\infty,\ N=\infty.
\end{cases}
\end{equation}
When $M=\infty$, the expectation is taken over the data-generating law of $(v_j^\dag,Y_j^\dag)$. When $M\in\bbN$, the objective averages over finitely many independent training trajectories. When $N = \infty$ the objective uses the mean-field filter; when $N$ is finite this filter is approximated by an interacting particle system. For the mean-field objectives, we use the shorthand
\begin{equation}
\label{eq:mean_field_loss_shorthand}
\mathcal{L}_{J,M}(\theta)
:=
\mathcal{L}_{J,M,\infty}(\theta),
\end{equation}
linking the objective function $\mathcal{L}_{J,\infty}$ defined in \eqref{eq:basic_obj} to the final item in \eqref{eq:unified_loss_JMN}.
The following recovery result concerns the objective $\mathcal{L}_{J,\infty}$.

\begin{proposition}\label{prop:opt_score_to_filter}
Consider minimizing the full-expectation mean-field objective $\mathcal{L}_{J,\infty}(\theta)$ defined in \eqref{eq:basic_obj}, subject to the recursive update \eqref{eq:q_update} initialized at $q_0^\theta(Y_0^\dag)=\pi_0$. Under the realizability assumption that there exists $\theta_B\in\Theta$ such that $\sB_{\theta_B}=\sB$ for $\sB$ defined in \eqref{eq:operator_B}, any minimizer
\begin{equation}
\label{eq:theta_J_star_full_expectation}
\theta_J^\ast\in\argmin_{\theta\in\Theta}\mathcal{L}_{J,\infty}(\theta)
\end{equation}
recovers the filtering distributions for every $j\in [J]$:
\begin{equation}
\label{eq:theta_J_star_filter_recovery}
q_j^{\theta_J^\ast}(Y_j^\dag)=\pi_j \quad Y_j^\dag\text{-a.s.}
\end{equation}
Consequently, $\sB_{\theta_J^\ast}$ generates the same analysis distributions as the true Bayesian analysis operator $\sB$ along the recursively generated inputs and observation histories, almost surely, through time $J$. This does not identify the analysis operator globally.
\end{proposition}

\begin{proof}
By Proposition~\ref{prop:filter_score}, we have $\bbE^{(v_j^\dag,Y_j^\dag)}\left[
\sS\bigl(q_j^\theta(Y_j^\dag),v_j^\dag\bigr)\right]\geq\bbE^{(v_j^\dag,Y_j^\dag)}
\left[\sS(\pi_j,v_j^\dag)\right],$
with equality if and only if $q_j^\theta(Y_j^\dag)=\pi_j$ $Y_j^\dag$-a.s.
By the realizability assumption, $\exists\theta_B\in\Theta$ such that $\sB_{\theta_B}=\sB$. By induction, we have $q_j^{\theta_B}(Y_j^\dag)=\pi_j$ a.s. for all $j\in[J]$. Thus, $\theta_B$ attains the lower bound in $\mathcal{L}_{J,\infty}$. 

For any minimizer $\theta_J^\ast$ of $\mathcal{L}_{J,\infty}$, we have $\mathcal{L}_{J,\infty}(\theta_J^\ast)=\mathcal{L}_{J,\infty}(\theta_B).$ By the equality condition in Proposition~\ref{prop:filter_score}, this implies
\begin{equation}
\label{eq:prop2_minimizer_filter_recovery}
    q_j^{\theta_J^\ast}(Y_j^\dag)=\pi_j\quad
    Y_j^\dag\text{-a.s. for every }j\in[J].
\end{equation}
\end{proof}

In practice, we approximate the mean-field filter with finitely many particles and the data expectation with finitely many trajectories. Theorem~\ref{thm:consistency} establishes a long-time mean-field result at a realizable exact-filter parameter for any fixed trajectory count $M$, including $M=1$. Our numerical experiments use $M=8192$ trajectories of length $J=60$, unless otherwise specified in Section~\ref{sec:numerical_exp}.
 
 For each training trajectory indexed by $m\in[M]$, we use a set of particles $\{\vn_j\}_{n=1}^N$, called the ensemble, to approximate the filtering distribution $\pi_j$. Here the subscript $m$ is omitted.
To enable ensemble-based learning, we distinguish between the measure-valued analysis operator and its particle-wise implementation. 
We define the analysis transport map $\fB_\theta: \bbR^{d_v}\times\bbR^{d_y} \times \bbR^{d_y} \times \cP(\bbR^{d_v + d_y}) \to \bbR^{d_v}$ such that $\sB_\theta: 
\cP(\bbR^{d_v+d_y})\times\bbR^{d_y}\to
\cP(\bbR^{d_v})$ is induced by applying $\fB_\theta$ to each particle:
\begin{equation}
    \sB_\theta(\rho; y^\dag) \;=\; \fB_\theta(\placeholder, \placeholder; y^\dag, \rho)_\sharp \,\rho.
\end{equation}
Keeping the prediction and extension steps the same as in \eqref{eq:meta_state}, the ensemble update is explicitly defined as:
\begin{subequations}
\label{eq:learn_mf}
\begin{align}
    &\text{Predict:} &\quad \hvn_{j+1} &= \Psi(\vn_j) + \xi^{(n)}_j, \quad \xi^{(n)}_j\sim\normal(0,\Sigma), \label{eq:state_pred_sec3}\\
    &\text{Extend to observation:} &\quad \hyn_{j+1} &= h(\,\hvn_{j+1}) + \eta^{(n)}_{j+1}, \quad \eta^{(n)}_{j+1}\sim\normal(0,\Gamma), \label{eq:state_extend_sec3}\\
    &\text{Analysis:} &\quad \vn_{j+1} &= \fB_\theta(\hvn_{j+1},\hyn_{j+1};y_{j+1}^\dagger,\rho^N_{j+1}),\label{eq:learned_state_analysis}
\end{align}
\end{subequations}
where $\rho^N_{j+1}$ is the empirical joint measure $\rho^N_{j+1} = \frac1N \sum_{n=1}^N\delta_{(\hvn_{j+1},\hyn_{j+1})}.$ This recursion is initialized using $\{\vn_0\}_{n=1}^N$ drawn from $\normal(v_0^\dag,C_0)$. In this paper, unless otherwise specified, we choose $C_0=I_{d_v}$. This experimental initialization, centered on the true initial state, need not be the same as the prior $\pi_0$ used in \eqref{eq:q_initialization}, so the theoretical recovery result does not directly justify this protocol. Appendix~\ref{sssec:L63_shifted_initial} examines sensitivity to perturbations of the initial ensemble mean and scales in Lorenz~'63.

\begin{remark}[Permutation invariance and varying ensemble size]\label{rmk:permutation_inv}
    The formulation of $\fB_\theta$ in \eqref{eq:learned_state_analysis} highlights a critical structural advantage. A naive approach treating the ensemble as a matrix in $\bbR^{N \times (d_v+d_y)}$ would result in an input dimension dependent on the ensemble size $N$, restricting the model to a fixed $N$ and requiring manual enforcement of permutation invariance.
    In contrast, by treating the ensemble as an empirical measure $\rho^N \in \cP(\bbR^{d_v+d_y})$, the domain of $\fB_\theta$ becomes the fixed space of probability measures, independent of $N$. This abstraction naturally accommodates varying ensemble sizes during training and inference. Furthermore, since the measure $\rho^N$ is inherently invariant to the order of particles, the resulting particle-wise map satisfies permutation invariance in its ensemble argument for a fixed query particle and observation by design. In our implementation, we parameterize $\fB_\theta$ using a transformer architecture (Appendix~\ref{app:attn_measures}), which naturally operates on set-valued inputs and respects this ensemble structure.
\end{remark}

Based on Data Assumption~\ref{da:training_data}, the practical training objective is $\mathcal{L}_{J,M,N}(\theta)$ with $M,N\in\bbN$. Compared with $\mathcal{L}_{J,\infty}$, this objective uses finitely many realized training trajectories and a finite particle ensemble. The parameter $M$ controls the Monte Carlo approximation of the data expectation, while $N$ controls the particle approximation of the mean-field filter.

\subsection{Bridging the Full-Expectation and Finite-Trajectory Mean-Field Objectives}\label{ssec:bridge_pop_empirical}

In this subsection, we work in the mean-field regime $N=\infty$. We fix a realizable exact-filter parameter and show that, under ergodicity, the finite-trajectory mean-field objectives $\mathcal{L}_{J,M}$ for all $M\in\bbN$ and the full-expectation mean-field objective $\mathcal{L}_{J,\infty}$ share the same long-time limit. The relationship between these objectives is summarized in Figure~\ref{fig:commutative_loss_monte_carlo}. 

\begin{remark}[Scope of finite-trajectory and finite-ensemble approximations]\label{rmk:finite_M_N_scope}
For each fixed $J\in\bbN$, $N\in\bbN\cup\{\infty\}$, and $\theta\in\Theta$, the finite-trajectory finite-ensemble objective $\mathcal{L}_{J,M,N}(\theta)$ is a Monte Carlo approximation of $\mathcal{L}_{J,\infty,N}(\theta)$. Under the usual integrability assumptions, one expects
\begin{equation}
\label{eq:fixed_J_MC_limit_finite_N}
\lim_{M\to\infty}
\mathcal{L}_{J,M,N}(\theta)
=
\mathcal{L}_{J,\infty,N}(\theta)
\end{equation}
for fixed $J$, $N$, and $\theta$. These fixed-horizon Monte Carlo statements are different from the long-time consistency result proved below. Theorem~\ref{thm:consistency} works at the mean-field level $N=\infty$ and uses temporal averaging as $J\to\infty$. It shows that, at a realizable exact-filter parameter, $\mathcal{L}_{J,M}$ converges to the same limit as $\mathcal{L}_{J,\infty}$ for every $M\in\bbN$ on a common probability-one event. It does not assert consistency of minimizers of the finite-ensemble empirical objective $\mathcal{L}_{J,M,N}$ with $N\in\bbN$, nor does it address convergence of neural-network optimization. Establishing such stronger statements would require additional assumptions, such as uniform-in-time propagation of chaos for the parameterized filter class and appropriate uniform laws of large numbers over $\Theta$.
\end{remark}

\begin{figure}[ht]
    \centering
    \begin{equation}
\label{eq:loss_bridge_diagram}
    \begin{tikzcd}[row sep=huge, column sep=large, cells={nodes={scale=1.1}}]
        \mathcal{L}_{1,M}
        \arrow[d, "{\substack{\text{Monte Carlo}\\ M\to\infty}}"{left}]
        &
        \mathcal{L}_{2,M}
        \arrow[d, "{\substack{\text{Monte Carlo}\\ M\to\infty}}"{left}]
        &
        \cdots
        &
        \mathcal{L}_{J,M}
        \arrow[r, "{\substack{J\to\infty,\text{ a.s.}}}"{above}, "{\text{Thm~\ref{thm:consistency}}}"{below}]
        \arrow[d, "{\substack{\text{Monte Carlo}\\ M\to\infty}}"{left}]
        &
        \mathcal{L}_{\star}
        \arrow[d, equal, "{\text{Thm~\ref{thm:consistency}}}"{right}]
        \\
        \mathcal{L}_{1,\infty}
        &
        \mathcal{L}_{2,\infty}
        &
        \cdots
        &
        \mathcal{L}_{J,\infty}
        \arrow[r, "{J\to\infty}"{above}, "{\text{Thm~\ref{thm:consistency}}}"{below}]
        &
        \mathcal{L}_{\star}
    \end{tikzcd}
\end{equation}
    \caption{
        Commutative diagram relating the finite-trajectory mean-field objectives $\mathcal{L}_{J,M}$ and the full-expectation mean-field objectives $\mathcal{L}_{J,\infty}$. For each fixed $J$, the vertical arrows represent the Monte Carlo limit $M\to\infty$. The upper horizontal arrow denotes the long-time limit $J\to\infty$ almost surely in the finite trajectory count $M$. The lower horizontal arrow is the corresponding full-expectation mean-field limit. Both horizontal limits and the identification of $\mathcal{L}_{\star}$ are justified by Theorem~\ref{thm:consistency}.
    }
    \label{fig:commutative_loss_monte_carlo}
\end{figure}

Let $\theta_B\in\Theta$ be a realizable parameter such that the recursion \eqref{eq:q_update}, initialized at $\pi_0$, satisfies
\begin{equation}
\label{eq:theta_B_exact_filter}
q_j^{\theta_B}(Y_j^\dag)
=
\pi_j
\quad
Y_j^\dag\text{-a.s. for every }j\in\bbN.
\end{equation}
Such a parameter exists under the realizability assumption in Proposition~\ref{prop:opt_score_to_filter}. We study the limits of $\mathcal{L}_{J,M}(\theta_B)$ for $M\in\bbN$ and $\mathcal{L}_{J,\infty}(\theta_B)$ as $J\to\infty$.

To establish the convergence of the time-averaged losses, we first formulate the state space. Recall that $\cP_1(\bbR^{d_v})$ denotes the space of probability measures on $\bbR^{d_v}$ with a finite first moment. Let $\Omega := \cP_1(\bbR^{d_v}) \times \bbR^{d_v}$ denote the product space of the filter distribution and the state. We define an unbounded metric on $\Omega$ between $Z = (q,v)$ and $Z' = (q', v')$ by
\begin{equation}
\label{eq:d_omega}
    \mathrm{D}_\Omega(Z,Z') := W_1(q,q') + \|v-v'\|.
\end{equation}
We define the augmented filter--state process $\{Z_j\}_{j\in\bbZ_+}$ on $\Omega$ evaluated at the optimum $\theta_B$ as:
\begin{equation}
\label{eq:augmented_state}
    Z_j := \bigl(\pi_j, \, v_j^\dag\bigr) = \bigl(q_j^{\theta_B},\vd_j\bigr),
\end{equation}
where $v_j^\dag$ is the true state at time $j$, and the filtering distribution $\pi_j = \Law\bigl(v_j^\dag\mid Y_j^\dag\bigr)$ is viewed as a random probability measure driven by the history of observations $Y_j^\dag$. The sequence $\{Z_j\}_{j\in\bbZ_+}$ forms a time-homogeneous  Markov chain on $\Omega$ since the observation $y_{j}^\dag$ is generated internally using $v_{j}^\dag$ and noise $\eta_{j}^\dag$: for $j\in[J]$ we have, from \eqref{eq:q_update0} and \eqref{eq:DA_setting},
\begin{subequations}\label{eq:q_update1}
\begin{align}
    \pi_j
    &=
    \sB\bigl(\sQ\sP \pi_{j-1};h(v_{j}^\dag) + \eta_{j}^\dag\bigr),\\
    v_{j}^\dag &= \Psi(v_{j-1}^\dag) + \xi_{j-1}^\dag.
\end{align}
\end{subequations}
We denote the marginal law of $Z_j$ by $\mu_j := \Law(Z_j)\in\cP_1(\Omega)$ for $j\in\bbZ_+$.

We formulate a unified assumption encapsulating the required regularity of the scoring rule and the ergodicity of the augmented filtering sequence.

\begin{restatable}[Regularity and ergodicity]{assumption}{dataassumptionStability}\label{da:ergodicity_on_state_filter}
We assume the following properties for the scoring rule and the Markov process on $\Omega$:
\begin{enumerate}[label=(\roman*)]
    \item \textbf{Regularity of the scoring rule:} The scoring rule $\sS:\Omega\to\bbR_+$ is Lipschitz continuous on the metric space $(\Omega, \mathrm{D}_\Omega)$ with constant $C_\sS$. That is, for any $Z, Z' \in \Omega$, $|\sS(Z) - \sS(Z')| \leq C_\sS \mathrm{D}_\Omega(Z,Z')$.
    \item \textbf{Ergodicity of the augmented process:} The process $\{Z_j\}_{j\in\bbZ_+}$ is a time-homogeneous ergodic Markov chain on $\Omega$. It possesses a unique invariant measure $\mu_\infty \in \cP_1(\Omega)$ such that convergence holds in the Wasserstein metric induced by $\mathrm{D}_\Omega$:
    \begin{equation}
\label{eq:def_mu_inf}
        \lim_{j \to \infty} W_{\mathrm{D}_\Omega}(\mu_j, \mu_\infty) = 0.
\end{equation}
\end{enumerate}
\end{restatable}

\begin{remark}[Verification of the assumptions]\label{rmk:assumption_verification}
Assumption~\ref{da:ergodicity_on_state_filter} specifies the theoretical prerequisites for our consistency result. We discuss their verification below; ergodicity requires the additional filter assumptions in Appendix~\ref{app:ergodicity}:
\begin{itemize}
    \item \textbf{Scoring rule regularity:} Our choice of the energy score with $\beta=1$ satisfies condition (i): as proven in Lemma~\ref{lem:ES1_lipschitz} in Appendix~\ref{app:lemmas}, the energy score $\sS^{\mathrm{ES}}_1$ is Lipschitz continuous on $(\Omega, \mathrm{D}_\Omega)$ with a constant $C_\sS = 2$. Theorem~\ref{thm:consistency} is formulated for scores satisfying Assumption~\ref{da:ergodicity_on_state_filter}, whereas the scoring-rule regularity verification provided here is specific to the energy score with $\beta=1$. Other values of $\beta\in(0,2)$ may also be considered for the energy score, but applying Theorem~\ref{thm:consistency} to these choices requires separate verification of Assumption~\ref{da:ergodicity_on_state_filter}(i).
    \item \textbf{Ergodicity of the dynamics:} Condition (ii) holds for the dynamical models considered in this work under the additional assumptions stated in Appendix~\ref{app:ergodicity}. As established in Theorem~\ref{thm:exact_ergodicity_unbounded} and Remark~\ref{rmk:models_verification} in Appendix~\ref{app:ergodicity}, the required convergence in the unbounded Wasserstein distance holds.

For the Lorenz ’63 and Lorenz ’96 settings (Subsections~\ref{ssec:L63} and~\ref{ssec:L96}), the argument relies not only on dissipativity and moment control, but also on non-degenerate system noise. In the case of the doubling-angle model (Subsection~\ref{ssec:doublin_angle_model}), we use the fact that its state space is compact. Furthermore, for all models the argument relies on a $W_1$-stability assumption for the exact filter. This stability condition is not proved here.
\end{itemize}
\end{remark}

Relying on Assumption~\ref{da:ergodicity_on_state_filter}, we present our main consistency theorem. This result establishes the horizontal convergence in Figure~\ref{fig:commutative_loss_monte_carlo}: for $\theta_B$ such that $\sB_{\theta_B}=\sB$, the finite-trajectory mean-field objectives $\mathcal{L}_{J,M}$ converge to the same limiting value as the population mean-field objective $\mathcal{L}_{J,\infty}$.

\begin{restatable}[]{theorem}{theoremConsistency}\label{thm:consistency}
Assume the realizability condition that there exists $\theta_B$ such that $\sB_{\theta_B}=\sB$. Under Assumption~\ref{da:ergodicity_on_state_filter}, for the invariant measure $\mu_\infty$ of the exact augmented process $\{Z_j\}_{j\in\bbZ_+}$ defined in \eqref{eq:def_mu_inf}, define
\begin{equation}
\label{eq:L_star_definition}
    \mathcal{L}_\star
    :=
    \int_{\Omega}
    \sS(q,v)
    \mu_\infty(\mathrm{d}q,\mathrm{d}v).
\end{equation}
Then $\mathcal{L}_\star<\infty$, and for $M\in\bbN$, we have
\begin{align}
\label{eq:thm_combined_limit}
    \lim_{J\to\infty}
    \mathcal{L}_{J,\infty}(\theta_B)
    &=
    \mathcal{L}_\star,\\
    \lim_{J\to\infty}
    \mathcal{L}_{J,M}(\theta_B)
    &=
    \mathcal{L}_\star
    \quad
    \text{a.s.},
\end{align}
where the finite-sample limit is almost sure with respect to the product law of these $M$ trajectories.
\end{restatable}

\begin{proof}
The proof characterizes the asymptotic limits of the full-expectation and multi-trajectory mean-field objectives, demonstrating that they converge to the same spatial expectation.

\paragraph{Step 1: Deterministic limit of the full-expectation objective}
Let 
\begin{equation}
    L_j(\theta):=\bbE^{(v_j^\dag,Y_j^\dag)}\left[\sS\bigl(q_j^{\theta}(Y_j^\dag),v_j^\dag\bigr)\right].
\end{equation}
The sequence $\{L_j(\theta_B)\}_{j\geq1}$ is deterministic. Since $q_j^{\theta_B}(Y_j^\dag)=\pi_j$, Lemma~\ref{lem:variable_change_augmented} gives
\begin{equation}
\label{eq:expected_step_score_augmented}
    L_j(\theta_B)
    =
    \bbE^{(q,v)\sim\mu_j}
    \left[
        \sS(q,v)
    \right]
    =
    \int_{\Omega}
    \sS(q,v)
    \mu_j(\mathrm{d}q,\mathrm{d}v).
\end{equation}
By Assumption~\ref{da:ergodicity_on_state_filter}, the score $\sS$ is Lipschitz continuous and $\lim_{j\to\infty}W_{\mathrm{D}_\Omega}(\mu_j,\mu_\infty)=0.$
Applying Lemma~\ref{lem:invariant_integrability_and_integral_convergence}, we deduce that $\mathcal{L}_\star<\infty$ and that the expected scores converge:
$\lim_{j\to\infty}L_j(\theta_B)=\int_{\Omega}\sS(q,v)\mu_\infty(\mathrm{d}q,\mathrm{d}v)=\mathcal{L}_\star.$ Since the sequence $\{L_j(\theta_B)\}_{j\in\bbZ_+}$ converges to the constant limit $\mathcal{L}_\star$, its time average (the Ces\`{a}ro mean) converges to the same limit: 
\begin{equation}
\label{eq:proof_full_expectation_limit}
    \lim_{J\to\infty}
    \mathcal{L}_{J,\infty}(\theta_B)
    =
    \lim_{J\to\infty}
    \frac{1}{J}
    \sum_{j=1}^J
    L_j(\theta_B)
    =
    \mathcal{L}_\star.
\end{equation}

\paragraph{Step 2: Ergodic limit of the multi-trajectory objective}
We start from the single-trajectory mean-field objective evaluated at the realizable exact-filter parameter, defined by (here the $M=1$ index is omitted)
\begin{equation}
\label{eq:single_trajectory_objective_exact_filter}
    \mathcal{L}_{J,1}(\theta_B)
    =
    \frac{1}{J}
    \sum_{j=1}^J
    \sS\bigl(q_j^{\theta_B}(Y_j^\dag),v_j^\dag\bigr)
    =
    \frac{1}{J}
    \sum_{j=1}^J
    \sS(Z_j).
\end{equation}
Assumption~\ref{da:ergodicity_on_state_filter}(ii) guarantees that $\{Z_j\}_{j\in\bbZ_+}$ is an ergodic Markov chain with invariant measure $\mu_\infty$. As established in Step 1, the score function $\sS$ is integrable with respect to $\mu_\infty$. By the ergodic theorem for Markov chains, the time average converges almost surely to the spatial expectation:
\begin{equation}
\label{eq:proof_single_trajectory_limit}
    \lim_{J\to\infty}
    \mathcal{L}_{J,1}(\theta_B)
    =
    \lim_{J\to\infty}
    \frac{1}{J}
    \sum_{j=1}^J
    \sS(Z_j)
    =
    \int_{\Omega}
    \sS(q,v)
    \mu_\infty(\mathrm{d}q,\mathrm{d}v)
    =
    \mathcal{L}_\star
    \quad
    \text{a.s.}
\end{equation}
Since $M$ is finite, the intersection of these $M$ probability-one events still has probability one. Hence,
\begin{equation}
\lim_{J\to\infty}\mathcal{L}_{J,M}(\theta_B) = \mathcal{L}_\star\quad\text{a.s.}
\end{equation}
\end{proof}

The finite-intersection argument in the proof also gives a weak finite-range uniform consequence. For any fixed $M_{\max}\in\bbN$, applying the theorem to the finitely many trajectory averages indexed by $m\in[M_{\max}]$ yields a probability-one event on which
\begin{equation}
\label{eq:finite_range_uniform_consequence}
    \lim_{J\to\infty}
    \sup_{1\leq M\leq M_{\max}}
    \left|
    \mathcal{L}_{J,M}(\theta_B)
    -
    \mathcal{L}_\star
    \right|
    =
    0.
\end{equation}
This is only a finite-range consequence; the theorem does not claim convergence uniformly over all $M\in\bbN$.

The theorem is stated at $\theta_B$ because $\theta_B$ is fixed across horizons and directly defines the exact augmented Markov process used in Assumption~\ref{da:ergodicity_on_state_filter}. Nevertheless, Proposition~\ref{prop:opt_score_to_filter} implies a finite-horizon equivalence for any full-expectation minimizer. Specifically, for each fixed $J$, if $\theta_J^\ast\in\argmin_{\theta\in\Theta}\mathcal{L}_{J,\infty}(\theta)$, then for every $j\in[J]$, $q_j^{\theta_J^\ast}(Y_j^\dag)=q_j^{\theta_B}(Y_j^\dag)=\pi_j$, $Y_j^\dag$ almost surely. Consequently, for $M\in\bbN\cup\{\infty\}$,
\begin{equation}
\label{eq:thetaJ_thetaB_M_loss_equiv_after_thm}
    \mathcal{L}_{J,M}(\theta_J^\ast)
    =
    \mathcal{L}_{J,M}(\theta_B)
\end{equation}
almost surely with respect to the $M$ sampled training trajectories, when $M\in\bbN$. This is an equivalence of induced filtering distributions and loss values over the finite horizon $[J]$, not an identification of the parameters themselves. In particular, $\theta_J^\ast$ need not be unique, need not equal $\theta_B$, may depend on $J$, and may behave differently from $\theta_B$ outside the histories and time indices covered by the horizon $[J]$.

\subsection{Neural Architecture and Practical Training}\label{subsec:architecture_training}

We now describe the finite-sample finite-ensemble implementation used for the practical objective $\mathcal{L}_{J,M,N}$ with $M,N\in\bbN$. In practical implementations, the ensemble is represented as a finite sequence of vectors rather than a probability measure.
\begin{equation}
\label{eq:seq_space}
    \mathcal{U}_F(\bbR^d) = \bigcup_{N=1}^\infty\mathcal{U}([N];\bbR^{d}),
\end{equation}
where $\mathcal{U}([N];\bbR^{d})$ denotes the space of sequences mapping $[N]=\{1,\dots,N\}$ to $\bbR^d$, isomorphic to $\bbR^{N\times d}$. Accordingly, within our framework, we parameterize the analysis map $\fB_\theta$ as a mapping dependent on the finite ensemble:
\begin{equation}
    \fB_\theta: \bbR^{d_v} \times \bbR^{d_y} \times \bbR^{d_y} \times \mathcal{U}_F(\bbR^{d_v+d_y}) \to \bbR^{d_v}.
\end{equation}

While Remark~\ref{rmk:permutation_inv} establishes the desired properties abstractly for probability measures, any neural implementation acting on the sequence space must formally satisfy two structural constraints to serve as a valid particle-level analysis operator:
\begin{enumerate}
    \item \textbf{Permutation invariance:} Since the sequence computationally represents an unordered set of samples, the particle-wise analysis update, with $(\hat v,\hat y)$ and $y^\dagger$ held fixed, must be invariant to the ordering of the ensemble members $\{(\hvn,\hyn)\}_{n=1}^N \in \mathcal{U}_F$. For any permutation $\sigma: [N]\to [N]$, the map must satisfy:
    \begin{equation}
        \fB_\theta\left(\hat{v}, \hat{y}~; y^\dagger, \{(\hvn,\hyn)\}_{n=1}^N\right) = \fB_\theta\left(\hat{v}, \hat{y}~; y^\dagger, \bigl\{\bigl(\hv^{(\sigma(n))},\hy^{(\sigma(n))}\bigr)\bigr\}_{n=1}^N\right).
\end{equation}
    \item \textbf{Varying ensemble size:} The architecture must handle input sequences of any finite length $N$ using a fixed set of parameters $\theta$, mapping them to a fixed output dimension in $\bbR^{d_v}$.
\end{enumerate}
To bridge the finite sequence representation with these structural requirements, we lift the input sequence $\{(\hvn,\hyn)\}_{n=1}^N$ to its empirical measure $\rho^N = \frac{1}{N}\sum_{n=1}^N \delta_{(\hvn,\hyn)}$ and employ a neural network architecture designed for measures.

\paragraph{Ensemble encoder}
We use the transformer measure neural mapping (TMNM) introduced in Appendix~\ref{app:attn_measures} to construct a trainable feature extractor, based on the set transformer architecture $\sF^\mathrm{ST}: \mathcal{U}_F(\bbR^{d_v+d_y}) \to \bbR^{\dst}$ \citep{lee2019set}. This component maps the finite sequence to a fixed-dimensional global feature vector $f_v \in \bbR^{\dst}$:
\begin{equation}
    f_v = \sF^\mathrm{ST}\left(\{(\hvn,\hyn)\}_{n=1}^N\right).
\end{equation}
The vector $f_v$ serves as a feature vector, encapsulating the statistical properties of the forecast ensemble necessary for the analysis step. Since $\sF^\mathrm{ST}$ acts on the empirical measure, it enforces permutation invariance and handles varying ensemble sizes $N$ by design. 

\paragraph{Residual end-to-end (EtE) analysis map}
The particle-wise analysis map $\fB_\theta$ is parameterized as a residual correction to the forecast state. We employ a multi-layer perceptron (MLP), $\fF^\mathrm{MLP}: \bbR^{d_v + 2d_y + \dst} \to \bbR^{d_v}$, to compute this correction by combining the local particle information with the global ensemble context $f_v$. Given a forecast particle $(\hat{v}, \hat{y})$, the true observation $y^\dagger$, and the full ensemble sequence, the updated state is defined by:
\begin{equation}
\label{eq:residual_arch}
    \fB_\theta\left(\hat{v}, \hat{y}~; y^\dagger, \{(\hvn,\hyn)\}_{n=1}^N\right) = \hat{v} + \fF^\mathrm{MLP}\Bigl(\hat{v}, \hat{y}~; y^\dagger, f_v\Bigr), 
\end{equation}
where $\theta$ represents the complete set of trainable parameters in both $\sF^\mathrm{ST}$ and $\fF^\mathrm{MLP}$. This residual architecture generalizes the EnKF update rule $v = \hat{v} + K(y^\dagger - \hat{y})$. Instead of deriving a linear correction from the empirical covariance matrix, $\sF^\mathrm{ST}$ nonlinearly extracts statistical features into $f_v$, and $\fF^\mathrm{MLP}$ applies a particle-specific nonlinear state correction. Applying the same map to every particle yields a permutation-equivariant ensemble update: reordering the input ensemble reorders the output ensemble in the same way.

We call this architecture \emph{end-to-end} (EtE) to distinguish it from other approaches, such as that taken in \cite{bach_learning_2026}, which instead learn a correction to an existing filter, thus starting with a strong inductive bias.

\paragraph{Empirical training objective}
We train the finite-particle realization of the filter directly using the empirical energy score  with $\beta=1$. For each training trajectory $m\in[M]$ and time $j\in[J]$, let $\{v_{j,m}^{(n)}(\theta)\}_{n=1}^N\subset\bbR^{d_v}$ denote the output ensemble. The empirical energy score objective is obtained by evaluating the energy score in Definition~\ref{def:energy_score}, equivalently the formula in \eqref{eq:ES}, on the empirical filtering distribution $q_{j,m}^{\theta,N}(Y_{j,m}^\dag)$:
\begin{equation}
\label{eq:empirical_energy_score}
    \sS^{\mathrm{ES}}\bigl(q_{j,m}^{\theta,N}(Y_{j,m}^\dag),v_{j,m}^\dag\bigr)
    =
    \frac{1}{N}
    \sum_{n=1}^N
    \left\|v_{j,m}^{(n)}(\theta)-v_{j,m}^\dag\right\|_2
    -
    \frac{1}{2N^2}
    \sum_{n=1}^N
    \sum_{k=1}^N
    \left\|v_{j,m}^{(n)}(\theta)-v_{j,m}^{(k)}(\theta)\right\|_2.
\end{equation}
In the finite-ensemble implementation, normalization constants other than $1/(2N^2)$ for the pairwise term are possible. We interpret the finite-$N$ objective as an approximation to the mean-field objective targeting the true filtering distribution; the empirical measures used in training will not generally consist of i.i.d.\ posterior samples. Equation~\eqref{eq:empirical_energy_score} is the energy score of the empirical measure, but is biased at finite $N$ as an estimator of the population score under i.i.d.\ sampling. Replacing $1/(2N^2)$ with $1/(2N(N-1))$ gives the off-diagonal form, which is unbiased when the ensemble members are i.i.d. \citep{fricker_three_2013}. Such unbiased scoring-rule estimators have also been used to train amortized posterior generators \citep{pacchiardi_likelihood_free_2022}. For a fixed posterior, minimizing the expected direct energy score over equally weighted $N$-point measures is equivalent to minimizing their energy distance to the posterior, the support-point criterion of \citet{mak2018support}. This supports using $1/(2N^2)$ to train an ensemble as a finite-point approximation of the posterior. In our experiments, the two normalizations generally produce similar test performance at the tested ensemble sizes, whereas the unbiased form can lead to degeneracy at smaller $N$ \citep{lang2026aifs}; Appendix~\ref{app_subsec:es_normalization_comparison} provides the supporting comparison. We therefore use the empirical-measure normalization in \eqref{eq:empirical_energy_score} throughout. The overall empirical training objective is
\begin{equation}
\label{eq:empirical_ES_loss}
    \mathcal{L}_{\mathrm{ES}}(\theta)
    :=
    \mathcal{L}_{J,M,N}(\theta)
    =
    \frac{1}{JM}
    \sum_{m=1}^M
    \sum_{j=1}^J
    \sS^{\mathrm{ES}}\bigl(q_{j,m}^{\theta,N}(Y_{j,m}^\dag),v_{j,m}^\dag\bigr).
\end{equation}
\paragraph{Differentiability requirements}
Since $\mathcal{L}_{J,M,N}(\theta)$ is differentiable with respect to the output ensembles, the optimization is carried out in finite-dimensional Euclidean space using automatic differentiation through the sequential forecast--analysis updates. Note that this requires access to the gradients of the forecast model, as well as the ensemble filtering algorithm; these can also be obtained through automatic differentiation provided they are implemented in a differentiable programming framework.

\begin{remark}[Truncated gradient propagation and filter stability]
\label{rmk:implementation_details}
For long training trajectories, accumulating gradients through the full recursive forecast--analysis loop can be memory-intensive and numerically unstable. In practice, we employ truncated backpropagation through time (TBPTT) by periodically detaching the computational graph while continuously preserving the forward propagation of the physical ensemble. 

This engineering heuristic is motivated by filter stability considerations. If the exact filter forgets its initial condition, then perturbations from distant past times are expected to have decreasing influence on the current filtering state. TBPTT exploits this intuition by discarding long-range gradient dependencies while preserving the forward propagation of the ensemble. We do not claim that the resulting truncated gradient is unbiased, nor that this argument proves convergence of the finite-ensemble neural optimizer for $\mathcal{L}_{J,M,N}$; we
simply point to the empirical success of the resulting trained models that we detail
later in the paper. See also Remark~\ref{rmk:finite_M_N_scope}.
\end{remark}

\paragraph{Transferability and fine-tuning}
Training the proposed model can be computationally expensive due to the $O(N'^2)$ complexity of the loss evaluation (although its gradients can be computed more efficiently; see \cite{hertrich_generative_2024}), and attention mechanism in particular, for (large) ensembles of size $N'.$ Because the set transformer processes the ensemble as a measure, a model trained with (small) ensemble size $N$ can be deployed directly on a different (large) ensemble size $N'$ without architectural modifications. This is a potentially
attractive feature of our methodology which, however, may need 
some modifications to be made practical. We now describe these modifications. 

The appropriate analysis correction often depends on the sampling noise inherent to the chosen ensemble size \citep{vishny_high-dimensional_2024};  for example, covariance localization and inflation are deployed in standard ensemble Kalman methods for this reason \citep{asch2016data}. 
To address this issue within the confines of our
proposed architecture, we adopt a two-stage strategy. Let $\theta = \{\theta_1,\theta_2\}$, where $\theta_1$ denotes the parameters of the set transformer $\sF^\mathrm{ST}$ and $\theta_2$ denotes the parameters of the residual MLP $\fF^\mathrm{MLP}$. First, we pre-train the full model on a computationally tractable ensemble size $N$:
\begin{equation}
\label{eq:pretraining_objective}
    \theta^{\mathrm{pre}}\in\argmin_{\theta}\mathcal{L}_{J,M,N}(\theta).
\end{equation}
Write $\theta^{\mathrm{pre}} = \{\theta_1^{\mathrm{pre}},\theta_2^{\mathrm{pre}}\}.$
Next, when deploying on a target ensemble size $N'$, we freeze the set transformer parameters $\theta_1$ and fine-tune only the lightweight residual block $\theta_2$:
\begin{equation}
\label{eq:fine_tuning_objective}
    \theta_2^\ast
    \in
    \argmin_{\theta_2}
    \mathcal{L}_{J,M,N'}(\theta_1^{\mathrm{pre}},\theta_2).
\end{equation}
This fine-tuning stage uses the same $M$ synthetic trajectories of length $J$ generated for pre-training, changing only the ensemble size used in the forward assimilation passes and the empirical loss evaluation. No new data is required.

As we will see later, we find that empirically fine-tuning has little effect on the performance of the end-to-end architecture, but can be significant for the architecture that uses the EnKF structure as an inductive bias.

%% file: sections/experiments.tex
\section{Numerical Experiments}
\label{sec:numerical_exp}
In the previous subsection we have defined an empirical loss function,
an architecture to approximate the analysis map, and a training strategy.
These together define the proposed proper scoring ensemble filter (PSEF).
This section evaluates the PSEF across standard diagnostic settings and nonlinear data assimilation benchmarks. Subsection~\ref{ssec:methods} introduces the established DA filters used as non-ML benchmarks and the learning-based comparison methods, together with the ES, L2, and NL2 training objectives used to separate the effects of architecture and loss function. Subsection~\ref{ssec:evaluation_metrics} summarizes the evaluation metrics used to assess distributional accuracy and ensemble calibration. The numerical studies then examine a linear--Gaussian diagnostic problem in Subsection~\ref{ssec:linear_gaussian_model}, the nonlinear and multimodal doubling-angle model in Subsection~\ref{ssec:doublin_angle_model}, the Lorenz~'63 system under several nonlinear observation maps in Subsection~\ref{ssec:L63}, and the high-dimensional Lorenz~'96 system in Subsection~\ref{ssec:L96}. Appendix~\ref{ssec:training_seed_sensitivity} examines sensitivity to the training seed. Finally, Subsection~\ref{ssec:runtime} compares the average computational cost per analysis update of different methods. Additional implementation details, including neural network architectures, training hyperparameters, hyperparameter-selection protocols, and signal-to-noise-ratio calibration, are provided in Appendix~\ref{app:implementation_details}, while supplementary visualizations and additional experimental results are reported in Appendix~\ref{app:additional_experimental_results}. We use PyTorch for our ML framework. The code, experiment scripts, and reproducibility materials are available in the GitHub repository\footnote{\url{https://github.com/wispcarey/Proper-Scoring-Ensemble-Filter}}.

All experiments use a default random seed of $42$, except for the SMF candidate-validation seeds in Appendix~\ref{app_subsec:grid_search} and the training seeds varied in Appendix~\ref{ssec:training_seed_sensitivity}. Across four training seeds, the SED standard deviation for PSEF remains below $8\%$ of the mean in the tested doubling-angle and Lorenz~'63 settings.

\subsection{Methods}\label{ssec:methods}
We compare the proposed machine learning--based PSEF with five established DA filters that require no offline ML training, hereafter called benchmark filters:
\begin{enumerate}
    \item \textbf{EnKF with perturbed observations (EnKF)} \citep{burgers1998analysis, anderson2001ensemble}: This is the standard stochastic implementation of the ensemble Kalman filter. Each ensemble member assimilates measurements corrupted by additive noise sampled from the observation error distribution. 

    \item \textbf{Ensemble square root filter (ESRF)} \citep{tippett_ensemble_2003, sakov2008deterministic}: This is a deterministic alternative of the EnKF. This filter circumvents sampling errors associated with perturbed observations by applying a matrix square root transformation to the ensemble anomalies, reproducing the requisite analysis error covariance.
    
    \item \textbf{Local ensemble transform Kalman filter (LETKF)} \citep{hunt2007efficient}: This approach incorporates spatial localization into the deterministic ensemble transform Kalman filter (ETKF) \citep{bishop2001adaptive}. While the ETKF and ESRF are theoretically isomorphic in the mean-field, linear--Gaussian limit, their numerical architectures are slightly different. 

    \item \textbf{Iterative ensemble Kalman filter (IEnKF)} \citep{sakov2012iterative}: To address strong system nonlinearities, this method embeds a variational minimization step within the stochastic EnKF framework. The optimization bridging adjacent observation times is executed with a maximum of $10$ iterations and a convergence tolerance of $10^{-5}$.  This approach is computationally costly for high-dimensional systems, because it requires repeatedly integrating the forward dynamic model to iteratively solve the optimization problem within a single assimilation cycle.

    \item \textbf{Stochastic map filter (SMF)} \citep{spantini2022coupling}: This nonlinear generalization of the stochastic EnKF fits a monotone triangular Knothe--Rosenblatt rearrangement to simulated joint forecast--observation samples. The induced stochastic coupling conditions the forecast ensemble on the realized observation without importance weights. We use a separable Gaussian radial basis function (RBF) parameterization with map order $2$, comprising a global linear term plus one Gaussian RBF in each active nonmonotone univariate component, which allows a non-Gaussian analysis update while retaining convex map estimation.
\end{enumerate}

When inflation is enabled, the Kalman-type filters apply multiplicative inflation to the analysis ensemble, directly increasing its posterior spread. SMF instead inflates only the forecast samples used to fit the transport map and then transports the original, uninflated ensemble. For the low-dimensional nonlinear experiments, no localization is used. For the higher-dimensional Lorenz~'96 experiment in Subsection~\ref{ssec:L96}, we compare EnKF, LETKF, and IEnKF with spatial localization based on the Gaspari--Cohn distance-to-weight function \citep{gaspari1999construction}. SMF instead uses hard localization of its nonlinear map: dependencies beyond the chosen radius are removed and sufficiently distant state components are left unchanged, rather than being smoothly downweighted by a Gaspari--Cohn covariance taper. The ESRF is not reported separately for Lorenz~'96, because applying localization to the deterministic square-root update yields a benchmark that is functionally represented by LETKF in our comparison. Further implementation and hyperparameter-selection details for the benchmark filters are given in Appendix~\ref{app_subsec:grid_search}.

We next introduce the notation and training objectives used for the learning-based filters. For a learning-based filter with trainable parameters $\theta$, let
$\{v_j^{(n)}(\theta)\}_{n=1}^{N}$ denote the output ensemble generated by the learned filter, and $v_j^\dagger$ the corresponding true state. Then define
\begin{equation}
\label{eq:learned_filter_empirical_distribution}
    \pi_{j}^{\theta,N}
    =
    \frac{1}{N}\sum_{n=1}^{N}\delta_{v_j^{(n)}(\theta)},
    \qquad
    \bar v_j(\theta)
    =
    \frac{1}{N}\sum_{n=1}^{N}v_j^{(n)}(\theta),
\end{equation}
the empirical analysis distribution and the ensemble mean at assimilation cycle $j$, respectively; note that $\bar v_j(\theta)=\bbE^{v \sim \pi_{j}^{\theta,N}}[v].$ The losses below are written for a single trajectory of length $J$ and are averaged over all training trajectories in implementation.

The ES objective is $\mathcal{L}_\text{ES}(\theta)$ as defined in \eqref{eq:empirical_energy_score}.
Because the energy score is a strictly proper scoring rule, minimizing \eqref{eq:empirical_ES_loss} trains the ensemble as a probabilistic approximation of the filtering distribution rather than only as a point estimator.
For comparison with deterministic state-estimation objectives, we also define the unnormalized and normalized squared-error losses on the ensemble mean:
\begin{align}
\label{eq:methods_L2_NL2_loss}
    \mathcal{L}_{\mathrm{L2}}(\theta)
    =
    \frac{1}{J}\sum_{j=1}^{J}
    \left\|\bar v_j(\theta)-v_j^\dagger\right\|_2^2,\quad
    \mathcal{L}_{\mathrm{NL2}}(\theta)
    =
    \frac{1}{J}\sum_{j=1}^{J}
    \frac{
    \left\|\bar v_j(\theta)-v_j^\dagger\right\|_2^2
    }{
    \left\|v_j^\dagger\right\|_2^2+10^{-8}
    }.
\end{align}
The additive $10^{-8}$ in the NL2 denominator prevents division by zero and bounds the normalization factor when the norm of the truth state is small.
Unlike ES, both L2 and NL2 depend only on the ensemble mean and therefore do not directly reward the learned filter for representing posterior spread or higher-order distributional structure. We include L2 in the linear--Gaussian diagnostic experiment in Subsection~\ref{ssec:linear_gaussian_model}, where L2 and NL2 exhibit broadly similar behavior, and both are less robust than ES across learning-rate choices. However, it is established in \cite{bach_learning_2026} that L2 is less effective and less stable than NL2. Therefore, in the nonlinear experiments, we use NL2 as the point-estimation baseline and do not include an additional L2 comparison. In this section, unless otherwise specified, we use $M=8192$ training trajectories with a length $J=60$.

Alongside these benchmark filters, we benchmark against other machine learning--based DA approaches, focusing on MNMEF \citep{bach_learning_2026}. This scheme learns parameterized correction terms applied to the EnKF for nonlinear and non-Gaussian problems. MNMEF is originally designed for state estimation and optimizes the normalized $L^2$ loss in \eqref{eq:methods_L2_NL2_loss}. In contrast, our proposed PSEF is an end-to-end (EtE) filtering framework, detailed in Section~\ref{sec:learning_true_filter}, which optimizes the updated ensemble using the proper scoring objective in \eqref{eq:empirical_ES_loss}. 
The MNMEF, since it is based on correcting the EnKF and retains the latter's gain structure, imposes an inductive bias on the learned filter. This inductive bias may be beneficial in reducing sample complexity in settings where the EnKF performs relatively well, namely those that are close to Gaussian. However, it may be a limitation for settings where the EnKF performs poorly, motivating the EtE approach.

To systematically isolate the contributions of the filtering architectures and the objective functions, we construct a comparison pairing the two architectural schemes, the EnKF correction-terms approach (CorrTerms) \citep{bach_learning_2026} and the end-to-end framework (EtE), with the two loss functions retained in the main experiments, NL2 and ES. This yields four distinct machine learning configurations:
\begin{enumerate}
    \item \textbf{CorrTerms + NL2 (MNMEF)}: This is the original MNMEF configuration proposed in \cite{bach_learning_2026}. It retains the EnKF correction-terms architecture and trains the correction model using the normalized $L^2$ loss in Equation~\eqref{eq:methods_L2_NL2_loss}. This configuration serves as the main machine learning baseline and represents a state estimation--oriented approach built upon an existing EnKF update.

    \item \textbf{CorrTerms + ES}: This configuration keeps the same EnKF correction-terms architecture as MNMEF but replaces the normalized $L^2$ objective with the energy score loss in Equation~\eqref{eq:empirical_ES_loss}. Comparing this model with CorrTerms + NL2 isolates the effect of changing the training objective from a mean-based state estimation loss to a strictly proper probabilistic scoring rule.

    \item \textbf{EtE + NL2}: This configuration uses our proposed end-to-end filtering architecture but trains it with the normalized $L^2$ loss in Equation~\eqref{eq:methods_L2_NL2_loss}. It therefore tests whether the architectural change alone improves filtering performance when the learning objective still focuses only on the ensemble mean.

    \item \textbf{EtE + ES (PSEF)}: This is the proposed PSEF method in this paper. It combines the end-to-end filtering architecture with the energy score loss in Equation~\eqref{eq:empirical_ES_loss}, thereby training the updated ensemble directly as a probabilistic approximation of the filtering distribution rather than merely optimizing its mean estimate.
\end{enumerate}
The machine learning models are initially trained using a fixed ensemble size. To evaluate inference performance across varying ensemble sizes, we apply lightweight fine-tuning. For the CorrTerms scheme, we adopt the fine-tuning procedure introduced in \cite{bach_learning_2026}. For the EtE configurations, we utilize the fine-tuning strategy detailed in Subsection~\ref{subsec:architecture_training}.

\subsection{Evaluation Metrics}\label{ssec:evaluation_metrics}

We use different evaluation metrics depending on whether an accurate numerical reference filtering distribution is available. For the linear--Gaussian diagnostic experiment in Subsection~\ref{ssec:linear_gaussian_model}, the exact analytical filtering distribution is available from the Kalman filter. We therefore evaluate an approximate ensemble filtering distribution by fitting a Gaussian distribution to its sample mean and sample covariance and computing its Wasserstein-$2$ distance to the exact Kalman filtering distribution, as defined in \eqref{eq:W_2_between_Gaussian}. For the remaining nonlinear experiments, Table~\ref{tab:evaluation_metric_summary} summarizes the abbreviations, formula references, and roles of the evaluation metrics.

\begin{table}[htbp]
\centering
\caption{Summary of evaluation metrics used in the nonlinear experiments.}
\label{tab:evaluation_metric_summary}
\renewcommand{\arraystretch}{1.25}
\begin{tabular}{llll}
    \toprule
    \textbf{Metric}
    & \textbf{Interpretation}
    & \textbf{Reference}
    & \textbf{Dataset} \\
    \midrule
    \multirow{2}{*}{SED}
    & Distance to the BPF reference
    & Appendix~\ref{app_subsec:sed_metric}
    & Doubling-angle \\
    & filtering distribution
    & \eqref{eq:projected_energy_distance}, \eqref{eq:sed_score}
    & Lorenz~'63 \\
    \midrule
    \multirow{2}{*}{FED}
    & Full-state energy distance
    & Appendix~\ref{app_subsec:fed_metric}
    & \multirow{2}{*}{Lorenz~'63} \\
    & to BPF reference subsamples
    & \eqref{eq:fed_subsample_energy_distance}, \eqref{eq:fed_score}
    & \\
    \midrule
    \multirow{2}{*}{ES}
    & Energy score against
    & Definition~\ref{def:energy_score}, \eqref{eq:ES}
    & \multirow{2}{*}{Lorenz~'96} \\
    & the verifying truth state
    & Subsection~\ref{ssec:methods}, \eqref{eq:empirical_ES_loss}
    & \\
    \midrule
    \multirow{2}{*}{\shortstack[l]{Variogram\\score (VS)}}
    & State-space pairwise
    & Appendix~\ref{app_subsec:variogram_score}
    & \multirow{2}{*}{Lorenz~'96} \\
    & differences against truth
    & \eqref{eq:variogram_score}, \eqref{eq:variogram_score_time_average}
    & \\
    \midrule
    \multirow{2}{*}{RHVar}
    & Scaled rank histogram
    & Appendix~\ref{app_subsec:rhvar_metric}
    & Lorenz~'63 \\
    & nonuniformity
    & \eqref{eq:project_directions}, \eqref{eq:rank_hist_variance}
    & Lorenz~'96\\
    \bottomrule
\end{tabular}
\end{table}

For the low-dimensional doubling-angle model in Subsection~\ref{ssec:doublin_angle_model} and Lorenz~'63 model in Subsection~\ref{ssec:L63}, we use the bootstrap particle filter (BPF) reference described in Appendix~\ref{app:bpf_reference} as a numerical proxy for the true filtering distribution. The BPF is run with $10^6$ particles, and its validity diagnostics are reported in Appendix~\ref{app_subsec:bpf_validity}. Since storing and comparing to all the particles is computationally costly, we use SED for both low-dimensional experiments and additionally report the full-state energy distance (FED) as a complementary joint-distribution check for Lorenz~'63. The projected energy distance is defined in \eqref{eq:projected_energy_distance}, and the reported time- and projection-averaged SED is defined in \eqref{eq:sed_score}. This metric compares the tested ensemble with the BPF reference after projection onto coordinate directions and BPF principal component directions. For doubling-angle, these directions are equivalent up to sign, and the comparison uses the chordal distance in \eqref{eq:torus-chordal-distance} to account for the periodic state space. A smaller SED indicates closer agreement with the BPF reference filtering distribution. Because this finite set of prescribed projections does not by itself characterize a three-dimensional joint law, FED, defined in \eqref{eq:fed_subsample_energy_distance}--\eqref{eq:fed_score}, provides a complementary Lorenz~'63 comparison that does not use projections. At each assimilation cycle of every test trajectory, FED computes the ordinary multivariate energy distance to each of four stored BPF reference subsamples containing $1000$ particles and averages the resulting four distances.

For the higher-dimensional Lorenz~'96 experiment in Subsection~\ref{ssec:L96}, constructing an accurate BPF reference is not computationally feasible because of the state dimension. We therefore evaluate the analysis ensembles directly against the truth state using the energy score (ES) and the variogram score (VS). Here, ES is the finite-ensemble score in \eqref{eq:empirical_ES_loss}, computed separately for each test trajectory. This is the same finite-ensemble energy score used as the ES training objective, but here it is evaluated on the test set as an evaluation metric. Lower values indicate better probabilistic agreement between the ensemble distribution and the verifying truth state. VS complements ES by measuring whether the analysis ensemble reproduces the truth's cross-component difference structure in the full state space; its definition and aggregation are given in Appendix~\ref{app_subsec:variogram_score}.

We also assess stability across $64$ independent test trajectories through the standard deviation of the scores. PSEF's standard deviations are of the same order as those of the benchmark filters and are generally smaller in Lorenz~'96 (Appendix~\ref{app:trajectory_variability}). Score means omit divergent trajectories, while standard deviations are reported only when all $64$ trajectory outputs are available.

We also consider a relative improvement (RI) panel. For an error metric $\mathcal{E}$ for which smaller values are better, the relative improvement of a reference method with error $\mathcal{E}_\mathrm{ref}$ over a method $e$ is defined as
\begin{equation}
\label{eq:relative_improvement}
    \operatorname{RI}_{\mathcal{E}}(e)
    =
    \frac{\mathcal{E}(e)-\mathcal{E}_{\mathrm{ref}}}{\mathcal{E}(e)}.
\end{equation}
Thus, a positive RI means that the designated reference method has lower error than method $e$, zero means equal errors, and a negative RI means that method $e$ performs better.
For the doubling-angle model, $\mathcal{E}$ is SED and the reference is EtE+ES at each plotted ensemble size. Learning-based comparisons use the same ensemble size, whereas each benchmark filter uses the fixed ensemble size stated in the Figure~\ref{fig:doubling_angle_metrics} legend. For the Lorenz~'63 SED comparisons, $\mathcal{E}$ is SED and the reference is EtE+ES under the same observation map and ensemble size. For the supplementary Lorenz~'63 FED comparison in Appendix~\ref{app:lorenz63_additional_results}, $\mathcal{E}$ is FED and the reference is EtE+ES under the same observation map and ensemble size. For Lorenz~'96, $\mathcal{E}$ is either ES or VS, and for each score the reference value is the smaller corresponding value between EtE+ES and CorrTerms+ES for the same observation map and ensemble size.

For Lorenz~'63 and Lorenz~'96, we additionally report RHVar as an exploratory diagnostic of marginal ensemble calibration. The primary metrics are SED for Lorenz~'63 and ES and VS for Lorenz~'96, while FED checks the Lorenz~'63 SED conclusions in the joint state space. The RHVar definition, normalization, and interpretation are given in Appendix~\ref{app_subsec:rhvar_metric}.

\subsection{Linear--Gaussian Model}\label{ssec:linear_gaussian_model}

In this subsection, we consider a linear time-invariant system with state dimension $d_v=20$ and observation dimension $d_y=10$. The primary objectives of this experiment are twofold: first, to demonstrate that our proposed training methodology can achieve performance approaching that of the exact true filter, namely the Kalman filter, in a linear--Gaussian setting; and second, to highlight the advantages of the energy score (ES) loss in training stability and convergence compared to the L2 and NL2 losses.

The dynamics and observation models are formulated as
\begin{subequations}
\begin{align}
    v_{j+1} &= A v_j+\xi_j,
    &
    \xi_j &\sim \normal(0,\sigma_v^2 I_{d_v}),
    \\
    y_{j+1} &= H v_{j+1}+\eta_{j+1},
    &
    \eta_{j+1} &\sim \normal(0,\sigma_y^2 I_{d_y}).
\end{align}
\end{subequations}
The detailed parameter settings for the linear--Gaussian system are shown in Table~\ref{tab:linear_gaussian_settings}.

\begin{table}[htbp]
\centering
\caption{Linear--Gaussian settings.}
\label{tab:linear_gaussian_settings}
\begin{tabular}{ll}
    \toprule
    \textbf{Category} & \textbf{Values} \\
    \midrule
    Parameters & $A\in\bbR^{20\times 20}$ is a fixed random orthogonal matrix \\
    \midrule
    States & $v=(x_1,x_2,\ldots,x_{20})\in\mathbb{R}^{20}$, $\sigma_v=0.01$ \\
    \midrule
    Observations & $Hv=(x_1,x_3,\ldots,x_{19})\in\bbR^{10}$, $\sigma_y=1.0$ \\
    \midrule
    Initialization & $v_0\sim\normal(m_0,C_0)$, where $m_0\sim\normal(0,I_{20})$, $C_0=MM^T+10^{-6}I_{20}$ \\
    & $M\in\bbR^{20\times 20}$ with entries independently sampled from $\normal(0,1)$ \\
    \bottomrule
\end{tabular}
\end{table}

To isolate the impact of the training objective, the experimental configuration here diverges from the general setup outlined in Subsection~\ref{ssec:methods}. We focus exclusively on the end-to-end (EtE) filtering architecture. The ES-trained, L2-trained, and NL2-trained models share identical architectures, ensemble sizes, and training hyperparameters, differing only in the objective function. The three objectives are defined in Subsection~\ref{ssec:methods}: ES in \eqref{eq:empirical_ES_loss}, L2 and NL2 in \eqref{eq:methods_L2_NL2_loss}. This diagnostic comparison is used to assess optimization stability under controlled linear--Gaussian dynamics.

The evaluation is conducted on a separate test set of $64$ trajectories of length $100$. All learning-based methods operate with an ensemble size of $N=10$. We benchmark these configurations against the standard EnKF and LETKF with $N=10$ ensemble members, calibrating the inflation and localization parameters via grid search to minimize the Wasserstein-2 distance.

Because the system is linear--Gaussian, the exact filtering distribution $\pi_j=\normal(\mu_j^\dagger,\Sigma_j^\dagger)$ is computed via the Kalman filter. We assess the approximated filtering distribution $\pi_j^N$, characterized by the sample mean $\hat{\mu}_j$ and sample covariance $\hat{\Sigma}_j$, using the Wasserstein-2 distance between Gaussians:
\begin{equation}
\label{eq:W_2_between_Gaussian}
    W_2(\pi_j^N,\pi_j)^2
    =
    \|\hat{\mu}_j-\mu_j^\dagger\|^2
    +
    \trace\left(
    \hat{\Sigma}_j+\Sigma_j^\dagger
    -
    2\left(
    (\Sigma_j^\dagger)^{1/2}
    \hat{\Sigma}_j
    (\Sigma_j^\dagger)^{1/2}
    \right)^{1/2}
    \right).
\end{equation}
To quantify the optimal i.i.d.\ sampling error associated with the finite ensemble size, we draw $N=10$ independent samples directly from the true posterior $\pi_j$ at each assimilation cycle $j$ to form an empirical measure $\tilde{\pi}_j^N$. We then compute the $W_2$ distance between the Gaussian fitted to $\tilde{\pi}_j^N$ and the exact filtering distribution $\pi_j$. Averaging this procedure over $100$ independent trials yields the expected approximation error $\bbE[W_2(\tilde{\pi}_j^N,\pi_j)]$, which serves as the optimal i.i.d.\ sampling error.

\begin{figure}[htbp]
    \centering
    \includegraphics[width=\textwidth]{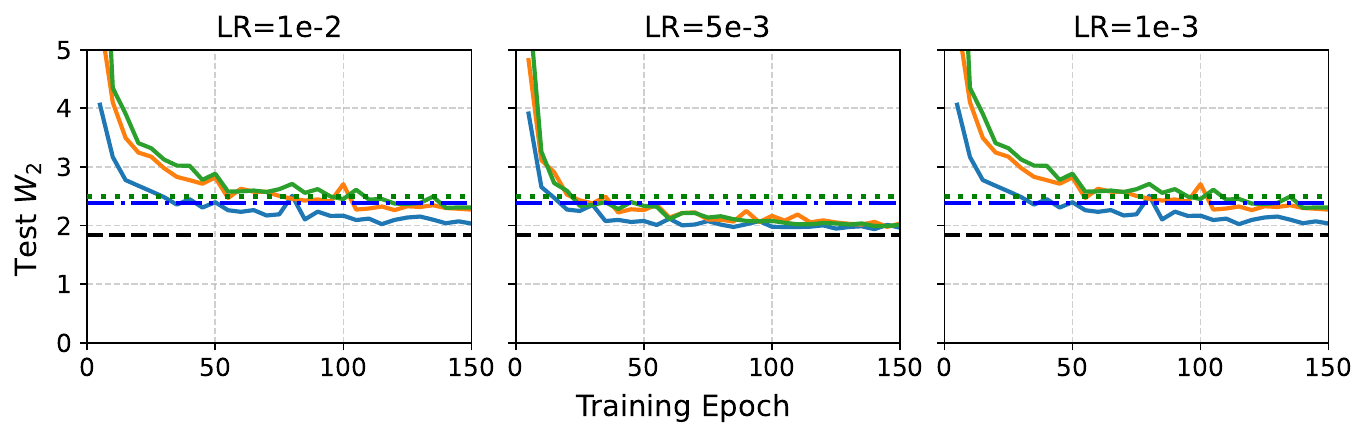}
    \includegraphics[width=\textwidth, trim={60 20 60 20}, clip]{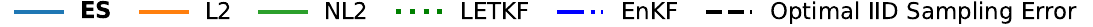}
    \caption{Convergence of the test $W_2$ distance over 150 training epochs for EtE models trained with ES, L2, and NL2 losses across three learning rates, $\mathrm{LR}=0.01$, $\mathrm{LR}=0.005$, and $\mathrm{LR}=0.001$. The benchmarks include LETKF and EnKF with $N=10$ ensemble members calibrated by grid search, together with the optimal i.i.d.\ sampling error.}
    \label{fig:linear_gaussian_convergence}
\end{figure}

Figure~\ref{fig:linear_gaussian_convergence} illustrates the convergence behavior of the evaluated models over 150 training epochs under the three learning rates $\mathrm{LR}=0.01$, $\mathrm{LR}=0.005$, and $\mathrm{LR}=0.001$. The ES-trained model exhibits consistently stable behavior across all three learning rates. Its test $W_2$ distance decreases rapidly and remains close to the optimal i.i.d.\ sampling error, showing stability that is better than the EnKF and LETKF baselines.

The deterministic mean-based losses are more sensitive to the learning rate than ES. Across the three learning rates, L2 and NL2 display broadly similar behavior in this linear--Gaussian experiment: both can achieve reasonable performance under suitable tuning, but both are less robust than ES when the learning rate is varied. In particular, ES remains stable across all tested learning rates and stays close to the optimal i.i.d.\ sampling error level, whereas the deterministic objectives show larger sensitivity to the optimization setup. 

\subsection{Doubling-Angle Model}\label{ssec:doublin_angle_model}

We consider the doubling-angle model, which expands the illustrative example introduced in Subsection~\ref{ssec:an_illustrative_example}. This experiment is designed as a controlled comparison of different filtering architectures, training objectives, and benchmark filters on a simple but strongly nonlinear and non-Gaussian problem. The hidden state is one-dimensional and is interpreted as an angle on the unit circle after multiplication by $2\pi$. The state evolves according to an expanding modulo-$1$ map, and the observation is the cosine of the resulting angle; see \eqref{eq:doubling}.
Here $v_j^\dagger\in[0,1)$ and $y_j\in\bbR$. Although the state dimension is only one, the model is challenging for filtering. The map $v\mapsto 2v \bmod 1$ amplifies small perturbations exponentially in time, while the nonlinear observation map $v\mapsto \cos(2\pi v)$ is not one-to-one. In particular, the states $v$ and $1-v$ generate the same noiseless observation, which naturally induces multimodal filtering distributions.

The detailed model and evaluation settings are summarized in Table~\ref{tab:doubling_angle_settings}. Since the state dimension is one, we can construct a highly accurate reference posterior using a bootstrap particle filter (BPF) with $10^6$ particles. This reference filter is used only for evaluation and visualization of the filtering distribution; the learning-based filters are trained using simulated truth trajectories rather than samples from the reference posterior.

Because the state is one-dimensional, the coordinate and BPF-PCA projections are equivalent, so SED reduces to an energy-distance comparison on the unit circle using the chordal distance in \eqref{eq:torus-chordal-distance}, with the $10^6$-particle BPF reference evaluated through its stored quantile representation.

\begin{table}[htbp]
\centering
\caption{Doubling-angle model settings.}
\label{tab:doubling_angle_settings}
\begin{tabular}{p{0.28\textwidth}p{0.64\textwidth}}
    \toprule
    \textbf{Category} & \textbf{Values} \\
    \midrule
    State
    & $v_j^\dagger\in[0,1)$ with modulo-$1$ geometry. $d_v=1$, $\sigma_v=0.01$\\
    \midrule
    Observation
    & $y_{j+1}=\cos\bigl(2\pi v_{j+1}^\dagger\bigr)+\eta_{j+1}$. $d_y=1$, $\sigma_y=0.2$\\
    \bottomrule
\end{tabular}
\end{table}

To evaluate the generalization capabilities of the learning-based methods, all machine learning (ML) configurations are trained using a fixed ensemble size of $N=30$. During testing, they are directly evaluated across varying ensemble sizes $N \in \{10, 30, 100, 300\}$ without any fine-tuning. In contrast, to provide the benchmark filters (EnKF, ESRF, IEnKF, and SMF) with a significant advantage in representing the filtering distribution, we evaluate them using substantially larger ensemble sizes of $N \in \{300, 1000, 3000\}$.

\begin{figure}[htbp]
    \centering
    \includegraphics[width=0.8\textwidth]{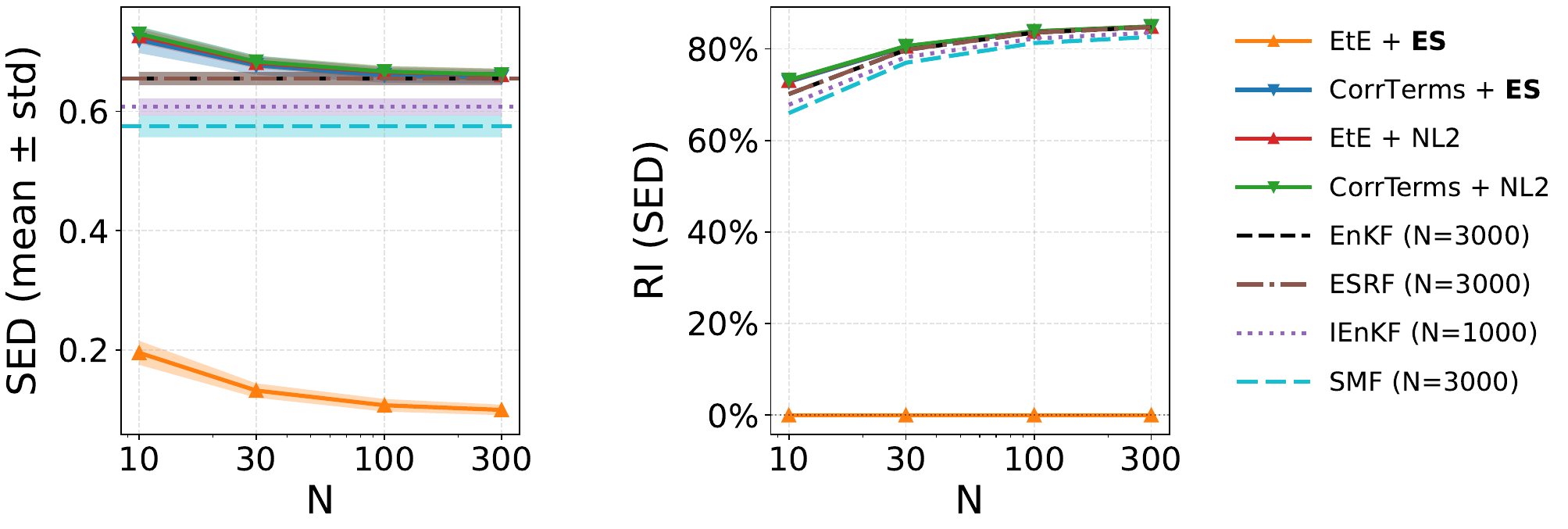}
    \caption{Performance on the doubling-angle model. The left panel reports the SED mean over available test trajectories, with $\pm$ standard deviation shown when all $64$ outputs are available; the right panel reports the RI in SED of EtE+ES over each comparison method. Benchmark ensemble sizes are shown in the legend.}
    \label{fig:doubling_angle_metrics}
\end{figure}

Figure~\ref{fig:doubling_angle_metrics} quantitatively compares the performance of the tested methods.

For SED evaluation, we embed the periodic state space $\mathbb{T}=\mathbb{R}/\mathbb{Z}$ into the unit circle via $\phi(x)=(\cos(2\pi x),\sin(2\pi x))$ and use the Euclidean chordal distance
\begin{equation}
d_{\mathrm{chord}}(x,y)=\|\phi(x)-\phi(y)\|_2=2|\sin(\pi(x-y))|.
\label{eq:torus-chordal-distance}
\end{equation}
This distance respects the periodic geometry and induces an energy distance that vanishes if and only if the two probability measures on $\mathbb{T}$ coincide, since it is the Euclidean energy distance between their images under the injective embedding $\phi$. The geodesic distance $d_{\mathrm{geo}}(x,y)=\min\{|x-y|,1-|x-y|\}$ for representatives $x,y\in[0,1)$ does not have this property: the distinct measures $\mu=\tfrac12(\delta_0+\delta_{1/2})$ and $\nu=\tfrac12(\delta_{1/4}+\delta_{3/4})$ have zero energy distance under $d_{\mathrm{geo}}$. Thus, although $d_{\mathrm{geo}}$ is a metric on the state space, its associated energy distance is not a metric on probability measures.

The chordal distance is used for SED evaluation against the BPF reference. During training, the ES objective is computed using the standard Euclidean distance between values represented in $[0,1)$.
As shown in the left panel, our proposed PSEF method (EtE + ES) achieves the lowest SED at every tested ensemble size and improves as $N$ increases. In the right panel, the large RI values show that EtE + ES substantially improves over the comparison methods in approximating the BPF reference distribution, with the improvement being especially pronounced relative to the NL2-trained models. These results indicate closer agreement with the BPF reference in the tested settings. In contrast, the NL2-trained models have much larger SED values, indicating that the tested mean-based training configurations recover this non-Gaussian filtering distribution less accurately. The benchmark filters and CorrTerms + ES also remain noticeably above EtE + ES in SED, despite the much larger ensemble sizes used by the benchmark filters. SMF has the lowest SED among the benchmark filters in this example, but remains noticeably above EtE+ES.

These quantitative findings are illustrated by the density distributions for the assimilation cycle $200$ visualized in Figure~\ref{fig:doubling1d_filtering} and Figure~\ref{fig:doubling1d_prior_filtering_step200}. Additional results for assimilation cycle $150$ in the same length-$200$ trajectory are provided in Appendix~\ref{app:doubling_angle_additional_results}. The true filtering density has a pronounced bimodal structure, as accurately captured by the BPF reference. The visualization demonstrates that only the proposed EtE + ES (PSEF) architecture successfully resolves and preserves this complex bimodal filtering distribution. The other approaches fail to reproduce this structure. The EnKF, ESRF, and IEnKF collapse to unimodal approximations, while SMF's fitted nonlinear map also misses the two modes. In these runs, the NL2-trained models and CorrTerms+ES recover the bimodal structure less accurately than EtE+ES. These results suggest that, in this experiment, both the expressive end-to-end architecture and the probabilistic ES objective are important for approximating the bimodal filtering distribution.

\subsection{Lorenz '63}\label{ssec:L63}

We consider the Lorenz '63 simplified model for atmospheric convection \citep{lorenz63}, denoting the state vector by $v = (u_1, u_2, u_3)^\top$. The system dynamics are described by the equations
\begin{subequations}\label{eq:L63_dynamic}
\begin{align}
\frac{du_1}{dt} &= \sigma (u_2 - u_1), \\ 
\frac{du_2}{dt} &= u_1 (\rho - u_3) - u_2, \\ 
\frac{du_3}{dt} &= u_1 u_2 - \beta u_3. 
\end{align} 
\end{subequations}

Map $\Psi$ is defined by integrating this equation over time $\Delta t.$ Here and in the Lorenz '96 experiments, the assimilation window refers to the time between observations, whereas the RK4 time step is the time step used to integrate the model. The experimental settings for the Lorenz '63 model are summarized in Table~\ref{tab:lorenz63_settings}. We consider three observation operators applied to the first state dimension: an identity map, a square map, and an arctan map. To ensure similar observation noise levels across these mappings, the observation noise standard deviation $\sigma_y$ is calibrated to match the empirical signal-to-noise ratio (SNR) of the default partial identity observation. The details of this SNR calibration procedure are provided in Appendix~\ref{app_subsec:snr_calibration}.

\begin{table}[htbp]
\centering
\caption{Lorenz '63 model settings.}
\label{tab:lorenz63_settings}
\renewcommand{\arraystretch}{1.2}
\begin{tabular}{ll}
    \toprule
    \textbf{Category} & \textbf{Values} \\
    \midrule
    Parameters & $\sigma = 10, \quad \rho = 28, \quad \beta = \frac{8}{3}$ \\
    \midrule
    States & $v = (u_1, u_2, u_3)^\top \in \mathbb{R}^d, \quad d_v = 3, \quad \sigma_v=0.01$ \\
    \midrule
    Observations & Partial + Identity: $h(v) = u_1, \quad d_y=1, \quad \sigma_y = 2.00$ \\
    & Partial + Square: $h(v) = u_1^2, \quad d_y=1, \quad \sigma_y = 16.98$ \\
    & Partial + Arctan: $h(v) = \arctan(u_1), \quad d_y=1, \quad \sigma_y = 0.33$ \\
    \midrule
    Time step & Assimilation window: $\Delta t=0.15$; $5$ RK4 steps within each window: $dt = 0.03$ \\
    \bottomrule
\end{tabular}
\end{table}

Figure~\ref{fig:L63_metric_curves} compares the performance of the tested methods across the three observation operators. We first examine the overall distributional accuracy using the SED to the BPF reference filtering distribution (first row) and the relative SED improvement (RI) of EtE+ES over each comparison method, as defined in \eqref{eq:relative_improvement} (second row). The two methods trained with the ES loss demonstrate the best performance among all methods. Specifically, our proposed end-to-end framework, EtE+ES, almost always achieves the lowest overall SED, followed by CorrTerms+ES as the second best.

\begin{figure}[t]
    \centering
    \includegraphics[width=\textwidth]{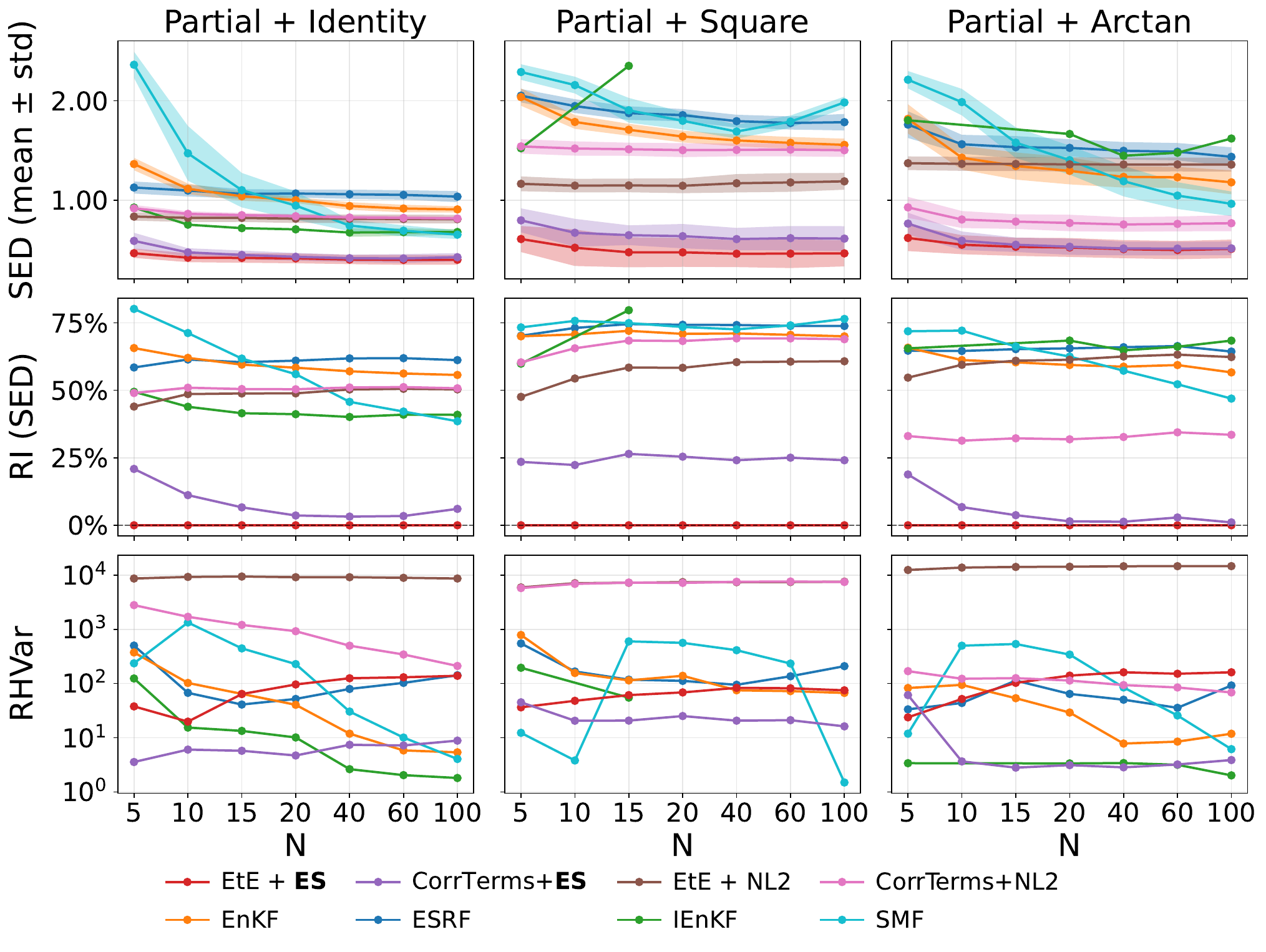}

    \vspace{-0.8em}
    \caption{Lorenz~'63 results under the partial identity, square, and arctan observations. Rows report SED, the RI in SED of EtE+ES over each comparison method, and RHVar. The SED mean omits divergent trajectories, while $\pm$ standard deviation is shown only when all $64$ outputs are available.}
    \label{fig:L63_metric_curves}
\end{figure}
Under the partial identity observation, the SED advantage of EtE+ES over CorrTerms+ES is clearest at small ensemble sizes and narrows as $N$ increases; under the partial square observation, EtE+ES retains a substantial gain throughout the tested range. Under the partial arctan observation, the two ES-based architectures become comparable as $N$ increases. Overall, both architectures trained with the ES loss exhibit a clear performance advantage over the two machine learning models trained with the NL2 objective and the four benchmark filters. SMF improves with ensemble size under identity and arctan observations and is the best-performing benchmark filter at $N=100$. Under the square observation, its SED remains above the ES-trained methods.

SED uses projected quantile summaries constructed from all $10^6$ BPF particles, but only along the prescribed coordinate and BPF-PCA directions. To check whether this finite projection set affects the Lorenz~'63 conclusions, Figure~\ref{fig:L63_fed_curves} in Appendix~\ref{app:lorenz63_additional_results} reports FED, an energy-distance comparison in the full three-dimensional state space using BPF reference subsamples. FED yields similar RI patterns and the same ordering of the leading methods as SED across the three observation maps: EtE+ES has the smallest FED among the tested methods, followed by CorrTerms+ES, with the two becoming close under the partial arctan observation. It also reproduces SMF's marked improvement with $N$ under identity and arctan observations and its weaker performance under the square observation.

The third row reports RHVar as a complementary diagnostic of marginal calibration; under the uniform-rank null hypothesis its calibrated baseline expectation is $1$, and values far above $1$ indicate rank nonuniformity. Under the partial identity and partial arctan observations, CorrTerms+ES and IEnKF generally have small RHVar values. Under the partial square observation, IEnKF diverges, while CorrTerms+ES generally has smaller RHVar than the other learned methods, although its values remain above $1$. Thus, EtE+ES gives the best agreement with the selected BPF reference projections in terms of SED, whereas CorrTerms+ES shows better rank histogram calibration in this low-dimensional experiment. We interpret this as a setting-dependent regularization effect of the correction-term architecture: its EnKF-type structure can help preserve ensemble calibration and spread when compatible with the posterior geometry, but it also limits distributional expressivity, as reflected by its larger SED relative to EtE+ES under the identity and square observations.

\begin{figure*}[t!]
    \centering
    \setlength{\tabcolsep}{1.5pt}
    \renewcommand{\arraystretch}{1.05}
    \setlength{\fboxsep}{0pt}
    \setlength{\fboxrule}{0.1pt}

    \resizebox{\textwidth}{!}{%
        \begin{tabular}{>{\centering\arraybackslash}m{0.080\textwidth} *{8}{>{\centering\arraybackslash}m{0.104\textwidth}}}
        \toprule
        & BPF & EtE & {\scriptsize CorrTerms} & EtE & {\scriptsize CorrTerms} & \multirow{2}{*}{ESRF} & \multirow{2}{*}{EnKF} & \multirow{2}{*}{SMF} \\
        & {\scriptsize reference} & + \textbf{ES} & + \textbf{ES} & + NL2 & + NL2 & & & \\
        \midrule
        \LorenzGridRowsNoIEnKF{200}{square} \\
        \midrule
        \LorenzGridRowsNoIEnKF{300}{square}
        \end{tabular}%
    }

    \caption{Visualization of filtering distributions for a test trajectory of $500$ assimilation cycles of the Lorenz '63 problem under the partial square observation. We compare the BPF reference (generated by $10^6$ particles), four machine learning filters (trained with $N=10$ and evaluated with $N=100$), and three benchmark filters. The ensembles are visualized through their two-dimensional projections. At assimilation cycle $200$, the BPF reference filtering distribution exhibits a cross-like structure, while at assimilation cycle $300$ it exhibits a bimodal structure. The ES-based methods capture these non-Gaussian structures better than the NL2-based methods and the benchmark filters. Our proposed PSEF (EtE+ES) provides the closest agreement with the BPF reference.}
    \label{fig:lorenz63_grid}
\end{figure*}

These quantitative findings are visually corroborated by the two-dimensional projected filtering distributions in Figure~\ref{fig:lorenz63_grid}, which illustrates the filtering ensembles for a test trajectory under the partial square observation setting at assimilation cycles $200$ and $300$. Because the square observation introduces sign ambiguities, the BPF posterior exhibits complex non-Gaussian structures, such as a cross-like geometry at assimilation cycle $200$ and a bimodal shape at assimilation cycle $300$. The EtE+ES architecture accurately resolves these multimodal geometries, providing the closest visual agreement with the BPF reference. While the other energy score--based method (CorrTerms+ES) captures part of the non-Gaussian features, it lacks the full structural expressivity of the end-to-end framework. The generally large RHVar values of the NL2-trained methods are consistent with their poor calibration and their failure to capture the multimodality. 
The EnKF and ESRF are inherently constrained to generating unimodal Gaussian approximations and completely miss the posterior structure, while SMF also fails empirically to resolve the cross-like and bimodal geometries in this test; IEnKF diverges. Additional visual results comparing the BPF reference and all tested methods under the partial identity and partial arctan observation functions across various assimilation cycles are provided in Appendix~\ref{app:lorenz63_additional_results}.

PSEF also retains its performance under perturbations of the initial ensemble mean and scales (Appendix~\ref{sssec:L63_shifted_initial}).

\subsection{Lorenz '96}\label{ssec:L96}

The Lorenz '96 model is a standard chaotic system used as a simplified testbed for atmospheric data assimilation \citep{lorenz96}. It consists of cyclically coupled state variables with linear damping, constant forcing, and a quadratic advection term. The dynamics are
\begin{equation}
\label{eq:L96_equations}
    \frac{du_i}{dt}
    =
    (u_{i+1}-u_{i-2})u_{i-1}-u_i+F,
\end{equation}
for $i=1,\ldots,d_v$, with periodic boundary conditions $u_{i+d_v}=u_i$. We use $d_v=40$ and $F=8$, which gives a chaotic dynamical regime. As for Lorenz '63, $\Psi$ here is obtained by integrating \eqref{eq:L96_equations} over $\Delta t$. This subsection first reports the main Lorenz~'96 comparison across ensemble sizes at the default assimilation window $\Delta t=0.15$, and then examines two additional tests: the effect of fine-tuning across target ensemble sizes and direct generalization across different assimilation windows.

As in the Lorenz~'63 experiment, we consider three partial observation operators: identity, square, and arctan. The observed coordinates are
\begin{equation}
\label{eq:l96_observed_indices}
    \mathcal{O}
    =
    \{4i+1\}_{i=0}^{9}.
\end{equation}
Thus $d_y=10$ for all three observation settings. To make the three observation maps comparable, the observation noise standard deviation $\sigma_y$ is calibrated by matching the empirical SNR of the default partial identity observation, as described in Appendix~\ref{app_subsec:snr_calibration}. The resulting settings are summarized in Table~\ref{tab:lorenz96_settings}.

\begin{table}[htbp]
\centering
\caption{Lorenz '96 model settings.}
\label{tab:lorenz96_settings}
\renewcommand{\arraystretch}{1.2}
\begin{tabular}{p{0.24\textwidth}p{0.68\textwidth}}
    \toprule
    \textbf{Category} & \textbf{Values} \\
    \midrule
    Parameters
    & $F=8$ \\
    \midrule
    States
    & $v=(u_1,u_2,\ldots,u_{40})^\top\in\bbR^{d_v}$, $d_v=40$ \\
    \midrule
    Observations
    & Partial + Identity: $h(v)=(u_i)_{i\in\mathcal{O}}$, $d_y=10$, $\sigma_y=1.00$ \\
    & Partial + Square: $h(v)=(u_i^2)_{i\in\mathcal{O}}$, $d_y=10$, $\sigma_y=6.69$ \\
    & Partial + Arctan: $h(v)=(\arctan(u_i))_{i\in\mathcal{O}}$, $d_y=10$, $\sigma_y=0.28$ \\
    \midrule
    Time step
    & Default assimilation window: $\Delta t=0.15$, consisting of $5$ RK4 steps with $dt=0.03$. 
    \\
    &
    In the $\Delta t$-generalization experiment, the RK4 step size remains $0.03$ and the number of RK4 steps between consecutive observations varies from $5$ to $15$. \\
    \bottomrule
\end{tabular}
\end{table}

Because the state dimension is $d_v=40$, constructing an accurate BPF reference is no longer computationally feasible. We therefore evaluate Lorenz~'96 using ES and VS as the primary scores, with RHVar used as a complementary diagnostic for marginal spread calibration. ES can have limited sensitivity to misspecified correlations, motivating the complementary VS evaluation \citep{scheuerer2015variogram}. VS is not used for training or fine-tuning, so it also checks performance using a criterion different from the optimized loss.

The absolute value of ES should be interpreted differently from SED. SED is a distance to a numerical reference filtering distribution, whereas the ES reported here is evaluated against the realized truth state. Even if an ensemble were sampled from the exact filtering distribution, its averaged ES would not be zero. The reason is that the expected ES contains an irreducible self-score of the true filtering distribution: for a non-degenerate filtering distribution, two independent draws from that distribution differ with positive probability, so the corresponding expected distance is positive. Propriety of ES means that this expected score is minimized by the true filtering distribution, but it does not mean that the minimum value is zero. After subtracting this unknown self-score, as in \eqref{eq:sdiv}, the resulting quantity defines the energy score divergence. The square root of this divergence is proportional to the energy distance, which is in fact a metric~\citep{bach2024inverse}. Since the unknown self-score cannot be subtracted in practice, we compute RI directly from the raw ES values using \eqref{eq:relative_improvement}, with the reference ES chosen as the smaller value between EtE+ES and CorrTerms+ES for the same observation map and ensemble size. We analogously compute RI(VS) from the raw VS values, using the smaller VS of EtE+ES and CorrTerms+ES under the same observation map and ensemble size as the reference.

We compare four benchmark filters: EnKF, LETKF, IEnKF, and SMF. LETKF serves as the localized deterministic square-root benchmark in place of ESRF. Hyperparameter selection is summarized in Appendix~\ref{app_subsec:grid_search}, and missing benchmark results indicate filter divergence.

All ML-based filters are first trained at ensemble size $N=10$. Unlike the Lorenz~'63 experiment, where the ML-based filters are evaluated across ensemble sizes without fine-tuning, the Lorenz~'96 results use fine-tuning for each target ensemble size. The CorrTerms architecture follows the fine-tuning procedure of MNMEF \citep{bach_learning_2026}. For the EtE architecture, we use the fine-tuning strategy in Subsection~\ref{subsec:architecture_training}: the set transformer encoder is frozen, and only the residual MLP is fine-tuned at the target ensemble size.

Figure~\ref{fig:L96_metric_curves} compares the tested methods under the three observation settings. For the partial identity observation, CorrTerms+ES and CorrTerms+NL2 achieve the lowest or near-lowest ES values, and their ES curves overlap. The ES curves of EnKF and LETKF also overlap, while their relative differences are more visible in the RI(ES) row. SMF becomes competitive with EnKF and LETKF at $N=100$. The EtE methods are less competitive in this setting. This is consistent with the fact that the observation map is linear, so the filtering distribution is closer to the regime where an EnKF-style correction is an effective inductive bias. 

The partial square observation leads to a different behavior. Since the square map loses sign information, the filtering distribution can become bimodal. In this case, CorrTerms+ES and CorrTerms+NL2 still behave similarly in ES, but they now have the largest ES values among the non-divergent methods and can perform worse than EnKF. This suggests that the EnKF correction-terms architecture is too restrictive when the observation map has multiple plausible preimages. The learned update remains tied to an EnKF-type correction, which limits its ability to represent the bimodal analysis structure. SMF's ES remains above EtE+ES across the tested ensemble sizes. EtE+ES gives the lowest ES in this setting, showing that the more flexible end-to-end architecture is useful when the filtering distribution is strongly non-Gaussian. 

\begin{figure}[htbp]
    \centering
    \includegraphics[width=0.87\textwidth]{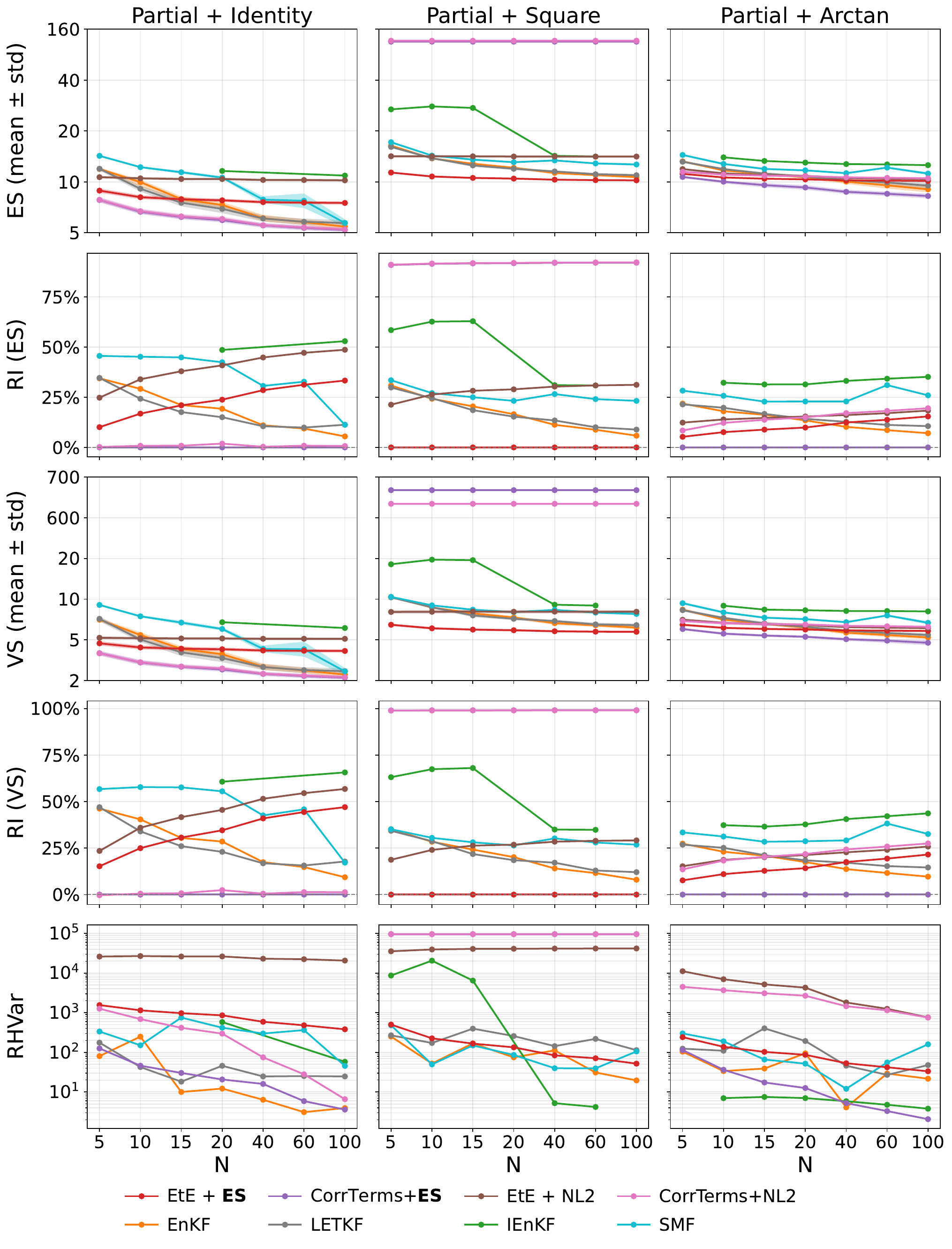}
    \caption{Lorenz~'96 results under the partial identity, square, and arctan observations. Rows report ES, RI(ES), VS, RI(VS), and RHVar. The score means omit divergent trajectories, while $\pm$ standard deviation is shown only when all $64$ outputs are available.}
    \label{fig:L96_metric_curves}
\end{figure}
Under the partial arctan observation, the observation map is nonlinear but remains one-to-one on each observed coordinate. Therefore, unlike the square observation, it does not introduce a sign ambiguity in the observed variables. In this setting, CorrTerms+ES again achieves the lowest ES, as in the partial identity case. SMF improves with ensemble size but does not match the strongest EnKF-type methods. This supports the interpretation that CorrTerms is effective when the filtering distribution remains essentially unimodal, while the EtE architecture becomes important when the observation map induces a more complex posterior geometry.

The VS and RI(VS) rows support the main performance patterns seen in ES and RI(ES). Under partial identity observations, CorrTerms+ES and CorrTerms+NL2 have comparable VS and are among the lowest-scoring methods. Under partial square observations, EtE+ES has the lowest VS across the tested ensemble sizes, while both CorrTerms variants have much larger scores. Under partial arctan observations, CorrTerms+ES has the lowest VS. Thus, the ES-trained architecture that performs best under each observation map also performs well on a dependence-sensitive criterion that was not used for training.

For the fifth-row RHVar results, the comparison between EtE+ES and CorrTerms+ES depends on the observation map. Under partial identity and partial arctan observations, CorrTerms+ES has lower RHVar than EtE+ES across the tested ensemble sizes. Under partial square observations, EtE+ES has the lowest RHVar among the ML-based methods, while both CorrTerms variants have much larger values. The RHVar of EtE+ES decreases with ensemble size in all three settings but remains above $1$ even at $N=100$. The reference expectation of $1$ assumes independent uniform ranks; because the ranks pooled here can be serially dependent, we interpret RHVar qualitatively as a measure of rank-histogram nonuniformity (Appendix~\ref{app_subsec:rhvar_metric}).

Overall, the best ES curve in each Lorenz~'96 setting is obtained by an ES-trained filter, but the preferred architecture depends on the observation-induced filtering distribution. CorrTerms+ES is strongest when the posterior remains closer to a nearly Gaussian regime, as in the partial identity and partial arctan settings. In contrast, under the partial square observation, the sign ambiguity creates a more complex high-dimensional non-Gaussian posterior, and EtE+ES gives both the lowest ES and the lowest RHVar among the ML-based methods. SMF improves markedly with ensemble size under the partial identity and arctan observations, but remains relatively inaccurate under the partial square observation across the tested ensemble sizes.

\subsubsection{Effect of fine-tuning}

\begin{figure}[t]
    \centering
    \includegraphics[width=\textwidth]{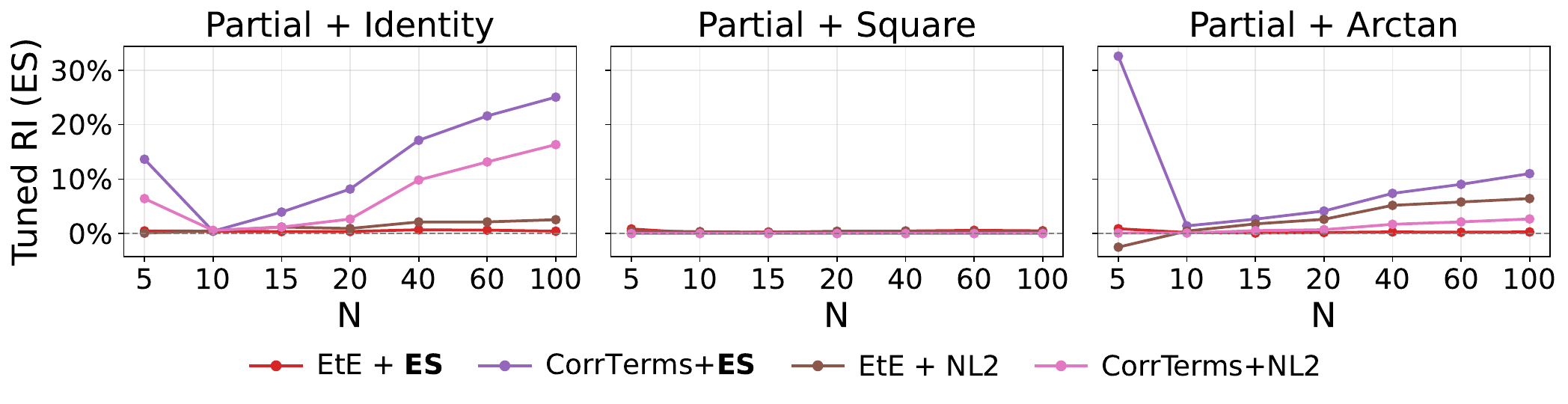}
    \caption{Effect of fine-tuning for the Lorenz~'96 ML filters. All models are pretrained at ensemble size $N=10$ and evaluated at target ensemble sizes $N\in\{5,10,15,20,40,60,100\}$ after fine-tuning. The figure reports the RI in ES of each fine-tuned model over its corresponding pretrained model under the partial identity, partial square, and partial arctan observation settings. Overall, fine-tuning yields a more visible improvement for the CorrTerms architecture, especially under the identity and arctan observations, whereas the gain for the EtE architecture is comparatively small.}
    \label{fig:L96_finetuning_effect}
\end{figure}

Figure~\ref{fig:L96_finetuning_effect} shows that fine-tuning mainly affects the CorrTerms variants. Under the partial identity and partial arctan observations, CorrTerms+ES and CorrTerms+NL2 obtain clear ES reductions after fine-tuning, especially when the test ensemble size differs substantially from the pretraining size $N=10$. This suggests that the correction-terms architecture is relatively sensitive to the ensemble size used during training, so adapting the model to the target ensemble size can noticeably improve its performance. In contrast, the EtE variants change much less after fine-tuning: EtE+ES is already close to its tuned performance in most settings, and EtE+NL2 shows only minor or inconsistent improvement. The partial square observation is the least affected by fine-tuning for all four ML configurations, which is consistent with the main Lorenz~'96 results: in this case, the dominant difficulty is the observation-induced non-Gaussian filtering structure rather than only the mismatch between the pretraining and testing ensemble sizes.

\subsubsection{Generalization across assimilation window length}

The purpose of this experiment is to evaluate whether the filters can be transferred to an assimilation window length different from the one used during training or hyperparameter selection. Throughout the experiment, the RK4 step size is kept fixed at $0.03$, and the assimilation window $\Delta t$ is varied by changing the number of RK4 steps between two consecutive observations. All ML-based filters are trained and fine-tuned only at the default assimilation window $\Delta t=0.15$. Similarly, for the benchmark filters, all applicable hyperparameters are selected only at $\Delta t=0.15$. For all other assimilation windows, we directly reuse the same learned models and the same selected hyperparameters without any additional training, fine-tuning, or hyperparameter optimization.

\begin{figure}[t]
    \centering
    \includegraphics[width=0.75\textwidth]{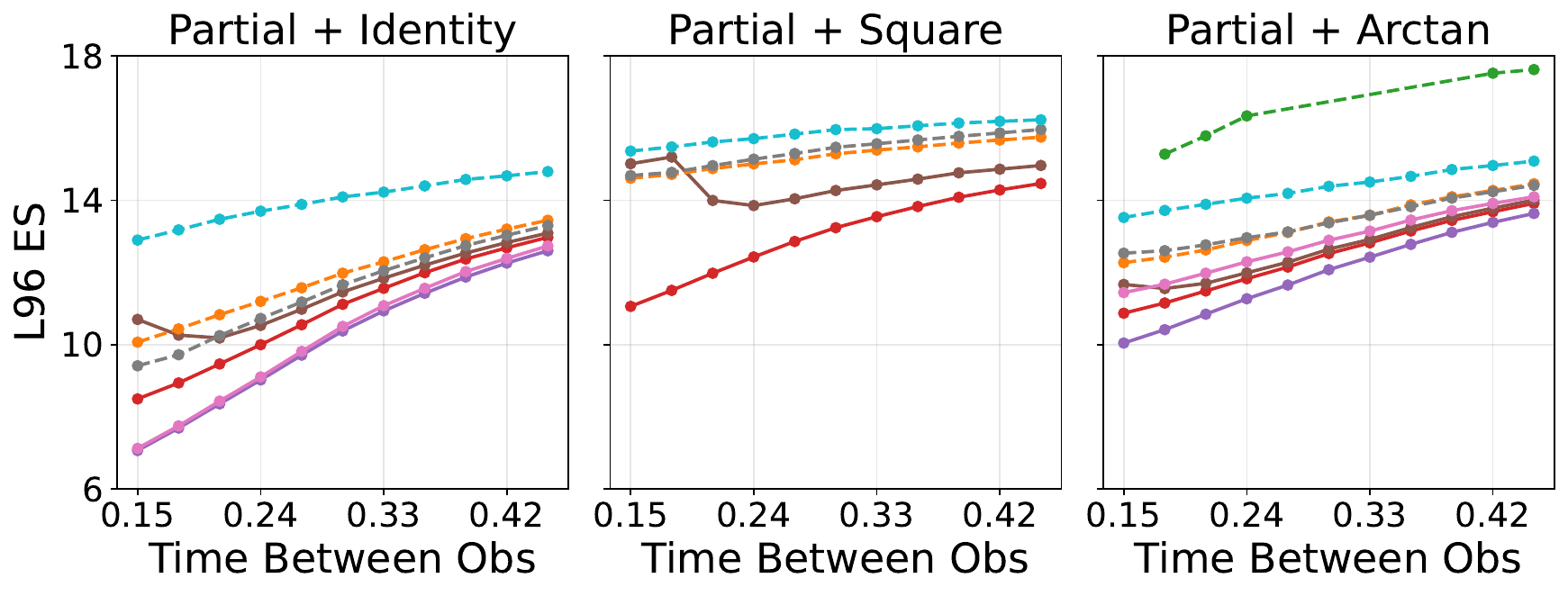}

    \includegraphics[width=0.65\textwidth]{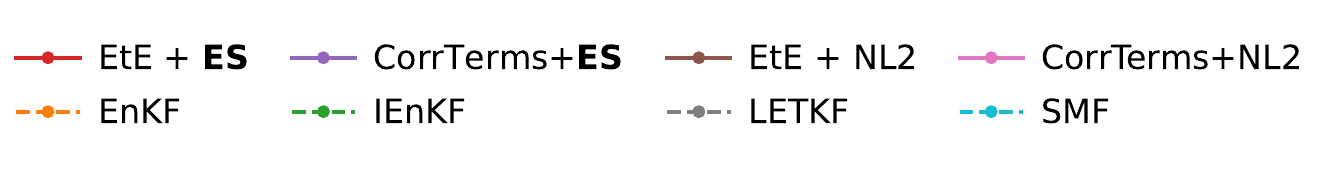}

    \vspace{-0.5em}
    \caption{ES versus assimilation window $\Delta t$ for the Lorenz~'96 problem with $N=10$ under the partial identity, partial square, and partial arctan observation settings. The RK4 step size is fixed at $0.03$, so $\Delta t$ is varied by using $5$ to $15$ RK4 steps between consecutive observations. All ML-based filters are trained and fine-tuned at $\Delta t=0.15$, and the hyperparameters of all benchmark filters are selected at $\Delta t=0.15$ and transferred directly to the other $\Delta t$ values; the ES axis is zoomed to $[6,18]$ to highlight visible method differences. Overall, ES-trained filters remain the most accurate across assimilation windows, with CorrTerms+ES best-performing for identity and arctan observations and EtE+ES best-performing for the square observation.}
    \label{fig:L96_metric_vs_dt}
\end{figure}

Figure~\ref{fig:L96_metric_vs_dt} shows that the strongest ES-trained filters transfer well across the tested assimilation windows, although the filtering problem becomes harder as $\Delta t$ increases. In all three observation settings, the ES values of the methods generally increase with $\Delta t$, reflecting the longer error growth interval between two consecutive observations. Nevertheless, the relative ranking of the strongest methods remains stable. Under the partial identity and partial arctan observations, CorrTerms+ES achieves the lowest ES across the tested range, while EtE+ES remains close but is not uniformly the best. Under the partial square observation, EtE+ES consistently gives the best ES. Thus, the main conclusion from the default $\Delta t=0.15$ experiment continues to hold when the assimilation window changes: ES-trained methods provide the most reliable performance, but the preferred architecture depends on the observation map.

The zoomed ES range further clarifies this distinction. The displayed range $(6,18)$ is used only to make the competitive methods easier to compare: no method has ES below the lower limit, while some poorly performing methods exceed the upper limit. For identity and arctan observations, the filtering distribution is closer to a unimodal regime, so the EnKF correction-terms inductive bias remains effective, and CorrTerms+ES performs best. For the square observation, the sign ambiguity induced by the observation map creates a more complex non-Gaussian filtering distribution, where the end-to-end architecture becomes more important. In this case, both CorrTerms variants exceed the displayed upper limit for part or all of the tested assimilation windows, indicating poor transfer under this observation-induced posterior geometry, whereas EtE+ES remains within the displayed range and retains the best ES performance. At the fixed $N=10$, SMF is among the weakest methods, consistent with its small-ensemble behavior in Figure~\ref{fig:L96_metric_curves}. Among the competitive methods that remain visible in the zoomed range, the gaps become smaller for larger $\Delta t$. Therefore, increasing the assimilation window makes the task harder for all methods and reduces the separation between the strongest visible methods, but it does not change the overall conclusion that ES training gives the most robust transfer across assimilation windows.

\subsection{Computational Time}
\label{ssec:runtime}

We conclude the numerical experiments with a runtime comparison for the Lorenz~'96 benchmark. The purpose of this experiment is to quantify the practical cost of the analysis update, rather than the cost of forecasting between observations. Therefore, for each method and ensemble size $N$, we measure the average wall-clock time of a single analysis step, and exclude the forward-model integration time. The reported values are averaged over test trajectories and assimilation cycles. Because forward-model calls are excluded, this comparison isolates the analysis routine and should not be interpreted as the total forecast--analysis cycle cost of full implementation. We also exclude the training cost of the ML-based filters. For reference, pretraining and ensemble-size fine-tuning timings for the original MNMEF configuration are reported in \citet[Section~5.8, Table~6]{bach_learning_2026}.

For the ML filters, the inference-time cost is determined by the filtering architecture, not by the training loss used to obtain the model parameters; thus, the ES and NL2 variants of a fixed architecture have the same computational graph at test time. We therefore collapse the four ML configurations in Subsection~\ref{ssec:methods} to the two ML architectures, CorrTerms and EtE. We report both CPU and GPU wall-clock timings for CorrTerms and EtE, while EnKF, LETKF, IEnKF, and SMF are reported in CPU-only wall-clock form. All CPU timings are measured on an Intel(R) Core(TM) i9-14900KF, and the ML GPU timings are measured on a single NVIDIA RTX 4080 Super. In all cases, the reported time is measured wall-clock time, rather than an operation count or a hardware-independent complexity measure. Our LETKF implementation is not parallelized over local analysis domains, so the LETKF timing should be interpreted as the wall-clock cost of this particular serial CPU implementation, rather than as an optimized parallel LETKF benchmark. As in the Lorenz~'96 comparison above, LETKF is used as the localized deterministic square-root baseline, so ESRF is not reported separately.

\begin{figure}[t]
    \centering
    \includegraphics[width=0.70\textwidth]{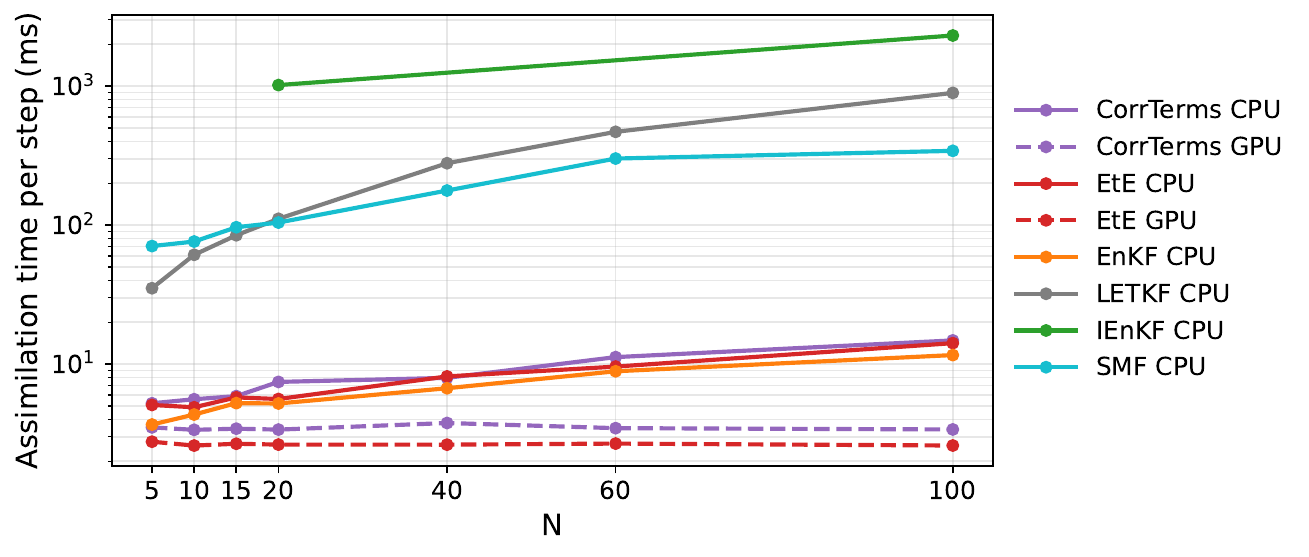}
    \caption{Average wall-clock time per analysis update for the Lorenz~'96 benchmark as a function of the ensemble size $N$. The reported time excludes forward-model integration and measures only the analysis/assimilation update. CorrTerms and EtE are timed on both CPU and GPU, while EnKF, LETKF, IEnKF, and SMF are timed on CPU. CPU timings are measured on an Intel(R) Core(TM) i9-14900KF, and the ML GPU timings are measured on a single NVIDIA RTX 4080 Super. The LETKF timing corresponds to a non-parallelized CPU implementation.}
    \label{fig:L96_assim_time_vs_N}
\end{figure}

Figure~\ref{fig:L96_assim_time_vs_N} shows that, on CPU, both ML architectures have a similar but slightly slower wall-clock analysis time compared to EnKF across the tested ensemble sizes, with a mild increase as $N$ grows. LETKF and IEnKF are substantially more expensive in our implementation. For LETKF, this cost should be interpreted with the caveat that the local analyses are not parallelized; parallelization over localization domains could reduce the observed wall-clock time. The GPU timings for CorrTerms and EtE are lower and more stable across $N$, reflecting the fact that these ML architectures can exploit parallel operations. The CPU curves provide the direct hardware-matched comparison with the benchmark filters, whereas the GPU curves show the accelerator performance available to the ML architectures.

%% file: sections/conclusions.tex
\section{Conclusion and Future Work}
\label{sec:conclusion}

In this paper we introduce the proper scoring ensemble filter (PSEF), a learning-based ensemble data assimilation method for approximating Bayesian filtering distributions. Building on the measure-based filtering framework of \citet{bach_learning_2026}, we train an end-to-end ensemble analysis map with strictly proper scoring rules, so that the resulting ensembles are optimized to represent the full filtering distribution. This allows the filter to be trained from simulated trajectories of state--observation pairs, without requiring samples from, or direct evaluations of, the true filtering distribution.

We recall the population mean-field characterization from \citet[Section~10.3]{bach2024inverse}. Under a realizability assumption, minimizing the proposed population mean-field strictly proper scoring rule objective recovers the true filtering distribution. Under regularity and ergodicity assumptions, the finite-trajectory and population mean-field objectives share the same long-time limit at a fixed realizable exact-filter parameter. This result concerns objective values and does not establish consistency of finite-ensemble minimizers or convergence of neural-network optimization.

The numerical experiments illustrate the distributional accuracy of the learned filters in the tested settings. The energy-score training objective uses $\beta=1$ in all experiments. In the linear--Gaussian setting, PSEF approaches the true posterior given by the Kalman filter up to finite-ensemble sampling error. In nonlinear and non-Gaussian examples, training with the energy score enables the learned ensemble to approximate challenging filtering distributions, including multimodal posteriors that are not captured by the tested Gaussian-based ensemble filters or by the tested learned filters trained with mean-based losses. In the higher-dimensional setting, PSEF also demonstrates competitive performance. These comparisons are limited to the implemented benchmark configurations and do not establish superiority over the broader class of non-Gaussian ensemble filters.

The present method nevertheless has limitations. In the current form, PSEF is not intended for direct application to extremely high-dimensional systems. The analysis map, ensemble interactions, and scoring rule loss all become more demanding as the state dimension and ensemble size increase. A natural next step is therefore to combine the proposed strictly proper scoring rule training approach with learned local structure or latent low-dimensional representations. This direction is particularly relevant for numerical weather prediction, where the full state dimension is extremely large, but the effective degrees of freedom may be much smaller \citep{patil_local_2001}. Learned multiscale local structure \citep{lam_learning_2023} and latent representations \citep{melinc_3d-var_2024} have proven successful in this atmospheric setting. For large ensembles, specialized algorithms for computing the gradients of the energy score exist \citep{hertrich_generative_2024}.

Another limitation of the current approach is that, in learning filters based on synthetic state--observation trajectories, one makes the assumption that the model that produced these trajectories also produced the real observations. In practice, the model is often imperfect. If the model error is small then learned filters could first be trained on synthetic trajectories and then fine-tuned on real observations. However, if it is significant then ideally it too can be modelled and learned \citep{levine_framework_2022}, in tandem with the data assimilation problem, providing an interesting challenge for machine learning.

More broadly, this work points toward a meta-amortized approach to Bayesian filtering. The analysis step can be viewed as a Bayesian inverse operator that conditions a forecast law on a realized observation. At this level of abstraction, the same conditioning principle applies across different dynamical systems, observation operators, observation patterns, and noise covariances; these problem-specific ingredients enter through the joint forecast distribution and the observed data. This suggests the possibility of training a single ``foundation'' meta-amortized filter that learns the Bayesian analysis operator across a family of data assimilation problems, rather than training a separate filter for each fixed setting.

Overall, PSEF provides a step toward learned data assimilation methods that retain the distributional target of Bayesian filtering while avoiding the need for inaccessible posterior supervision. Extending this framework to incorporate local and latent representations, to imperfect models, and ultimately to meta-amortized Bayesian inverse operators is a promising direction for scalable probabilistic data assimilation in realistic scientific forecasting systems.

%% file: sections/appendix.tex
\appendix

\section{Attention Mechanism on Probability Measures}
\label{app:attn_measures}
The attention mechanism \citep{vaswani2017attention} can be viewed as an operator on probability measures. We retain the basic definitions from \citet[Sections~3--4 and Appendix~B]{bach_learning_2026}, to which we refer for the detailed transformer architecture. Consider $d_u,d_w,d_k,d_V\in\bbN$ and fixed matrices $Q\in\bbR^{d_k\times d_u}$, $K\in\bbR^{d_k\times d_w}$, and $V\in\bbR^{d_V\times d_w}$. The single-head attention on probability measure is defined below:

\begin{definition}[Single-head attention on measures]\label{def:attn_measures}
The attention operator is defined by 
\begin{equation}
\label{eq:att_pushforward}
    \sA:\cP(\bbR^{d_u})\times \cP_+(\bbR^{d_w})\to \cP(\bbR^{d_V}), \qquad\sA(\mu,\pp) = \fA(\placeholder;\pp)_\sharp\mu,
\end{equation}
where the pushforward map $\fA:\bbR^{d_u}\times \cP_+(\bbR^{d_w})\to \bbR^{d_V}$ is given, for $u\in\bbR^{d_u}$, by
\begin{align}
    \fA(u;\pp) &= \E_{w\sim p(\cdot;u,\pp)}\,Vw,\label{eq:pushforward_fA}\\
    p(dw;u,\pp) &= \frac{\exp\bigl(\langle Qu,Kw\rangle\bigr)\pp(dw)}{\int_{\bbR^{d_w}}\exp\bigl(\langle Qu,Kz\rangle\bigr)\pp(dz)}.\label{eq:tilted_measure_p}
\end{align}
The domain requirement $\pp\in\cP_+(\bbR^{d_w})$ (Definition~\ref{def:tail_decay}) guarantees the normalizing denominator in~\eqref{eq:tilted_measure_p} is finite for every $u\in\bbR^{d_u}$ \citep{bach_learning_2026}, hence $p(\cdot;u,\pp)$ is well-defined and $\sA(\mu,\pp)\in\cP(\bbR^{d_V})$ for all $\mu\in\cP(\bbR^{d_u})$.
\end{definition}

\begin{definition}[Set of Measures with Fast Tail Decay]
\label{def:tail_decay}
The class $\cP_+(\bbR^d)$ consists of probability measures $\pp$ for which there exist $c_\pp>0$, $R_\pp\geq0$, and $\delta_\pp>0$ such that
\begin{equation}
\pp\bigl(\{x:\|x\|\geq t\}\bigr)\leq\exp\bigl(-c_\pp t^{1+\delta_\pp}\bigr),\qquad t\geq R_\pp.
\end{equation}
\end{definition}

\begin{proposition}[Attention on Empirical Measures]
\label{prop:cross_sequences}
  Let $\mu^N=\tfrac1N\sum_{j=1}^N\delta_{u(j)}$, $\pp^M=\tfrac1M\sum_{k=1}^M\delta_{w(k)}$. We have $\sA(\mu^N,\pp^M)=\tfrac1N\sum_{j=1}^N\delta_{\fA(u(j);\pp^M)}$, where
\[
    \fA\bigl(u(j);\pp^M\bigr)=\frac{\sum_{k=1}^M \exp\langle Qu(j),Kw(k)\rangle\,Vw(k)}
{\sum_{\ell=1}^M \exp\langle Qu(j),Kw(\ell)\rangle}.
\]
This is the standard softmax attention formula \citep{vaswani2017attention}, expressed as a map on empirical measures.
\end{proposition}
\begin{proof}
Substitute the empirical measures into \eqref{eq:att_pushforward}--\eqref{eq:tilted_measure_p}.
\end{proof}

\begin{definition}[Multi-Head Attention on Measures]\label{def:mha_measures}
For $H\in\bbN$ heads of the form \eqref{eq:pushforward_fA}--\eqref{eq:tilted_measure_p}, with head-specific matrices $Q_h,K_h,V_h$, concatenate the outputs and apply $W_O\in\bbR^{d_u\times Hd_V}$:
\begin{equation}
\fA_{\mathrm{MHA}}(u;\pp)=W_O\bigl[\fA_1(u;\pp)^\top,\ldots,\fA_H(u;\pp)^\top\bigr]^\top.
\end{equation}
The corresponding measure-valued operator is $\sA_{\mathrm{MHA}}(\mu,\pp)=\fA_{\mathrm{MHA}}(\placeholder;\pp)_\sharp\mu$.
\end{definition}

\begin{definition}[Transformer Block on Measures]\label{def:trans_block_measure}
Let $\fT(u;\pp)$ combine multi-head attention with residual connections, layer normalization, and a pointwise feedforward network, as specified in \citet[Appendix~B]{bach_learning_2026}. The associated operator is
\begin{equation}
    \sT(\mu,\pp) = \fT(\placeholder; \pp)_\sharp\mu.
\end{equation}
Self-attention uses $\sT(\mu,\mu)$, while cross-attention allows distinct query and key--value measures.
\end{definition}

The transformer measure neural mapping (TMNM) of \citet{bach_learning_2026} extends the set transformer \citep{lee2019set} to probability measures. Self-attention blocks and pooling with $N_s$ learned seed queries produce a fixed-dimensional feature matrix $\sF^{\mathrm{TMNM}}(\mu)\in\bbR^{N_s\times d_u}$. For empirical inputs this yields the permutation-invariant ensemble features used in Subsection~\ref{subsec:architecture_training}; the experiment-specific settings are listed in Appendix~\ref{subsec:appendix_architecture_training}.

\section{Auxiliary Results}\label{app:lemmas}

To contextualize the lemmas presented in this section, we briefly recall the augmented state space formulation introduced in the main text. Note that $\cP_k(\bbR^d)$ denotes the space of probability measures on $\bbR^d$ with finite $k$-th moments Consider $\Omega := \cP_1(\bbR^{d_v}) \times \bbR^{d_v}$, equipped with the metric
\begin{equation}
    \mathrm{D}_\Omega(Z,Z') = W_1(q,q') + \|v-v'\|,
\end{equation}
for any $Z = (q,v)$ and $Z' = (q',v')$ in $\Omega$. Within this metric space, we investigate the augmented process, defined as
\begin{equation}
    Z_j := \bigl(\pi_j, \, v_j^\dag\bigr),
\end{equation}
where $\pi_j = \Law\bigl(v_j^\dag\mid Y_j^\dag\bigr)$ is the true filtering distribution. We denote its law as $\mu_j := \Law(Z_j) \in \cP_1(\Omega)$.

Based on this formulation, we introduce auxiliary lemmas regarding the properties of the energy score and the convergence of expected sequence scores.

\subsection{Lipschitz Continuity of the Energy Score}

We first provide a result on the Lipschitz continuity of the energy score. We recall the definition of the energy score (Definition~\ref{def:energy_score}).

\begin{lemma}[Lipschitz Continuity of the Energy Score with $\beta=1$]\label{lem:ES1_lipschitz}
The energy score (Definition~\ref{def:energy_score}) with $\beta=1$ is Lipschitz continuous on $\Omega$ with a constant $C_\sS = 2$. That is, for $\sS = \sS^\mathrm{ES}_1$, 
\begin{equation}
    |\sS(Z)-\sS(Z')|\leq 2 \mathrm{D}_\Omega(Z,Z').
\end{equation}
\end{lemma}

\begin{proof}
    Let $Z=(q,v)$ and $Z'=(q',v')$ be elements in $\Omega$. By the triangle inequality, the difference in the scoring rule evaluates to
    \begin{equation}
        |\sS^{\mathrm{ES}}_1(Z) - \sS^{\mathrm{ES}}_1(Z')| \le |\sS^{\mathrm{ES}}_1(q,v) - \sS^{\mathrm{ES}}_1(q',v)| + |\sS^{\mathrm{ES}}_1(q',v) - \sS^{\mathrm{ES}}_1(q',v')|.
\end{equation}
    For the first term, expanding the definition of the energy score yields
    \begin{multline}
        |\sS^{\mathrm{ES}}_1(q,v) - \sS^{\mathrm{ES}}_1(q',v)| \le \left| \mathbb{E}^{x\sim q}[\|x-v\|] - \mathbb{E}^{y\sim q'}[\|y-v\|] \right| \\+ \frac{1}{2} \left| \mathbb{E}^{x,x'\sim q}[\|x-x'\|] - \mathbb{E}^{y,y'\sim q'}[\|y-y'\|] \right|.
\end{multline}
    The function $f(x) = \|x-v\|$ is $1$-Lipschitz with respect to the Euclidean norm. By the Kantorovich--Rubinstein duality theorem, it follows that
    \begin{equation}
        \left| \mathbb{E}^{x\sim q}[\|x-v\|] - \mathbb{E}^{y\sim q'}[\|y-v\|] \right| \le W_1(q, q').
\end{equation}
    To bound the interaction term, let $\pp \in \Pi(q, q')$ be an optimal coupling such that $\mathbb{E}^{(x,y)\sim\pp}[\|x-y\|] = W_1(q, q')$. Let $\pp \otimes \pp$ denote the joint distribution of two independent draws from $\pp$, where $\otimes$ denotes the product of measures. Let $(X, Y)$ and $(X', Y')$ be two such independent pairs of random variables distributed according to $\pp$. Under this construction, the marginal distributions satisfy $(X, X') \sim q \otimes q$ and $(Y, Y') \sim q' \otimes q'$. The triangle inequality guarantees $\bigl| \|X-X'\| - \|Y-Y'\| \bigr| \le \|X-Y\| + \|X'-Y'\|$. It follows that the difference of the interaction expectations is bounded by
    \begin{equation}
    \begin{aligned}
        \left| \mathbb{E}^{x,x'\sim q}[\|x-x'\|] - \mathbb{E}^{y,y'\sim q'}[\|y-y'\|] \right|
        &= \left| \mathbb{E}^{(X,Y),(X',Y')\sim \pp\otimes\pp} \bigl[ \|X-X'\| - \|Y-Y'\| \bigr] \right| \\
        &\le \mathbb{E}^{(X,Y),(X',Y')\sim \pp\otimes\pp} \bigl[ \|X-Y\| + \|X'-Y'\| \bigr] \\
        &= 2\mathbb{E}^{(X,Y)\sim \pp} \bigl[ \|X-Y\| \bigr]= 2W_1(q,q').
    \end{aligned}
\end{equation}
    Combining these two bounds, we obtain
    \begin{equation}
        |\sS^{\mathrm{ES}}_1(q,v) - \sS^{\mathrm{ES}}_1(q',v)| \le W_1(q, q') + \frac{1}{2}\bigl(2 W_1(q, q')\bigr) = 2 W_1(q, q').
\end{equation}
    For the second term, the pairwise distance expectations under $q'$ cancel out, leaving
    \begin{equation}
        |\sS^{\mathrm{ES}}_1(q',v) - \sS^{\mathrm{ES}}_1(q',v')| = \left| \mathbb{E}^{x\sim q'}[\|x-v\|] - \mathbb{E}^{x\sim q'}[\|x-v'\|] \right| \le \mathbb{E}^{x\sim q'}\bigl[ \|v-v'\| \bigr] = \|v-v'\|.
\end{equation}
    Therefore,
    \begin{equation}
        |\sS^{\mathrm{ES}}_1(Z) - \sS^{\mathrm{ES}}_1(Z')| \le 2 W_1(q,q') + \|v-v'\| \leq 2 \mathrm{D}_\Omega(Z,Z').
\end{equation}
\end{proof}

\subsection{Convergence Properties of Expected Scores}

Based on the augmented space formulation, we introduce two lemmas below to analytically support the convergence propositions established in the main text.

\begin{lemma}[Change of Variables for Expected Score]\label{lem:variable_change_augmented}
Let $\mu_j = \Law(Z_j)\in\cP_1(\Omega)$ for $j\in\bbZ_+$. Then for any measurable scoring rule $\sS:\Omega\to\bbR_+$, we have
\begin{equation}
\label{eq:variable_change}
    \bbE^{(v_j^\dag,Y_j^\dag)}
    \left[
        \sS(\pi_j,v_j^\dag)
    \right]
    =
    \bbE^{(q,v)\sim\mu_j}
    \left[
        \sS(q,v)
    \right].
\end{equation}
\end{lemma}

\begin{proof}
It suffices to note that there exists a measurable map
\begin{equation}
\mathsf{F}_j:\prod_{k=1}^j\bbR^{d_y}\to\cP_1(\bbR^{d_v}),\qquad \text{s.t. }\pi_j=\mathsf{F}_j(Y_j^\dag).
\end{equation}
Indeed, since the exact filter is initialized at $\pi_0$ and updated recursively by
\begin{equation}
    \pi_{\ell+1}=\sB(\sQ\sP\pi_\ell;y_{\ell+1}^\dag),
    \qquad \ell\in\bbZ_+,
\end{equation}
the filtering distribution $\pi_j$ is deterministically given by the observation history sequence $Y_j^\dag = (y_1^\dag, \ldots, y_j^\dag)$. Therefore, we can write $\mathsf{F}_j$ explicitly as the composition
\begin{equation}
\mathsf{F}_j(y_1^\dag,\ldots,y_j^\dag)
=
\sB\Bigl(
    \sQ\sP\,\sB\bigl(
        \cdots \sB(\sQ\sP\pi_0;y_1^\dag)\cdots
    ;y_{j-1}^\dag\bigr)
;y_j^\dag\Bigr).
\end{equation}
Because the augmented state $Z_j = (\pi_j, v_j^\dag) = \bigl(\mathsf{F}_j(Y_j^\dag), v_j^\dag\bigr)$ is completely determined by the random variables $(v_j^\dag, Y_j^\dag)$, taking the expectation of $\sS(\pi_j, v_j^\dag)$ with respect to the joint law of the physical state and observations is strictly equivalent to taking the expectation of $\sS(q, v)$ with respect to the induced push-forward measure $\mu_j = \Law(Z_j)$ by the law of the unconscious statistician.
\end{proof}

\begin{lemma}[Invariant Integrability and Integral Convergence]\label{lem:invariant_integrability_and_integral_convergence}
Let $\{\pp_j\}_{j\in\bbZ_+}$ be a sequence of probability measures in $\cP_1(\Omega)$ that converges to a limiting measure $\pp_\infty \in \cP(\Omega)$ in the Wasserstein-1 distance, i.e., $\lim_{j\to\infty} W_1(\pp_j, \pp_\infty) = 0$. If a scoring rule $\sS:\Omega\to\bbR_+$ is non-negative and Lipschitz continuous with constant $C_\sS$ with respect to $\mathrm{D}_\Omega$, then $\pp_\infty \in \cP_1(\Omega)$, the scoring rule $\sS$ is integrable with respect to both $\pp_j$ for $j\in\bbZ_+$ and $\pp_\infty$, and we have the limit:
\begin{equation}
    \lim_{j\to\infty}
    \int_{\Omega} \sS(Z)\,\pp_j(\mathrm{d}Z)
    =
    \int_{\Omega} \sS(Z)\,\pp_\infty(\mathrm{d}Z).
\end{equation}
\end{lemma}

\begin{proof}

First, we show that $\pp_\infty \in \cP_1(\Omega)$. Let $\delta_{Z_0}$ be the Dirac measure at an arbitrary reference point $Z_0 \in \Omega$. By the triangle inequality of the Wasserstein-1 metric, we have
\begin{equation}
    W_1(\pp_\infty, \delta_{Z_0}) \le W_1(\pp_\infty, \pp_j) + W_1(\pp_j, \delta_{Z_0}).
\end{equation}
Since $\lim_{j\to\infty} W_1(\pp_j, \pp_\infty) = 0$, the term $W_1(\pp_\infty, \pp_j)$ is finite for sufficiently large $j$. Additionally, the assumption $\pp_j \in \cP_1(\Omega)$ implies $W_1(\pp_j, \delta_{Z_0}) = \int_{\Omega} \mathrm{D}_\Omega(Z, Z_0) \,\pp_j(\mathrm{d}Z) < \infty$. Thus, $W_1(\pp_\infty, \delta_{Z_0}) < \infty$, establishing that $\pp_\infty \in \cP_1(\Omega)$.

Next, we establish the integrability of the scoring rule $\sS$. Since $\sS$ is $C_\sS$-Lipschitz continuous, for any $Z \in \Omega$, we have $\sS(Z) \le \sS(Z_0) + C_\sS \mathrm{D}_\Omega(Z, Z_0)$. Integrating both sides with respect to $\pp_j$ yields
\begin{equation}
    \int_{\Omega} \sS(Z) \,\pp_j(\mathrm{d}Z) \le \sS(Z_0) + C_\sS \int_{\Omega} \mathrm{D}_\Omega(Z, Z_0) \,\pp_j(\mathrm{d}Z).
\end{equation}
Because $\pp_j \in \cP_1(\Omega)$, its first moment is finite, meaning the integral on the right-hand side is bounded. Thus, $\sS$ is integrable with respect to $\pp_j$. By the exact same logic, since we just established $\pp_\infty \in \cP_1(\Omega)$, $\sS$ is also integrable with respect to $\pp_\infty$.

Finally, we prove the integral convergence. Since $\sS$ is absolutely integrable and $C_\sS$-Lipschitz, we directly apply the Kantorovich--Rubinstein duality to the scaled function $\sS / C_\sS$:
\begin{equation}
    \left| \int_{\Omega} \sS(Z) \,\pp_j(\mathrm{d}Z) - \int_{\Omega} \sS(Z) \,\pp_\infty(\mathrm{d}Z) \right| \le C_\sS W_1(\pp_j, \pp_\infty).
\end{equation}
Taking the limit as $j \to \infty$, the assumption $\lim_{j\to\infty} W_1(\pp_j, \pp_\infty) = 0$ ensures that the right-hand side vanishes. Expressing $Z \in \Omega$ as $(q,v)$, the desired convergence result follows immediately.
\end{proof}

\subsection{Ergodicity of the Dynamical Models and the Augmented Process}\label{app:ergodicity}

In Assumption~\ref{da:ergodicity_on_state_filter}(ii), we assume that the exact augmented state--filter process $\{Z_j\}_{j\in\bbZ_+}$ is an ergodic Markov chain admitting a unique invariant measure. In this subsection, we outline theoretical foundations that justify this assumption. The first part of this analysis, 
in Appendix~\ref{app:ergodicity_state}, establishes that the underlying physical state process $\{v_j^\dag\}_{j \in \bbZ_+}$ is ergodic, which serves as a building block in our approach to proving the ergodicity of the joint state--filter process
in Appendix~\ref{app:ergodicity_statef}.

\subsubsection{Ergodicity of the Physical State Process}\label{app:ergodicity_state}
Recall that our generic dynamical model for the unobserved physical state is given by
\begin{equation}
\label{eq:app_state_dyn}
    v_{j+1}^\dag = \Psi(v_j^\dag) + \xi_j^\dag, \qquad j \in \bbZ_+,
\end{equation}
where $\xi_j^\dag \sim \normal(0, \Sigma)$ are independent and identically distributed (i.i.d.) system noise vectors and $\Sigma$ is an invertible covariance matrix. This defines a time-homogeneous Markov process on $\bbR^{d_v}$ for every initial distribution of $v_0^\dag$. Here we sketch a methodology for establishing
ergodicity of the resulting process, drawing on the approach described in \cite{meyn2012markov,hairer2016convergence}. This approach combines an assumed Lyapunov structure induced by the deterministic part of the dynamics, with a probabilistic analysis
confined to a compact subset of phase space, using the controllability implied by the noise having invertible covariance matrix.

To formalize the probabilistic setup, let $q \in \cP(\bbR^{d_v})$ be an initial distribution for the state $v_0^\dag$.
We denote by $\bbP_{q}$ the probability law of the entire state trajectory (the probability measure on the path) initialized with $v_0^\dag \sim q$, i.e.
\begin{equation}
\bbP_q := \Law\left(\{\vd_j\}_{j\in\bbZ_+}\mid \vd_0\sim q\right) \in\cP\left((\bbR^{d_v})^{\bbZ_+}\right).    
\end{equation}
If the initial state is a deterministic point, i.e., $q = \delta_{v_0^\dag}$, we write $\bbP_{v_0^\dag}$. The corresponding expectation operator over these paths is denoted by $\bbE^{\bbP_q}$. Recall $\sP$ from \eqref{eq:operator_P}, the transition kernel of the state process; hence the marginal distribution of $v_j^\dag$ is $\sP^j q$.

For the dynamical models considered in this paper, one can prove (see below) that a, possibly weighted, total variation (TV) distance between any two paths decays exponentially fast. Since the weighted TV distances bound the standard unweighted TV distance $\|\placeholder\|_{\mathrm{TV}}$ we obtain that $\|\sP^n q - \sP^n q'\|_{\mathrm{TV}} \to 0$ with exponential rate as $n \to \infty$, for any initial probability measures $q, q'$. Furthermore, there exists a unique invariant measure $q_\infty$, implying $\|\sP^n q - q_\infty\|_{\mathrm{TV}} \to 0$. Under our strict positivity
assumption on covariance $\Sigma$, $q_\infty$ is absolutely continuous with respect to Lebesgue measure. Indeed, due to the specific form of the kernel (and in particular the fact that $\Sigma > 0$), $\sP q$ is absolutely continuous with respect to Lebesgue measure for any probability measure $q \in \cP(\bbR^{d_v})$, so this is true in particular the invariant measure $q_{\infty} = \sP q_{\infty}$.

Existence and uniqueness, along with exponential convergence, 
can be established by using the  methodology pioneered by \cite{nummelin2004general}.  In particular it rests on an
application of Harris's theorem \cite[Theorem 3.6]{hairer2016convergence}; 
we sketch the proof in our context. Broadly, it needs to be shown that the dynamics admit a Lyapunov function and that the level sets of this Lyapunov function are so-called ``small sets'' \citep{meyn2012markov}.
A Lyapunov function is a nonnegative function $V$ so that
 \begin{equation}
\label{eq:app_lyapunov}
    \E(V(v_{j+1}^\dag) | v_j^\dag=x ) \le \alpha V(x) + r,
\qquad \text{with $0 < \alpha < 1$ and $r > 0$}.
\end{equation}
Taking $V(v^\dag) = \|v^\dag\|^2$ as a Lyapunov function for all the models considered in this paper except the doubling-angle map, which we put aside for the moment, this estimate follows from \emph{a priori} estimates shown, for instance, in \citet[Lemma~9]{oljaca_almost_sure_error_2017} for a class of models which includes the Lorenz~'63 (Subsection~\ref{ssec:L63}) and Lorenz~'96 (Subsection~\ref{ssec:L96});
see, for example,~\cite{mattingly2002ergodicity,law_accuracy_lorenz96_2015}.
For the doubling-angle model (Subsection~\ref{ssec:doublin_angle_model}), a constant Lyapunov function suffices as the entire state space (the compact 1D~torus) turns out to be small.
A set $A$ of the state space is small if there is $0 \leq \alpha < 1$ so that 
\begin{equation}
\label{eq:smallset}
\|\sP(\placeholder,x) - \sP(\placeholder,y)\|_{\mathrm{TV}} 
\leq 2 ( 1 - \alpha)
\qquad \text{for all $x, y \in A$}.
\end{equation}
To prove that the level sets of our Lyapunov function are small sets, we use the Bretagnolle--Huber inequality:
\begin{equation}
\label{eq:bretagnollehuber}
\|\sP(\placeholder,x) - \sP(\placeholder,y)\|_{\mathrm{TV}} 
\leq 2 \big[ 1 - \frac{1}{2}\exp\big(-\mathcal{D}_\text{KL}(\sP(\placeholder,x), \sP(\placeholder,y)) \big) \big].
\end{equation}
For a set to be small it is thus sufficient that the function $(x, y) \mapsto \mathcal{D}_\text{KL}(\sP(\placeholder,x), \sP(\placeholder,y))$ is bounded on that set.
Specialising to our setup we obtain, using the formula for the KL divergence
between Gaussians \cite[Chapter 11]{bach2024inverse},
\begin{equation}
\label{eq:kulback}
\mathcal{D}_\text{KL}(\sP(\placeholder,x), \sP(\placeholder,y))
\leq \frac{1}{2}(\Psi(x) - \Psi(y))^\top \Sigma^{-1}(\Psi(x) - \Psi(y)), 
\end{equation}
which is clearly bounded on level sets of $V$.
We stress however that here we use the non-degeneracy of the noise in our models, i.e., the fact that the covariance matrix $\Sigma$ is full-rank. If the noise is degenerate, the statement of Harris's theorem is no longer true in general without compensating assumptions, and the analysis becomes more complicated -- see, for example, the analysis of the Lorenz '63 model with degenerate noise in \cite{mattingly2002ergodicity}.

A consequence of the exponential ergodicity \cite[Ch. 6, Theorem 1.8]{revuz_markov_1984} is that the algebra of tail events $\mathcal{A}$, and therefore also the algebra of invariant events, is $\bbP_q$-trivial for every initial probability $q$. This means that for any event $A \in \mathcal{A}$, we have either $\bbP_q(A) = 0$ for all $q$, or $\bbP_q(A) = 1$ for all $q$. 
This entails the following generalisation of the ergodic theorem: Let $\phi$ be a function integrable with respect to the invariant measure $q_\infty$. Then the time average satisfies
\begin{equation}
\label{eq:app_ergodic_sum}
    \frac{1}{N} \sum_{j=1}^{N} \phi(v_j^\dag) \to \int \phi \,\mathrm{d}q_\infty,\qquad \bbP_q\text{-almost surely for any initial }q\in\cP(\bbR^{d_v}).
\end{equation}
To see this, note that the event where the limit superior equals the limit inferior ($\limsup = \liminf$) of the left-hand side of \eqref{eq:app_ergodic_sum} is an invariant event. By the ergodic theorem, this event occurs with $\bbP_{q_\infty} = 1$. Due to the triviality property, it also occurs with $\bbP_q = 1$ for any initial distribution $q$. 

With regards to the convergence of these averages in $L^1$ with respect to $\bbP_q$ for an arbitrary $q\in\cP(\bbR^{d_v})$, we require the condition that the random variables $\{\phi(v_j^\dag)\}_{j \ge 1}$ are uniformly integrable. This uniform integrability is satisfied if initialized at $q = q_\infty$, since the sequence $\{v_j^\dag\}$ then consists of identically distributed random variables.

\subsubsection{Ergodicity of the Joint State--Filter Process} \label{app:ergodicity_statef}
We now use the results from Appendix~\ref{app:ergodicity_state}, concerning
ergodicity of the physical state process $\{v_j^\dag\}_{j \in \bbZ_+}$,
to establish ergodicity of the joint state--filter augmented process. We
consider the settings where the filter is either the true filtering distribution,
or some parameterized approximation. We employ the notation established in Subsection \ref{ssec:bridge_pop_empirical}. To this end, recall the definition of the augmented process $Z_j = (q_j, v_j^\dag) \in \Omega = \cP(\bbR^{d_v}) \times \bbR^{d_v}$ defined
by \eqref{eq:DA_setting} and \eqref{eq:q_update}. For the existence argument, $\Omega$ carries the product of the weak topology on $\cP(\bbR^{d_v})$ and the Euclidean topology on $\bbR^{d_v}$. The moment bounds below recover the finite-first-moment setting used in Subsection~\ref{ssec:bridge_pop_empirical}. The transition from $j$ to $j+1$ of the process $\{Z_j\}_{j \in \bbZ_+}$ is determined by the autonomous evolution of the signal and the recursive update of the filter:
\begin{subequations} \label{eq:cite}
\begin{align}
    v_{j+1}^\dag &= \Psi(v_j^\dag) + \xi_j^\dag,\quad \xi_j^\dag\sim\normal(0,\Sigma), \label{eq:j10} \\
    y_{j+1}^\dag &= h(v_{j+1}^\dag) + \eta_{j+1}^\dag,\quad \eta_{j+1}^\dag\sim\normal(0,\Gamma), \label{eq:j20} \\
    q_{j+1} &= \sB_\theta\bigl(\sQ\sP q_j; y_{j+1}^\dag\bigr), \label{eq:j30}
\end{align}
\end{subequations}
If $\sB$ is realizable for some choice $\theta=\theta_B$ then this formulation
also includes the true filter. We make this realizability assumption simply to streamline the analysis that follows to apply to both the true filter and the approximate filter; it is not otherwise needed.
We fix the parameter $\theta$ and write $\mathsf{F}$ for the map:
\begin{equation}
\label{eq:F-map}
    \mathsf{F} : \cP(\bbR^{d_v}) \times \bbR^{d_y} \to \cP(\bbR^{d_v}); \quad (q, y) \mapsto \sB_\theta\bigl(\sQ\sP q; y\bigr),
\end{equation}
which iterates the approximate filter. Using \eqref{eq:cite},
we could include the observation $y_{j+1}^\dag$ into the augmented state $Z_j$ as a third component, but doing so is unnecessary -- see \eqref{eq:q_update1}. To this end
we write \eqref{eq:cite} as
\begin{subequations} \label{eq:cite2}
\begin{align}
    v_{j+1}^\dag &= \Psi(v_j^\dag) + \xi_j^\dag,\quad \xi_j^\dag\sim\normal(0,\Sigma),\\
    q_{j+1} &= \mathsf{F}\bigl(q_j; h(v_{j+1}^\dag) + \eta_{j+1}^\dag\bigr).
\end{align}
\end{subequations}

As in the preceding subsection we assume that $\Sigma$ is an invertible covariance matrix. We establish an ergodicity result for the joint process $Z_j$ defined
by \eqref{eq:cite2}, relying on the state ergodicity (Appendix~\ref{app:ergodicity_state}) and a stability assumption on the filter dynamics. As discussed, given an initial probability measure $q\in\cP(\bbR^{d_v})$, equation~\eqref{eq:j10} defines the state path distribution $\bbP_q$; we use the same notation for its extension with the independent observation noises in~\eqref{eq:j20}. Likewise, given $\mu\in\cP(\Omega)$, the full model \eqref{eq:j10}--\eqref{eq:j30} defines a trajectory distribution $\bbP_{\mu,\Omega}$ for $\{Z_j\}_{j\in\bbZ_+}$, i.e.,
\begin{equation}
\bbP_{\mu,\Omega} := \Law\left(\{Z_j\}_{j\in\bbZ_+}\mid Z_0\sim \mu\right) \in\cP\left((\Omega)^{\bbZ_+}\right).    
\end{equation}
The marginal distribution of the physical state process under $\bbP_{\mu,\Omega}$ is $\bbP_{q}$, where $q$ is the projection of $\mu$ onto the state space.

\begin{assumption}[on the filter]\label{ass:stability}
\begin{enumerate}
\item There exists a distance $d$ which is allowed to take the value $+\infty$ and is lower semicontinuous with respect to weak convergence of both arguments on $\cP(\bbR^{d_v})$ so that with $\bbP_{q_{\infty}}$--probability one, for any two sequences of filtering distributions defined through
\begin{equation}
\label{eqn:stability}
q_{j+1} = \mathsf{F}(q_j, y_{j+1}^\dag), \qquad
q'_{j+1} = \mathsf{F}(q'_j, y_{j+1}^\dag), \qquad
\text{for $j \in \bbZ_+$}.
\end{equation}
with arbitrary $q_0,q'_0\in\cP(\bbR^{d_v})$ we have
\begin{equation}
\label{eqn:stability_convg}
d(q_j, q'_j) \to 0 \qquad \text{as~} j \to \infty.
\end{equation}
\item There exists an initial state distribution $q_\star \in \cP(\bbR^{d_v})$ and an initial filter condition $q_0 \in \cP(\bbR^{d_v})$ such that the laws of $\{q_j\}_{j \in \bbZ_+}$ (defined in Eq.~\eqref{eqn:stability} under $\bbP_{q_\star}$) are uniformly tight for the weak topology on $\cP(\bbR^{d_v})$.
\item The mapping $\mathsf{F}: \cP(\bbR^{d_v}) \times \bbR^{d_y} \to \cP(\bbR^{d_v})$ is continuous (assuming the weak topology on $\cP(\bbR^{d_v})$).
\end{enumerate}
\end{assumption}

The initial probability $q_{\infty}$ in Assumption~\ref{ass:stability}, item~1, is the invariant probability of the state process (considered in Appendix~\ref{app:ergodicity_state}). The exceptional set of zero probability, where the convergence in item~1 does not hold, is assumed to be independent of $q_0, q'_0$.

For an overview of nonlinear filter stability under misspecified initial distributions, see~\cite{chigansky_intrinsic_stability_2011}. Stability of the approximate filters considered here remains to be established. Item~2 imposes a tightness condition on the approximate filter; the exact filter satisfies this condition when the physical state is stationary and the filter is initialized consistently. The continuity condition in item~3 needs to be verified for each filter and may also be needed to relate finite ensembles to the mean-field limit.

The choice $d=W_1$ satisfies the lower semicontinuity requirement, interpreting $W_1$ as an extended distance on $\cP(\bbR^{d_v})$; it is finite on $\cP_1(\bbR^{d_v})$. Although $W_1$ convergence implies weak convergence, it is stronger on noncompact state spaces. Tightness and continuity in items~2--3 use the weak topology, independently of the distance used to measure stability in item~1. Choosing $d=W_1$ does not itself establish that stability.
To measure convergence on the augmented space $\Omega$, we use the bounded metric
\begin{equation}
\label{eq:app_D_metric}
    \hat{\mathrm{D}}_{\Omega}\bigl((q, v), (q', v')\bigr)
    :=
    \mathbf{1}_{v \neq v'} + \min\{d(q, q'),1\}.
\end{equation}

\begin{remark}[Regarding the metric $\hat{\mathrm{D}}_{\Omega}$ and physical marginals]\label{rmk:total_metric}
The term $\mathbf{1}_{v \neq v'}$ in Equation~\eqref{eq:app_D_metric} is the trivial metric, defined as $0$ if $v = v'$ and $1$ otherwise. Moreover, the Wasserstein metric evaluated with respect to the trivial metric coincides with the total variation distance. Therefore, if $\mu$ and $\mu'$ are joint distributions on $\Omega$, and if $q$ and $q'$ denote their respective marginals on the physical state space, then convergence of the joint process in $W_{\hat{\mathrm{D}}_{\Omega}}$ implies convergence of the physical marginals in total variation. More precisely,
\begin{equation}
    \|\bbP_{q}(v_j^\dag \in \placeholder) - \bbP_{q'}(v_j^\dag \in \placeholder)\|_{\mathrm{TV}}
    \le
    W_{\hat{\mathrm{D}}_{\Omega}}
    \Bigl(
        \bbP_{\mu,\Omega}(Z_j \in \placeholder),
        \bbP_{\mu',\Omega}(Z_j \in \placeholder)
    \Bigr).
\end{equation}
\end{remark}

\begin{theorem}\label{thm:ergodic_X}
Suppose that Assumption~\ref{ass:stability} holds, that $\Psi$ and $h$ are continuous, and that there exists an invariant probability measure $q_\infty$ for the state process $\{v_j^\dag\}_{j\in\bbZ_+}$ such that $\|\sP^j q - q_\infty\|_{\mathrm{TV}} \to 0$ as $j\to\infty$ for any probability $q\in\cP(\bbR^{d_v})$. Then,
for any two initial joint probability measures $\mu, \mu'$ on $\Omega$, we have
\begin{equation}
\label{eq:j90}
    W_{\hat{\mathrm{D}}_{\Omega}}
    \Bigl(
        \bbP_{\mu,\Omega}(Z_j \in \placeholder),
        \bbP_{\mu',\Omega}(Z_j \in \placeholder)
    \Bigr)
    \to 0
    \qquad
    \text{as } j \to \infty,
\end{equation}
where $W_{\hat{\mathrm{D}}_{\Omega}}$ is the Wasserstein metric induced by the distance $\hat{\mathrm{D}}_{\Omega}$.
Furthermore, there exists a unique invariant probability measure $\mu_\infty$ for the augmented state process $\{Z_j\}_{j\in\bbZ_+}$.
\end{theorem}

Ergodicity of the joint state–filter augmented process has been demonstrated for instance in~\cite{BUDHIRAJA2003919}, albeit only for the exact filter (and under different conditions which ensure stability of the exact but not necessarily any approximating filter). Furthermore, the method of proof differs from the
one we in introduce here. We prove Theorem~\ref{thm:ergodic_X} using the theory of coupling, the basics of which we now outline. Let $\{U_j\}_{j\in\bbZ_+}$ and $\{U'_j\}_{j\in\bbZ_+}$ be two stochastic processes on the same state space defined on probability spaces $(\Omega_U, \mathcal{A}_U, \bbP_U)$ and $(\Omega'_U, \mathcal{A}'_U, \bbP'_U)$. A coupling of $U, U'$ is a pair process $\{(\hat{U}_j, \hat{U}'_j)\}_{j\in\bbZ_+}$ on a joint probability space $(\hat{\Omega}, \hat{\mathcal{A}}, \bbM)$ such that $\Law(\hat{U}) = \Law(U)$ and $\Law(\hat{U}') = \Law(U')$. For such a coupling, consider the coupling time
\begin{equation}
\label{eq:j110}
    T := \min\bigl\{j \in \bbZ_+ \mid \hat{U}_k = \hat{U}'_k \text{ for all } k \ge j\bigr\}.
\end{equation}
Here $T \in \bbZ_+\cup\{\infty\}$ is a random variable, with $\min\varnothing=\infty$; we call the coupling successful if $T < \infty$ almost surely. A fundamental fact of coupling theory \citep{lindvall2002lectures} states that for any coupling,
\begin{equation}
\label{eq:j120}
    \|\bbP_U(U_j \in \placeholder) - \bbP'_U(U'_j \in \placeholder)\|_{\mathrm{TV}} \le 2\bbM(T > j).
\end{equation}
For two copies of the same time-homogeneous Markov chain, convergence of their time-$j$ marginal laws in total variation implies the existence of a successful coupling \citep{lindvall2002lectures}. We use this fact for the physical state process below.

\begin{proof}[of Theorem~\ref{thm:ergodic_X}]
The existence of $\mu_\infty$ follows from a Krylov--Bogoliubov argument:

Let $\mu_j$ be the law of $Z_j = (q_j,v_j^\dag)$ for $j \in \bbZ_+$, as defined in~\eqref{eq:cite2}, with $\mu_0 := \delta_{q_0}\otimes q_\star$, where $q_\star$ and $q_0$ are as in Assumption~\ref{ass:stability}, item~2. By that assumption, the laws of $q_j$ are uniformly tight. The laws of $v_j^\dag$ are also uniformly tight by the ergodicity assumption on the state process. Uniform tightness of both marginals implies uniform tightness of the joint laws $\mu_j$.

To complete the Krylov--Bogoliubov argument we need to show that the transition kernel of $\{Z_j\}_{j \in \bbZ_+}$ is Feller. Let $\phi: \Omega\to\bbR$ be continuous and bounded, and consider a sequence $X_j = (p_j,u_j)\in\Omega$ with $X_j\to X=(p,u)$ as $j\to\infty$. Clearly $\Psi(u_j) + \xi \to \Psi(u) + \xi$ as $j\to\infty$ for all $\xi$, since $\Psi$ is continuous. Similarly $\mathsf{F}(p_j,h(\Psi(u_j)+\xi)+\eta)\to\mathsf{F}(p,h(\Psi(u)+\xi)+\eta)$ as $j\to\infty$ for all $\xi,\eta$ by Assumption~\ref{ass:stability}, item~3 and since $h$ is continuous. Therefore $\phi(\mathsf{F}(p_j,h(\Psi(u_j)+\xi)+\eta),\Psi(u_j)+\xi)\to\phi(\mathsf{F}(p,h(\Psi(u)+\xi)+\eta),\Psi(u)+\xi)$, and since $\phi$ is bounded, this convergence also holds if we integrate over $\xi$ and $\eta$ (which are independent with normal distributions). This proves the Feller property of the transition kernel of $\{Z_j\}_{j \in \bbZ_+}$.

The Cesaro averages $K^{-1}\sum_{j=0}^{K-1}\mu_j$ are tight, so the Krylov--Bogoliubov theorem yields an invariant measure $\mu_\infty$.

The uniqueness of the invariant measure is a consequence of the stability result \eqref{eq:j90} which we now prove. Let $\mu,\mu'\in\cP(\Omega)$. It is sufficient to prove the stability result \eqref{eq:j90} when $\mu$ has the marginal state distribution $q_{\infty}$, the probability that appears in Assumption~\ref{ass:stability}, item~1, and $\mu'$ is arbitrary.

The case of two arbitrary initial distributions follows via the triangle inequality.

Write $q = q_{\infty}$ and $q'$ for the marginals of $\mu$ and $\mu'$ on the physical state space $\bbR^{d_v}$. The distributions of the state trajectories under $ \bbP_{\mu,\Omega}$ and $ \bbP_{\mu',\Omega}$ are given by the path measures $\bbP_{q}$ and $\bbP_{q'}$. Since the state process is ergodic, we have
\begin{equation}
\label{eq:j130}
    \|\bbP_{q}(v_j^\dag \in \placeholder) - \bbP_{q'}(v_j^\dag \in \placeholder)\|_{\mathrm{TV}} = \|\sP^j q - \sP^j q'\|_{\mathrm{TV}} \to 0,\qquad \text{ as }j\to\infty.
\end{equation}
We conclude from \eqref{eq:j130} and the coupling result that there exists a coupling $\{(\hat{U}_j, \hat{U}'_j)\}$ on some space $(\hat{\Omega}, \hat{\mathcal{A}}, \bbM)$ such that 
\begin{equation}
    \Law\bigl((\hat U_j)_{j\in\bbZ_+}\bigr)=\bbP_q,
\qquad
\Law\bigl((\hat U'_j)_{j\in\bbZ_+}\bigr)=\bbP_{q'},
\end{equation}
and
\begin{equation}
\label{eq:j140}
    \bbM\left(\hat{U}_k = \hat{U}'_k \text{ for all } k \ge j\right) = \bbM(T \le j) \to 1 \qquad \text{as } j \to \infty.
\end{equation}

We now extend this state coupling to a coupling of the augmented processes. Let $\{\eta_j^\dag\}_{j\in\bbZ_+}$ be an i.i.d.\ sequence with distribution $\normal(0,\Gamma)$, independent of the coupled state trajectories $(\hat U,\hat U')$. We define
\begin{equation}
    \hat{y}_j := h(\hat{U}_j)+\eta_j^\dag,
    \qquad
    \hat{y}'_j := h(\hat{U}'_j)+\eta_j^\dag.
\end{equation}

Since $\hat U$ has path law $\bbP_{q_{\infty}}$ and the noise sequence $\{\eta_j^\dag\}$ is independent with the correct distribution, the pair $(\hat U,\hat y)$ has the same state--observation law as the model initialized from $q_{\infty}$. Similarly, $(\hat U',\hat y')$ has the correct state--observation law initialized from $q'$.

Choose $\hat q_0$ and $\hat q'_0$ from the regular conditional laws of the filter components given $\hat U_0$ and $\hat U'_0$, using auxiliary randomness independent of the coupled state paths and observation noise, such that
\begin{equation}
    \Law(\hat q_0,\hat U_0)=\mu,
    \qquad
    \Law(\hat q'_0,\hat U'_0)=\mu'.
\end{equation}
Then define
\begin{equation}
    \hat q_{j+1}=\mathsf{F}(\hat q_j,\hat y_{j+1}),
    \qquad
    \hat q'_{j+1}=\mathsf{F}(\hat q'_j,\hat y'_{j+1}).
\end{equation}
where $\mathsf{F}$ is the map defined in Equation~\eqref{eq:F-map} which iterates the approximate filter. 
The resulting processes
\begin{equation}
    \hat Z_j=(\hat q_j,\hat U_j),
    \qquad
    \hat Z'_j=(\hat q'_j,\hat U'_j)
\end{equation}
form a valid coupling of the augmented processes with laws $\bbP_{\mu,\Omega}$ and $\bbP_{\mu',\Omega}$. Evaluating the Wasserstein metric $W_{\hat{\mathrm{D}}_{\Omega}}$ using this coupling, and letting $\bbE^\bbM$ denote the expectation with respect to $\bbM$, we obtain
\begin{equation}
\label{eq:j150}
\begin{aligned}
    &W_{\hat{\mathrm{D}}_{\Omega}}
    \Bigl(
        \bbP_{\mu,\Omega}(Z_j \in \placeholder),
        \bbP_{\mu',\Omega}(Z_j \in \placeholder)
    \Bigr) \\
    &\le
    \bbE^\bbM
    \bigl[
        \hat{\mathrm{D}}_{\Omega}(\hat{Z}_j, \hat{Z}'_j)
    \bigr] \\
    &=
    \bbM(\hat{U}_j \neq \hat{U}'_j)
    +
    \bbE^\bbM
    \bigl[
        \min\{d(\hat{q}_j, \hat{q}'_j),1\}
    \bigr].
\end{aligned}
\end{equation}
The first term goes to zero since $\bbM(\hat{U}_j \neq \hat{U}'_j) \le \bbM(T > j) \to 0$. Because the state coupling is successful, there exists an almost surely finite time $T$ such that $\hat U_k=\hat U'_k$ for all $k\ge T$. By the construction with the same observational noise, this implies
\begin{equation}
    \hat y_k=\hat y'_k
    \qquad
    \text{for all } k\ge T.
\end{equation}
Hence, from time $T$ onward, the two filter recursions are driven by the same sequence of observations. 

Now let $\mathfrak{Y}$ be a measurable set of observation sequences so that whenever $\{y_k^\dag\}\in\mathfrak{Y}$, any two sequences of filtering distributions defined as in Equation~\ref{eqn:stability} will exhibit the convergence in~\eqref{eqn:stability_convg}. By item~1 of Assumption~\ref{ass:stability}, we can choose $\mathfrak{Y}$ so that $\bbP_{q_{\infty}}(\{y_k^\dag\}\in\mathfrak{Y})=1$. Now let
\begin{equation}
\mathfrak{Y}_0:=\big\{\{y_k\}_{k\ge1}\in\mathfrak{Y};\;
\{y_{k+j}\}_{k\ge1}\in\mathfrak{Y}\text{ for infinitely many }j\in\bbZ_+\big\},
\end{equation}
then by stationarity and a countable intersection $\bbP_{q_{\infty}}(\{y_k^\dag\}\in\mathfrak{Y}_0)=1$. On the other hand, if $\{\hat y_k\}\in\mathfrak{Y}_0$ and $T<\infty$, then we can choose $r\ge T$ such that $\{\hat y_{k+r}\}_{k\ge1}\in\mathfrak{Y}$. The same exceptional set applies to every pair of initial filter distributions, so stability applies from time $r$ to $\hat q_r,\hat q'_r$, even if they depend on the future observation tail. We can therefore conclude that
\begin{equation}
d(\hat q_j,\hat q'_j)\to0,
\end{equation}
$\bbM$--almost~surely.
Since $\min\{d(\hat{q}_j, \hat{q}'_j),1\}$ is bounded by $1$, the dominated convergence theorem guarantees that the second term in \eqref{eq:j150} converges to zero. 
\end{proof}
\begin{corollary}[Convergence of marginal laws]\label{cor:marginal_convergence_hatD}
Let the assumptions of Theorem~\ref{thm:ergodic_X} hold, and let $\mu_\infty \in \cP(\Omega)$ denote the unique invariant measure of the augmented process. For any initial law $Z_0\sim\mu \in \cP(\Omega)$, consider
\begin{equation}
    \mu_j = \Law(Z_j) = \bbP_{\mu,\Omega}(Z_j\in\cdot), \qquad j \in \bbZ_+,
\end{equation}
then we have
\begin{equation}
    W_{\hat{\mathrm{D}}_\Omega}(\mu_j,\mu_\infty)
    \to 0
    \qquad
    \text{as } j \to \infty.
\end{equation}
\end{corollary}

\begin{proof}
Apply Theorem~\ref{thm:ergodic_X} with $\mu' = \mu_\infty$. Since $\mu_\infty$ is invariant, we have
\begin{equation}
    \Law_{\bbP_{\mu_\infty,\Omega}}(Z_j) = \mu_\infty
    \qquad \text{for all } j \in \bbZ_+.
\end{equation}
Therefore,
\begin{equation}
    W_{\hat{\mathrm{D}}_\Omega}(\mu_j,\mu_\infty)
    =
    W_{\hat{\mathrm{D}}_\Omega}
    \Bigl(
        \bbP_{\mu,\Omega}(Z_j \in \placeholder),
        \bbP_{\mu_\infty,\Omega}(Z_j \in \placeholder)
    \Bigr)
    \to 0,
\end{equation}
which proves the claim.
\end{proof}

Theorem~\ref{thm:ergodic_X} provides the existence of the invariant measure and convergence to it in the Wasserstein distance induced by the bounded metric $\hat{\mathrm{D}}_\Omega$. However, the convergence statement in Assumption~\ref{da:ergodicity_on_state_filter}(ii) is formulated using the Wasserstein distance induced by the unbounded metric
\begin{equation}
\label{eq:app_DOmega_unbounded_metric}
    \mathrm{D}_\Omega\bigl((q,v),(q',v')\bigr)
    =
    W_1(q,q')+\|v-v'\|.
\end{equation}
We now explain how the convergence in $W_{\hat{\mathrm{D}}_\Omega}$ can be upgraded to convergence in $W_{\mathrm{D}_\Omega}$ once the augmented process has the required first-moment control. The key point is that $\hat{\mathrm{D}}_\Omega$ controls the bounded truncation $\min\{\mathrm{D}_\Omega, 1\}$, while uniform integrability controls the tails of the unbounded cost $\mathrm{D}_\Omega$. 

Crucially, while Theorem~\ref{thm:ergodic_X} and Corollary~\ref{cor:marginal_convergence_hatD} apply to any generic filter recursion satisfying the stability Assumption~\ref{ass:stability}, establishing the necessary uniform integrability for the unbounded metric requires specific structural properties. Therefore, for the remainder of this subsection, we restrict our focus to the \emph{exact augmented process}, where the filter is the true Bayesian posterior $q_j = \pi_j :=\Law\bigl(v_j^\dag\mid Y_j^\dag\bigr)$.

\begin{lemma}[Moment control for the exact augmented process]\label{lem:augmented_moment_control}
Let \(Z_j=(\pi_j,v_j^\dag)\) be the exact augmented process. Let \(Z_0\sim\mu\in\cP(\Omega)\), and let \(q\in\cP(\bbR^{d_v})\) be the projection of \(\mu\) onto the physical state space. Assume that the exact filter is initialized consistently with the state--observation path law under \(\bbP_{\mu,\Omega}\), so that \(\pi_j=\bbP_{\mu,\Omega}(v_j^\dag\in\cdot\mid Y_j^\dag)\) for every \(j\in\bbZ_+\). Suppose that the physical state process satisfies the quadratic Lyapunov estimate \eqref{eq:app_lyapunov}, that is, \eqref{eq:app_lyapunov} holds with \(V(v^\dag)=\|v^\dag\|^2\). If \(\bbE^{\bbP_q}\|v_0^\dag\|^2<\infty\), then \(\sup_{j\ge0}\bbE^{\bbP_q}\|v_j^\dag\|^2<\infty\). Consequently, the family \(\{\|v_j^\dag\|\}_{j\ge0}\) is uniformly integrable under \(\bbP_q\):
\begin{equation}
    \lim_{R\to\infty}\sup_{j\ge0}\bbE^{\bbP_q}\left[\|v_j^\dag\|\mathbf{1}_{\{\|v_j^\dag\|>R\}}\right]=0.
\end{equation}
Moreover, for \(Z_\star=(\delta_0,0)\), the family \(\{\mathrm{D}_\Omega(Z_j,Z_\star)\}_{j\ge0}\) is uniformly integrable under \(\bbP_{\mu,\Omega}\):
\begin{equation}
    \lim_{R\to\infty}\sup_{j\ge0}\bbE^{\bbP_{\mu,\Omega}}\left[\mathrm{D}_\Omega(Z_j,Z_\star)\mathbf{1}_{\{\mathrm{D}_\Omega(Z_j,Z_\star)>R\}}\right]=0.
\end{equation}
\end{lemma}

\begin{proof}
The physical state marginal of \(\bbP_{\mu,\Omega}\) is \(\bbP_q\). By the quadratic Lyapunov estimate and the tower property,
\begin{equation}
    \bbE^{\bbP_q}\|v_{j+1}^\dag\|^2
    =\bbE^{\bbP_q}\left[\bbE^{\bbP_q}\left(\|v_{j+1}^\dag\|^2\mid v_j^\dag\right)\right] \le \alpha\bbE^{\bbP_q}\|v_j^\dag\|^2+r.
\end{equation}
Iterating this inequality yields
\begin{equation}
    \bbE^{\bbP_q}\|v_j^\dag\|^2
    \le \alpha^j\bbE^{\bbP_q}\|v_0^\dag\|^2+\frac{r(1-\alpha^j)}{1-\alpha}
    \le \bbE^{\bbP_q}\|v_0^\dag\|^2+\frac{r}{1-\alpha}.
\end{equation}
Hence \(\sup_{j\ge0}\bbE^{\bbP_q}\|v_j^\dag\|^2<\infty\). This implies uniform integrability of \(\{\|v_j^\dag\|\}_{j\ge0}\), since
\begin{equation}
    \sup_{j\ge0}\bbE^{\bbP_q}\left[\|v_j^\dag\|\mathbf{1}_{\{\|v_j^\dag\|>R\}}\right]
    \le \frac{1}{R}\sup_{j\ge0}\bbE^{\bbP_q}\|v_j^\dag\|^2\to0, \text{ as } R\to\infty.
\end{equation}

It remains to pass from the physical state to the exact augmented state. Since \(\pi_j=\bbP_{\mu,\Omega}(v_j^\dag\in\cdot\mid Y_j^\dag)\), we have
\begin{equation}
    \mathrm{D}_\Omega(Z_j,Z_\star)
    =W_1(\pi_j,\delta_0)+\|v_j^\dag\|
    =\bbE^{\bbP_{\mu,\Omega}}\left[\|v_j^\dag\|\mid Y_j^\dag\right]+\|v_j^\dag\|.
\end{equation}
By Jensen's inequality and the identity of the physical marginals,
\begin{equation}
\bbE^{\bbP_{\mu,\Omega}}\left[\mathrm{D}_\Omega(Z_j,Z_\star)^2\right]
\le 2\bbE^{\bbP_{\mu,\Omega}}\left[\left(\bbE^{\bbP_{\mu,\Omega}}\left[\|v_j^\dag\|\mid Y_j^\dag\right]\right)^2\right]
+2\bbE^{\bbP_{\mu,\Omega}}\|v_j^\dag\|^2 \le 4\bbE^{\bbP_q}\|v_j^\dag\|^2.
\end{equation}
Therefore, \(\sup_{j\ge0}\bbE^{\bbP_{\mu,\Omega}}\left[\mathrm{D}_\Omega(Z_j,Z_\star)^2\right]<\infty.\)
This second-moment bound implies uniform integrability of \(\{\mathrm{D}_\Omega(Z_j,Z_\star)\}_{j\ge0}\), since
\begin{equation}
    \sup_{j\ge0}\bbE^{\bbP_{\mu,\Omega}}\left[\mathrm{D}_\Omega(Z_j,Z_\star)\mathbf{1}_{\{\mathrm{D}_\Omega(Z_j,Z_\star)>R\}}\right]
    \le \frac{1}{R}\sup_{j\ge0}\bbE^{\bbP_{\mu,\Omega}}\left[\mathrm{D}_\Omega(Z_j,Z_\star)^2\right]\to0, \text{ as }R\to\infty.
\end{equation}
\end{proof}

\begin{lemma}[From $\hat{\mathrm{D}}_\Omega$ to $\mathrm{D}_\Omega$]\label{prop:hatD_to_DOmega}
Consider the filter metric in $\hat{\mathrm{D}}_\Omega$ with $d=W_1$. Let $Z_0\sim\mu$ be any initial law and $\mu_j=\Law(Z_j)$, for $j\in\bbZ_+$. Assume that, for some $Z_\star\in\Omega$, the family $\left\{\mathrm{D}_\Omega(Z_j,Z_\star)\right\}_{j\ge0}$ is uniformly integrable. Then for $\mu_\infty\in\cP(\Omega)$,
\begin{equation}
    \lim_{j\to\infty}W_{\hat{\mathrm{D}}_\Omega}(\mu_j,\mu_\infty)=0
    \Longrightarrow \lim_{j\to\infty}W_{\mathrm{D}_\Omega}(\mu_j,\mu_\infty)=0.
\end{equation}
\end{lemma}

\begin{proof}
Let $r(Z)=\mathrm{D}_\Omega(Z,Z_\star)$. Since the filter metric in $\hat{\mathrm{D}}_\Omega$ is $d=W_1$, we have, for all
$Z=(q,v)$ and $Z'=(q',v')$ in $\Omega$,
\begin{equation}
\label{eq:app_truncated_D_controlled_by_hatD}
    \min\{\mathrm{D}_\Omega(Z,Z'),1\}
    \le
    \hat{\mathrm{D}}_\Omega(Z,Z').
\end{equation}
Indeed, if $v=v'$, then
$\min\{\mathrm{D}_\Omega(Z,Z'), 1\}=\min\{W_1(q,q'), 1\}=\hat{\mathrm{D}}_\Omega(Z,Z')$,
while if $v\neq v'$, $\min\{\mathrm{D}_\Omega(Z,Z'),1\}$ is at most $1$ and $\hat{\mathrm{D}}_\Omega(Z,Z')$ is at least $1$. Then
\begin{equation}
\label{eq:app_truncated_wasserstein_convergence}
    0=\lim_{j\to\infty}W_{\hat{\mathrm{D}}_\Omega}(\mu_j,\mu_\infty)\geq \limsup_{j\to\infty}W_{\min\{\mathrm{D}_\Omega,1\}}(\mu_j,\mu_\infty)\geq 0 \Longrightarrow \lim_{j\to\infty}W_{\min\{\mathrm{D}_\Omega,1\}}(\mu_j,\mu_\infty)= 0.
\end{equation}

We first show that $\mu_\infty$ has finite first moment with respect to
$\mathrm{D}_\Omega$. For every $M>0$, the truncated function $\min\{r(\cdot), M\}$ is bounded and Lipschitz with respect to the bounded metric $\min\{\mathrm{D}_\Omega,1\}$.
Therefore, by \eqref{eq:app_truncated_wasserstein_convergence},
\begin{equation}
    \int_\Omega \min\{r(Z), M\}\mu_\infty(\mathrm{d}Z)
    =
    \lim_{j\to\infty}
    \int_\Omega \min\{r(Z), M\}\mu_j(\mathrm{d}Z).
\end{equation}
The uniform integrability assumption on $\left\{\mathrm{D}_\Omega(Z_j,Z_\star)\right\}_{j\ge0}$ implies $\sup_{j\ge0}\int_\Omega r(Z)\mu_j(\mathrm{d}Z)<\infty.$
Hence, letting $M\to\infty$ and applying the monotone convergence theorem gives
\begin{equation}
\label{eq:app_mu_infty_first_moment}
    \int_\Omega r(Z)\mu_\infty(\mathrm{d}Z) = \lim_{M\to\infty}\int_\Omega (r(Z)\wedge M)\mu_\infty(\mathrm{d}Z)
    \le
    \sup_{j\ge0}
    \int_\Omega r(Z)\mu_j(\mathrm{d}Z)<\infty.
\end{equation}

For each $j$, choose a coupling $\gamma_j\in\Pi(\mu_j,\mu_\infty)$ such that
\begin{equation}
\label{eq:app_near_optimal_hatD_coupling}
    \int_{\Omega\times\Omega}
        \hat{\mathrm{D}}_\Omega(Z,Z')
        \gamma_j(\mathrm{d}Z,\mathrm{d}Z')
    \le
    W_{\hat{\mathrm{D}}_\Omega}(\mu_j,\mu_\infty)
    +
    \frac{1}{j+1}.
\end{equation}
By \eqref{eq:app_truncated_D_controlled_by_hatD} and
\eqref{eq:app_near_optimal_hatD_coupling},
\begin{equation}
    \int_{\Omega\times\Omega}
        \min\{\mathrm{D}_\Omega(Z,Z'),1\}
        \gamma_j(\mathrm{d}Z,\mathrm{d}Z')
    \to0.
\end{equation}
Thus $\mathrm{D}_\Omega(Z,Z')\to0$ in $\gamma_j$-probability. It remains to upgrade this convergence in probability to convergence in
expectation. By the triangle inequality, $\mathrm{D}_\Omega(Z,Z')\le r(Z)+r(Z')$ Therefore, for every $R>0$,
\begin{equation}
    \mathrm{D}_\Omega(Z,Z')
    \mathbf{1}_{\{\mathrm{D}_\Omega(Z,Z')>R\}}
    \le
    2r(Z)\mathbf{1}_{\{r(Z)>R/2\}}
    +
    2r(Z')\mathbf{1}_{\{r(Z')>R/2\}}.
\end{equation}
Integrating this inequality with respect to $\gamma_j$ gives
\begin{multline}
    \int_{\Omega\times\Omega}
        \mathrm{D}_\Omega(Z,Z')
        \mathbf{1}_{\{\mathrm{D}_\Omega(Z,Z')>R\}}
        \gamma_j(\mathrm{d}Z,\mathrm{d}Z')
    \le
    2\int_\Omega
        r(Z)\mathbf{1}_{\{r(Z)>R/2\}}
        \mu_j(\mathrm{d}Z) 
    \\+ 
    2\int_\Omega
        r(Z)\mathbf{1}_{\{r(Z)>R/2\}}
        \mu_\infty(\mathrm{d}Z).
\end{multline}
The first term on the right-hand side vanishes uniformly in $j$ as
$R\to\infty$ by the assumed uniform integrability. The second term vanishes
as $R\to\infty$ by \eqref{eq:app_mu_infty_first_moment}. Hence the family
$\{\mathrm{D}_\Omega(Z,Z')\}_{j\ge0}$ under the couplings $\gamma_j$ is
uniformly integrable. Since $\mathrm{D}_\Omega(Z,Z')\to0$ in $\gamma_j$-probability and the family
is uniformly integrable, we obtain
\begin{equation}
    \int_{\Omega\times\Omega}
        \mathrm{D}_\Omega(Z,Z')
        \gamma_j(\mathrm{d}Z,\mathrm{d}Z')
    \to0.
\end{equation}
Finally, since $\gamma_j$ is a coupling of $\mu_j$ and $\mu_\infty$,
\begin{equation}
    W_{\mathrm{D}_\Omega}(\mu_j,\mu_\infty)
    \le
    \int_{\Omega\times\Omega}
        \mathrm{D}_\Omega(Z,Z')
        \gamma_j(\mathrm{d}Z,\mathrm{d}Z')
    \to0.
\end{equation}
This proves the implication.
\end{proof}

By synthesizing the bounded metric convergence with the uniform integrability established above, we can now state the main ergodicity result for the exact augmented process, which serves as the rigorous theoretical foundation forAssumption~\ref{da:ergodicity_on_state_filter}(ii).

\begin{theorem}[Ergodicity of the exact augmented process in unbounded metric]\label{thm:exact_ergodicity_unbounded}
Suppose that the physical state $\{v_j^\dag\}_{j \in \bbZ_+}$ admits a unique invariant measure $q_\infty$ with total-variation convergence from every initial state law, as in Theorem~\ref{thm:ergodic_X}, and satisfies the Lyapunov estimate \eqref{eq:app_lyapunov} with $V(v)=\|v\|^2$. Furthermore, suppose $\Psi$ and $h$ are continuous and the exact filter recursion satisfies all three parts of Assumption~\ref{ass:stability} with $d=W_1$.

 Let $\mu \in \cP(\Omega)$ be an initial distribution for the augmented state $Z_0 = (\pi_0, v_0^\dag)$ such that the physical state marginal has a finite second moment (i.e., $\bbE\|v_0^\dag\|^2 < \infty$). Assume that the filter is initialized consistently with the state--observation law under $\bbP_{\mu,\Omega}$, so that $\pi_j=\bbP_{\mu,\Omega}(v_j^\dag\in\cdot\mid Y_j^\dag)$ for every $j\in\bbZ_+$. In particular, with $Y_0^\dag=\varnothing$, the initial filter is $\pi_0=\Law(v_0^\dag)$ almost surely. Then the augmented process $Z_j = (\pi_j, v_j^\dag)$ admits a unique invariant probability measure $\mu_\infty \in \cP(\Omega)$, which has a finite first moment with respect to $\mathrm{D}_\Omega$. Furthermore, the marginal laws $\mu_j = \Law(Z_j)$ converge to $\mu_\infty$ in the unbounded Wasserstein metric:
\begin{equation}
    \lim_{j \to \infty} W_{\mathrm{D}_\Omega}(\mu_j, \mu_\infty) = 0.
\end{equation}
\end{theorem}

\begin{proof}
By Theorem~\ref{thm:ergodic_X} and Corollary~\ref{cor:marginal_convergence_hatD}, there exists a unique invariant measure $\mu_\infty \in \cP(\Omega)$, and the marginal laws satisfy $\lim_{j\to\infty} W_{\hat{\mathrm{D}}_\Omega}(\mu_j, \mu_\infty) = 0$. 

For the exact augmented process, the consistent initialization assumed here allows us to apply Lemma~\ref{lem:augmented_moment_control}, which guarantees that the Lyapunov condition on the physical state translates to the uniform integrability of the sequence $\{\mathrm{D}_\Omega(Z_j, Z_\star)\}_{j\ge0}$. Finally, applying Lemma~\ref{prop:hatD_to_DOmega}, this uniform integrability rigorously upgrades the convergence in the bounded metric $\hat{\mathrm{D}}_\Omega$ to convergence in the unbounded metric $\mathrm{D}_\Omega$, completing the proof.
\end{proof}

\begin{remark}[Verification for specific dynamical models]\label{rmk:models_verification}
Theorem~\ref{thm:exact_ergodicity_unbounded} verifies Assumption~\ref{da:ergodicity_on_state_filter}(ii) for the models considered in this paper under the stated non-degeneracy and filter-stability assumptions. Assuming all three parts of Assumption~\ref{ass:stability} for the exact filter with $d=W_1$, non-degenerate system noise, and consistent initialization:
\begin{itemize}
    \item \textbf{Doubling-angle model (Subsection~\ref{ssec:doublin_angle_model}:} The physical state space is the 1D torus, which is intrinsically compact. Consequently, the exact filter components $\pi_j$ are probability measures supported on this compact space. Both $\|v_j^\dag\|$ and $W_1(\pi_j, \delta_0)$ are uniformly bounded, meaning the uniform integrability requirement is trivially satisfied without any moment assumptions. Therefore, Theorem~\ref{thm:exact_ergodicity_unbounded} applies immediately.
    
    \item \textbf{Lorenz~'63 and Lorenz~'96 models (Subsections~\ref{ssec:L63} and \ref{ssec:L96}:} While their state spaces are unbounded Euclidean spaces, the underlying continuous-time ordinary differential equations for both models are strictly dissipative; see, for example, \cite{mattingly2002ergodicity,law_accuracy_lorenz96_2015}. This inherent macroscopic dissipativity naturally translates to the discrete-time Lyapunov estimate \eqref{eq:app_lyapunov}. Thus, as long as the initial state has a finite second moment, uniform second-moment control holds, and Theorem~\ref{thm:exact_ergodicity_unbounded} applies under the stated filter assumptions. In particular, $W_1$-stability of the exact filter is assumed here, not established by the dissipativity argument.
\end{itemize}
\end{remark}

\subsection{Flatness of the Rank Histogram of the Filtering Distribution}\label{appendix:rank_hist_uniform}

This subsection provides the theoretical justification for the rank-histogram diagnostic and the $\operatorname{RHVar}$ metric defined in Appendix~\ref{app_subsec:rhvar_metric}: Proposition~\ref{prop:flatness_rhvar} shows that an ensemble sampled from the true filtering distribution has a uniform rank statistic, so a flat rank histogram is the ideal calibration behavior.

\begin{proposition}\label{prop:flatness_rhvar}
    The rank statistic of $f(v^\dag_j)$ with respect to an ensemble $\{f(v^{(n)}_j)\}_{n=1}^N$, where $\{v_j^{(n)}\}_{n=1}^N$ are drawn independently from the true filtering distribution $\mathbb{P}(v_j^\dag\mid Y_j^\dag)$ (Definition~\ref{def:filtering_distribution}) and $f$ is a measurable scalar function, is uniform over the integers $[0, N]$. Here, we assume $(v_j^{(1)},\ldots,v_j^{(N)})\perp v_j^\dag\mid Y_j^\dag$ and break ties in rank uniformly at random.
\end{proposition}
\begin{proof}
    The proof is a direct application of Theorem~1 in \citet{talts_validating_2020}, which states that for any joint distribution $\mathbb{P}(\theta,z)$, letting $(\tilde\theta,\tilde z)\sim\mathbb{P}(\theta,z)$, and $\{\theta^{(1)},\ldots,\theta^{(N)}\}\stackrel{\mathrm{i.i.d.}}{\sim}\mathbb{P}(\theta\mid\tilde z)$ with $(\theta^{(1)},\ldots,\theta^{(N)})\perp\tilde\theta\mid\tilde z$, implies that the rank statistic of $f(\tilde\theta)$ with respect to $\{f(\theta^{(1)}),\ldots,f(\theta^{(N)})\}$ is uniformly distributed over integers $[0,N]$, where $f:\Theta\to\mathbb{R}$ is any measurable one-dimensional function and $\Theta$ is the sample space of $\theta$.

    Picking $\theta:=v_j$ and $z:=Y_j$, where $v_j$ and $Y_j$ are generated according to the dynamics--data model~\eqref{eq:DA_setting}, we have $(\tilde\theta,\tilde z)\sim\mathbb{P}(v_j^\dag,Y_j^\dag)$. Then, we have
    \begin{equation}
        \{\theta^{(1)},\ldots,\theta^{(N)}\}=\{v_j^{(1)},\ldots,v_j^{(N)}\}\stackrel{\mathrm{i.i.d.}}{\sim}\mathbb{P}(v_j^\dag\mid Y_j^\dag),
\end{equation}
    i.e., samples drawn from the true filtering distribution. The sampling assumptions make $v_j^\dag,v_j^{(1)},\ldots,v_j^{(N)}$ conditionally exchangeable given $Y_j^\dag$, so uniform random tie-breaking also gives uniform ranks when ties occur. The stated result follows.
\end{proof}

\section{Bootstrap Particle Filter Reference and Evaluation}
\label{app:bpf_reference}

In the low-dimensional experiments for the doubling-angle model in Subsection~\ref{ssec:doublin_angle_model} and the Lorenz~'63 model in Subsection~\ref{ssec:L63}, we use a bootstrap particle filter (BPF) with $10^6$ particles as an accurate numerical proxy for the true filtering distribution. This appendix describes how this reference filter is constructed and how its outputs are stored for later comparison. We first introduce BPF in Appendix~\ref{app_subsec:bpf}. We then describe the recorded BPF quantities in Appendix~\ref{app_subsec:bpf_recording} for the comparison metric calculations. Next, we report diagnostic evidence in Appendix~\ref{app_subsec:bpf_validity} for the two low-dimensional models showing that the BPF is effective. Finally, we define the evaluation metrics used against the BPF reference: SED for both low-dimensional models in Appendix~\ref{app_subsec:sed_metric}, and FED only for Lorenz~'63 in Appendix~\ref{app_subsec:fed_metric}.

\subsection{Bootstrap Particle Filter}
\label{app_subsec:bpf}

We consider the dynamics--observation model in \eqref{eq:DA_setting}. At assimilation cycle $j$, the BPF represents the filtering distribution $\pi_{j-1}$ by an equally weighted particle system of $M$ particles $\pi_{j-1}^M = \frac{1}{M}\sum_{m=1}^M \delta_{v_{j-1}^{(m)}}.$
The forecast particles are obtained by independently propagating each analysis particle through the stochastic dynamics,
\begin{equation}
    \hat v_j^{(m)} = \Psi\bigl(v_{j-1}^{(m)}\bigr) + \xi_{j-1}^{(m)},
    \qquad
    \xi_{j-1}^{(m)} \sim \normal(0,\Sigma),
    \qquad
    m \in [M].
\end{equation}
The empirical predictive distribution is therefore $\hat\pi_j^M = \frac{1}{M}\sum_{m=1}^M \delta_{\hat v_j^{(m)}}.$ With the observation $y^\dagger_j$, since the observation noise is Gaussian, the normalized BPF weights $w_j^{(m)}$ are
\begin{equation}
    \tilde w_j^{(m)}
    =
    \exp\left(
    -\frac{1}{2}
    \left\|y_j^\dagger - h\bigl(\hat v_j^{(m)}\bigr)\right\|_{\Gamma}^2
    \right),
    \quad
    w_j^{(m)}
    =
    \frac{\tilde w_j^{(m)}}{\sum_{r=1}^M \tilde w_j^{(r)}},
    \quad
    m=1,2,\ldots,M.
    \label{eq:bpf_normalized_weight}
\end{equation}
The weighted particle approximation of the filtering distribution is $\bar\pi_j^M = \sum_{m=1}^M w_j^{(m)}\delta_{\hat v_j^{(m)}}.$
To obtain an equally weighted analysis ensemble, we resample $M$ times from the forecast particles according to the normalized weights. Specifically, conditionally on the forecast particles and the observation history, the resampling indices satisfy
\begin{equation}
    \mathbb{P}\left(I_j^{(m)} = r\right) = w_j^{(r)},
    \qquad
    r \in [M],
    \qquad
    m \in [M].
\end{equation}
The resampled analysis particles and the resulting empirical filtering distribution are
\begin{equation}
    v_j^{(m)} = \hat v_j^{\left(I_j^{(m)}\right)},
    \quad
    m \in [M],
    \quad
    \pi_j^M = \frac{1}{M}\sum_{m=1}^M \delta_{v_j^{(m)}}.
\end{equation}

\subsection{Recorded BPF Quantities}
\label{app_subsec:bpf_recording}

At every assimilation cycle of every test trajectory, we record two classes of quantities from the BPF. The first class consists of distributional features used to describe the BPF forecast or analysis ensemble, including means, covariances, principal-component information, and projected quantile functions and, for Lorenz~'63, the BPF reference subsamples used by FED. The second class consists of pre-resampling weight diagnostics used to assess whether the BPF remains reliable, including the effective sample size, and the weight abundance.

For a state particle ensemble at time $j$, written as $\{v_j^{(m)}\}_{m=1}^M$, the empirical mean and covariance are
\begin{align}
    \bar v_j^M
    =
    \frac{1}{M}\sum_{m=1}^M v_j^{(m)},
    \qquad
    C_j^M
    =
    \frac{1}{M-1}
    \sum_{m=1}^M
    \bigl(v_j^{(m)}-\bar v_j^M\bigr)\otimes
    \bigl(v_j^{(m)}-\bar v_j^M\bigr).
\end{align}
For principal component analysis (PCA), we compute the eigendecomposition
\begin{equation}
    C_j^M = U_j^M\Lambda_j^M\left(U_j^M\right)^\top,
    \qquad
    U_j^M = \left[u_{j,1}^M,\ldots,u_{j,d_v}^M\right],
\end{equation}
where the eigenvalues in $\Lambda_j^M$ are sorted in descending order. The centered PCA score of particle $m$ along the $\ell$-th principal direction is
\begin{equation}
    p_{j,\ell}^{(m)}
    =
    \left(u_{j,\ell}^M\right)^\top
    \bigl(v_j^{(m)}-\bar v_j^M\bigr).
\end{equation}

The projected quantile functions provide the inputs for SED. For any one-dimensional projection direction $e \in \bbR^{d_v}$, we define the scalar projected particle by
\begin{equation}
    x_{j,e}^{(m)} = e^\top v_j^{(m)},
    \qquad
    m \in [M].
\end{equation}
Let $i_1,\ldots,i_M$ be a permutation of $[M]$ such that the projected samples are sorted in nondecreasing order,
\begin{equation}
    x_{j,e}^{(i_1)}
    \leq
    x_{j,e}^{(i_2)}
    \leq
    \cdots
    \leq
    x_{j,e}^{(i_M)}.
\end{equation}
For $\tau \in [0,1]$, the empirical projected quantile function is computed by linear interpolation along the sorted projected samples,
\begin{subequations}\label{eq:bpf_projected_quantile}
\begin{align}
    &Q_{j,e}(\tau)
    =
    x_{j,e}^{(i_{a_\tau})}
    +
    \bigl(r_\tau-a_\tau\bigr)
    \left(
    x_{j,e}^{(i_{b_\tau})}
    -
    x_{j,e}^{(i_{a_\tau})}
    \right),\\
    \text{where }
    &r_\tau = 1+(M-1)\tau,
    \quad
    a_\tau = \lfloor r_\tau \rfloor,
    \quad
    b_\tau = \lceil r_\tau \rceil.
\end{align}
\end{subequations}
When the time index is clear, we write $Q_e(\tau)$ for $Q_{j,e}(\tau)$. We record this quantile function for every coordinate direction and every PCA direction. More precisely, with
\begin{equation}
    \mathcal{E}^{\mathrm{coord}}
    =
    \{e_1,\ldots,e_{d_v}\},
    \quad
    \mathcal{E}^{\mathrm{pc}}_j
    =
    \{u_{j,1}^M,\ldots,u_{j,d_v}^M\},
    \quad
    \mathcal{E}_j
    =
    \mathcal{E}^{\mathrm{coord}}
    \cup
    \mathcal{E}^{\mathrm{pc}}_j,
\end{equation}
we store $Q_{j,e}(\tau_k)$ for all $e \in \mathcal{E}_j$ and all quantile levels $\tau_k = \frac{k-1}{K-1}$ for $k=1,\ldots,K.$
In our implementation, $K=257$. The PCA directions used in evaluation are always computed from the BPF reference, so all ensemble-based methods are compared against the same projected BPF targets.

For every Lorenz~'63 test trajectory and assimilation cycle, we also independently draw $R_{\mathrm{sub}}=4$ reference subsamples from the equally weighted $M=10^6$ BPF analysis particles, using subsampling seeds $\{0,1,2,42\}$. Each subsample contains $M_{\mathrm{sub}}=1000$ particles and is denoted by $\mathcal{V}_{j,r}^{\mathrm{BPF}}=\{\tilde v_{j,r}^{(m)}\}_{m=1}^{M_{\mathrm{sub}}}$ for $r\in[R_{\mathrm{sub}}]$. The four subsamples are stored and reused for every comparison method at the corresponding trajectory and assimilation cycle; the trajectory index is omitted here for readability.

The weight diagnostics are computed before resampling, using the normalized weights $\{w_j^{(m)}\}_{m=1}^M$. We record two complementary diagnostics:
\begin{subequations}
\begin{align}
    \text{Effective sample size: } \mathrm{ESS}_j
    &=
    \left(
    \sum_{m=1}^M \left(w_j^{(m)}\right)^2
    \right)^{-1},\label{eq:ESS}
    \\
    \text{Weight abundance: } W_{A,j}
    &=
    \exp\left(
    -\sum_{m=1}^M w_j^{(m)}\log w_j^{(m)}
    \right).\label{eq:WA}
\end{align}
\end{subequations}
The effective sample size is a standard diagnostic for weight degeneracy in sequential Monte Carlo methods \citep{liu1998sequential,del2006sequential}. The weight abundance is the exponential of the Shannon entropy of the normalized weights, and can be interpreted as the effective number of active particles on the entropy scale \citep{jost2006entropy}.

\begin{table}[htpb]
\centering
\caption{BPF quantities recorded at each assimilation cycle. The first group records distributional features of the state particles and, for Lorenz~'63, the reference subsamples used by FED. The second group records pre-resampling diagnostics for validating the BPF reference.}
\label{tab:bpf_recorded_quantities}
\renewcommand{\arraystretch}{1.6}
\resizebox{\textwidth}{!}{%
\begin{tabular}{lllp{0.46\textwidth}}
\toprule
\textbf{Category} & \textbf{Metric} & \textbf{Notation} & \textbf{Definition} \\
\midrule
Recorded features
& Mean
& $\bar v_j^M$
& $\frac{1}{M}\sum_{m=1}^M v_j^{(m)}$
\\
& Covariance
& $C_j^M$
& $\frac{1}{M-1}\sum_{m=1}^M \bigl(v_j^{(m)}-\bar v_j^M\bigr)\otimes\bigl(v_j^{(m)}-\bar v_j^M\bigr)$
\\
& Principal components
& $U_j^M,\Lambda_j^M$
& $C_j^M=U_j^M\Lambda_j^M\left(U_j^M\right)^\top$
\\
& Projected quantile
& $Q_{j,e}(\tau)$
& Refer to \eqref{eq:bpf_projected_quantile}
\\
& FED reference subsamples (Lorenz~'63)
& $\{\mathcal{V}_{j,r}^{\mathrm{BPF}}\}_{r=1}^{R_{\mathrm{sub}}}$
& $R_{\mathrm{sub}}=4$ subsamples, each containing $M_{\mathrm{sub}}=1000$ BPF analysis particles
\\
\midrule
BPF validity diagnostics
& Effective sample size
& $\mathrm{ESS}_j$
& $\left(\sum_{m=1}^M \left(w_j^{(m)}\right)^2\right)^{-1}$
\\
& Weight abundance
& $W_{A,j}$
& $\exp\left(-\sum_{m=1}^M w_j^{(m)}\log w_j^{(m)}\right)$
\\
\bottomrule
\end{tabular}%
}
\end{table}

The two diagnostics capture related but not identical aspects of the pre-resampling weight distribution. The effective sample size $\mathrm{ESS}_j$ is based on the second moment of the normalized weights. It is highly sensitive to very large weights and therefore provides a conservative indicator of particle degeneracy. In the ideal case where all particles have equal weights, $\mathrm{ESS}_j=M$; if one particle carries all the weight, then $\mathrm{ESS}_j=1$. Thus, a small value of $\mathrm{ESS}_j/M$ indicates that only a small fraction of particles effectively contributes to the likelihood-weighted approximation before resampling.

The weight abundance $W_{A,j}$ measures the effective number of active particles on the entropy scale. Compared with $\mathrm{ESS}_j$, it is less dominated by the largest few weights and instead summarizes how broadly the likelihood mass is distributed across particles. It also satisfies $1\leq W_{A,j}\leq M$, with $W_{A,j}=M$ for uniform weights and $W_{A,j}=1$ for a completely collapsed weight vector. We therefore use $\mathrm{ESS}_j$ and $W_{A,j}$ together: the former detects severe concentration of the weights, while the latter gives an entropy-based measure of the number of particles that carry non-negligible mass. In the low-dimensional experiments, consistently large values of both diagnostics provide evidence that the BPF with $M=10^6$ particles remains a reliable numerical proxy for the filtering distribution.

\subsection{BPF Validity for the Low-Dimensional Models}
\label{app_subsec:bpf_validity}

For the two low-dimensional examples, we use the BPF reference only after checking two conditions: the pre-resampling weights should not collapse, and the resulting BPF posterior statistics should be stable with respect to the Monte Carlo seed. The first condition is assessed using the effective sample size $\mathrm{ESS}_j$ (Equation~\eqref{eq:ESS}) and the weight abundance $W_{A,j}$ (Equation~\eqref{eq:WA}). The second condition is assessed by repeating the same BPF experiment with independent random seeds and measuring the across-seed variability of the posterior mean and covariance.

For a fixed model, observation function, particle number $M$, and test trajectory, let $\bar v_j^{M,s}\in\bbR^{d_v}$ and $C_j^{M,s}\in\bbR^{d_v\times d_v}$ denote the BPF posterior mean and covariance at assimilation cycle $j$ from run $s=1,\ldots,R$. In our diagnostics, we use $R=12$ independent seeds, namely $\{0,1,\ldots,10,42\}$. For any seed-indexed scalar quantity $z^s$, write $\langle z\rangle_R=R^{-1}\sum_{s=1}^R z^s$. The across-seed standard error of the $\ell$-th component of the BPF posterior mean is
\begin{equation}
    \widehat{\mathrm{se}}_{\bar v,j,\ell}^{M}
    =
    \left[
    \frac{1}{R(R-1)}
    \sum_{s=1}^R
    \left(
    \bar v_{j,\ell}^{M,s}
    -
    \langle \bar v_{j,\ell}^{M}\rangle_R
    \right)^2
    \right]^{1/2}.
\end{equation}
We define $\mathrm{SE}_{\mathrm{mean}}(M)$ as the average of $\widehat{\mathrm{se}}_{\bar v,j,\ell}^{M}$ over assimilation cycles and state components. The covariance diagnostic is defined entrywise in the same way:
\begin{equation}
    \widehat{\mathrm{se}}_{C,j,\ell q}^{M}
    =
    \left[
    \frac{1}{R(R-1)}
    \sum_{s=1}^R
    \left(
    C_{j,\ell q}^{M,s}
    -
    \langle C_{j,\ell q}^{M}\rangle_R
    \right)^2
    \right]^{1/2}.
\end{equation}
We define $\mathrm{SE}_{\mathrm{cov}}(M)$ as the average of $\widehat{\mathrm{se}}_{C,j,\ell q}^{M}$ over assimilation cycles and covariance entries. In the experiments, all reported averages are also taken over $B=64$ independent test trajectories; the trajectory index is omitted from the notation above for readability. We use $J=200$ assimilation cycles for the doubling-angle model and $J=500$ assimilation cycles for the Lorenz~'63 model. We perform this diagnostic for four BPF reference settings: the doubling-angle model with cosine observation in Subsection~\ref{ssec:doublin_angle_model}, and the Lorenz~'63 model in Subsection~\ref{ssec:L63} with identity, square, and arctan observations. For each setting, we repeat the analysis for a range of particle numbers up to $M=10^6$.

Figure~\ref{fig:bpf_validity_low_dim} summarizes the results. The first row reports the time- and trajectory-averaged $\mathrm{ESS}_j$ and $W_{A,j}$ as functions of $M$. In all four settings, both diagnostics scale approximately linearly with $M$, indicating that the likelihood weights retain a large active particle population before resampling. This provides evidence against severe weight collapse in the BPF reference runs. The second row reports $\mathrm{SE}_{\mathrm{mean}}(M)$ and $\mathrm{SE}_{\mathrm{cov}}(M)$. The mean standard error decreases as $M$ increases. The covariance standard error is not strictly monotone over the tested particle numbers, but its magnitude remains small in all four settings. Together, the weight diagnostics and seed-stability diagnostics indicate that the BPF reference is numerically stable for these low-dimensional models. We therefore use the $M=10^6$ BPF output as the numerical reference for evaluating ensemble-based methods in the main experiments.

\begin{figure}[htbp]
    \centering
    \includegraphics[width=\textwidth]{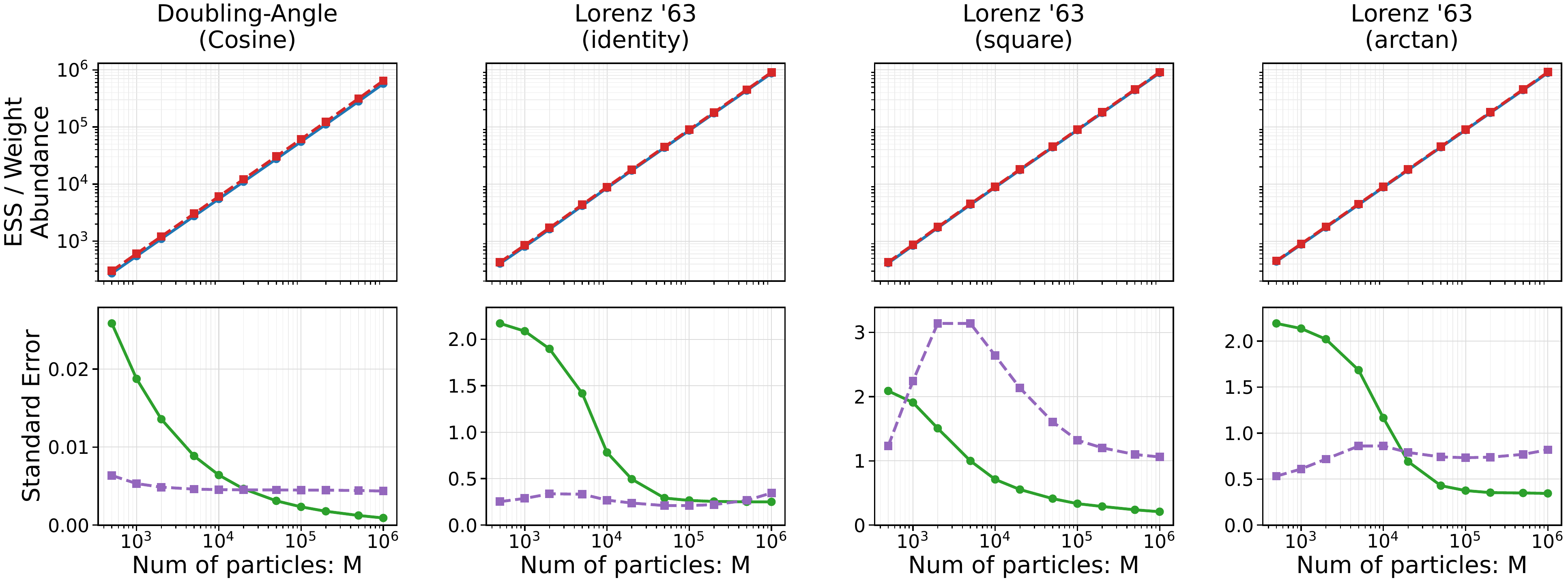}
    \includegraphics[width=0.65\textwidth]{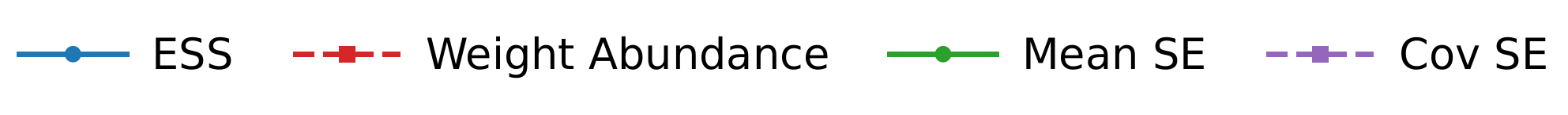}
    \caption{BPF validity diagnostics for the low-dimensional reference experiments. Columns correspond to the doubling-angle model with cosine observation and the Lorenz~'63 model with identity, square, and arctan observations. The first row shows the averaged pre-resampling effective sample size and weight abundance as functions of the particle number $M$. The second row shows the seed standard errors of the BPF posterior mean and covariance, averaged over assimilation cycles, test trajectories, and state or covariance entries.}
    \label{fig:bpf_validity_low_dim}
\end{figure}

\subsection{Sliced Energy Distance}
\label{app_subsec:sed_metric}

We now define the metric used to compare an ensemble-based method with the BPF reference. Storing all $10^6$ BPF particles for every assimilation cycle of every test trajectory would be costly. For SED, we store their quantiles along the prescribed coordinate and BPF-PCA directions $e\in\mathcal{E}_j$ at each cycle and use these one-dimensional targets in the comparison.

For an ensemble-based method with analysis ensemble $\{v_j^{(n)}\}_{n=1}^N$, define the projected ensemble samples by $x_{j,e}^{(n)}=e^\top v_j^{(n)}$. The corresponding BPF target quantiles are $q_{j,e,k}=Q_{j,e}(\tau_k)$, where $\tau_k=(k-1)/(K-1)$ for $k=1,\ldots,K$. To approximate integration over the quantile level, we use the trapezoidal weights
\begin{equation}
\begin{aligned}
    \alpha_1
    &=
    \frac{1}{2}(\tau_2-\tau_1),
    &
    \alpha_K
    &=
    \frac{1}{2}(\tau_K-\tau_{K-1}),
    \\
    \alpha_k
    &=
    \frac{1}{2}(\tau_{k+1}-\tau_{k-1}),
    &
    \bar\alpha_k
    &=
    \frac{\alpha_k}{\sum_{\ell=1}^K \alpha_\ell},
\end{aligned}
\label{eq:sed_trapezoidal_weights}
\end{equation}
where the definition of $\alpha_k$ in the second line applies for $2\leq k\leq K-1$. The projected ensemble law is $F_{j,e}^N=N^{-1}\sum_{n=1}^N\delta_{x_{j,e}^{(n)}}$, and the quantile-based BPF approximation is $G_{j,e}^Q=\sum_{k=1}^K\bar\alpha_k\delta_{q_{j,e,k}}$.

Let $d(a,b)=|a-b|$ for Lorenz~'63 and $d(a,b)=d_{\mathrm{chord}}(a,b)$ from \eqref{eq:torus-chordal-distance} for doubling-angle. For a fixed projection, define the empirical energy-distance contribution to SED between the projected ensemble and the projected quantile-based BPF approximation by
\begin{multline}
    \operatorname{SED}_{j,e}
    =
    \Biggl[2\Biggl(
    \frac{1}{N}\sum_{n=1}^N\sum_{k=1}^K \bar\alpha_k d\!\left(x_{j,e}^{(n)},q_{j,e,k}\right)
    -
    \frac{1}{2N^2}\sum_{n=1}^N\sum_{n'=1}^N d\!\left(x_{j,e}^{(n)},x_{j,e}^{(n')}\right)
    \\-
    \frac{1}{2}\sum_{k=1}^K\sum_{\ell=1}^K \bar\alpha_k\bar\alpha_\ell d\!\left(q_{j,e,k},q_{j,e,\ell}\right)
    \Biggr)\Biggr]^{1/2}.
    \label{eq:projected_energy_distance}
\end{multline}
The three terms inside the parentheses are respectively the cross-distance between the method and the BPF target, the self-distance of the method, and the self-distance of the BPF target. For doubling-angle, all three terms use distances between the unit-circle embeddings of the scalar states and quantiles. The coordinate and BPF-PCA directions give the same contribution because $d_{\mathrm{chord}}(-a,-b)=d_{\mathrm{chord}}(a,b)$. The time-averaged SED is then
\begin{equation}
    \operatorname{SED}
    =
    \frac{1}{J}
    \sum_{j=1}^J
    \frac{1}{|\mathcal{E}_j|}
    \sum_{e\in\mathcal{E}_j}
    \operatorname{SED}_{j,e}.
    \label{eq:sed_score}
\end{equation}
A smaller $\operatorname{SED}$ indicates closer agreement with the BPF reference across the chosen coordinate and BPF-PCA projections. Across test trajectories, the mean uses available outputs, whereas the standard deviation is reported only when all $64$ outputs are available; the trajectory index is omitted above for readability.

For Lorenz~'63, $\mathcal{E}_j$ contains only finitely many prescribed directions, so SED is a projection-based diagnostic and need not distinguish all multivariate distributions. The FED comparison below therefore complements SED with an ordinary energy distance computed directly in the full Lorenz~'63 state space.

\subsection{Full-State Energy Distance}
\label{app_subsec:fed_metric}

We also compare distributions in the joint state space using the full-state energy distance (FED). For an ensemble-based method with analysis ensemble $\{v_j^{(n)}\}_{n=1}^N$, let $D_{j,r}$ denote the energy distance between its empirical distribution and the $r$-th stored BPF reference subsample at assimilation cycle $j$:
\begin{multline}
    D_{j,r}
    =
    \Biggl[
    \frac{2}{N M_{\mathrm{sub}}}
    \sum_{n=1}^N\sum_{m=1}^{M_{\mathrm{sub}}}
    d_\star\!\left(v_j^{(n)},\tilde v_{j,r}^{(m)}\right)
    \\
    {}-
    \frac{1}{N^2}
    \sum_{n=1}^N\sum_{n'=1}^N
    d_\star\!\left(v_j^{(n)},v_j^{(n')}\right)
    -
    \frac{1}{M_{\mathrm{sub}}^2}
    \sum_{m=1}^{M_{\mathrm{sub}}}\sum_{m'=1}^{M_{\mathrm{sub}}}
    d_\star\!\left(\tilde v_{j,r}^{(m)},\tilde v_{j,r}^{(m')}\right)
    \Biggr]^{1/2}.
    \label{eq:fed_subsample_energy_distance}
\end{multline}
For Lorenz~'63, $d_\star$ is the Euclidean distance on the three-dimensional state space. The FED for one test trajectory is obtained by first computing the energy distance for each stored reference subsample and then averaging over the four subsamples and assimilation cycles:
\begin{equation}
    \operatorname{FED}
    =
    \frac{1}{J}
    \sum_{j=1}^J
    \frac{1}{R_{\mathrm{sub}}}
    \sum_{r=1}^{R_{\mathrm{sub}}}
    D_{j,r}.
    \label{eq:fed_score}
\end{equation}
The same four stored subsamples are used for all methods and ensemble sizes. When multiple test trajectories are considered, \eqref{eq:fed_score} is computed separately for each trajectory and then averaged over trajectories. Thus, FED approximates the energy distance to the BPF reference in the joint state space by averaging over four finite reference subsamples; smaller values indicate closer empirical agreement with the BPF reference.

For the Lorenz~'63 BPF-subsampling diagnostic reported in Appendix~\ref{app:lorenz63_additional_results}, we restore the test-trajectory index $b\in[B]$, with $B=64$. Let $D_{b,j,r}$ denote the trajectory-indexed version of \eqref{eq:fed_subsample_energy_distance}. We first aggregate the distance associated with each reference-subsample replicate $r$ over trajectories and assimilation cycles, and then compute the standard error across the $R_{\mathrm{sub}}=4$ replicate aggregates:
\begin{equation}
\begin{aligned}
\operatorname{FED}_r
&=\frac{1}{BJ}\sum_{b=1}^{B}\sum_{j=1}^{J}D_{b,j,r},
\qquad
\operatorname{FED}=\frac{1}{R_{\mathrm{sub}}}\sum_{r=1}^{R_{\mathrm{sub}}}\operatorname{FED}_r,
\\
\widehat{\operatorname{SE}}_{\mathrm{sub}}(\operatorname{FED})
&=\left[
\frac{1}{R_{\mathrm{sub}}(R_{\mathrm{sub}}-1)}
\sum_{r=1}^{R_{\mathrm{sub}}}
\left(\operatorname{FED}_r-\operatorname{FED}\right)^2
\right]^{1/2}.
\end{aligned}
\label{eq:fed_subsampling_se}
\end{equation}
This standard error measures only the Monte Carlo variability introduced by replacing each fixed $10^6$-particle BPF analysis ensemble with a $1000$-particle reference subsample. It is distinct from variability across test trajectories, independent BPF runs, or independently trained neural-network models.

\section{Truth-Based Evaluation Metrics}
\label{app:truth_based_evaluation_metrics}

This section defines the two truth-based diagnostics used in the nonlinear experiments. Rank-histogram variance (RHVar), used for Lorenz~'63 and Lorenz~'96, assesses marginal ensemble calibration through projected rank histograms. The variogram score (VS), used only for Lorenz~'96, assesses whether an analysis ensemble reproduces cross-component difference structure in the full state space.

\subsection{Rank-Histogram Variance}
\label{app_subsec:rhvar_metric}

Rank histograms, also known as Talagrand diagrams, compare the rank of the verifying truth relative to the sorted ensemble members and are useful for diagnosing reliability, bias, and spread errors \citep{candille2005evaluation,hamill2001interpretation}. For a vector-valued state, we compute a separate rank histogram for each state coordinate. Thus $L=d_v$, and $u_\ell$ is the $\ell$th coordinate unit vector for $\ell\in[L]$. For each scalar verification case $i$ and direction $u_\ell$, define
\begin{equation}
z_{i,\ell}^\dagger = \langle u_\ell, v_i^\dagger\rangle,
\qquad
z_{i,\ell}^{(n)} = \langle u_\ell, v_i^{(n)}\rangle,
\qquad n=1,\ldots,N.
\label{eq:project_directions}
\end{equation}
Here $i$ indexes the scalar verification cases obtained by pooling the relevant test trajectories and assimilation cycles. The rank of the projected truth among the projected ensemble members is $R_{i,\ell}=\sum_{n=1}^{N}\mathbf{1}_{z_{i,\ell}^{(n)}<z_{i,\ell}^\dagger}$. We use random tie-breaking when ties occur, so $R_{i,\ell}\in\{0,1,\ldots,N\}$. If the ensemble is statistically consistent with the truth, then the truth is exchangeable with the ensemble members and, under the uniform-rank null hypothesis, $R_{i,\ell}\sim\operatorname{Unif}\{0,1,\ldots,N\}$. Appendix~\ref{appendix:rank_hist_uniform} provides the corresponding theoretical justification for ensembles sampled from the true filtering distribution. A U-shaped histogram indicates underdispersion, a dome-shaped histogram indicates overdispersion, and a skewed histogram indicates bias.

To summarize rank-histogram flatness, let $\mathcal{I}$ be the collection of scalar verification cases and let $N_{\mathrm{case}}=|\mathcal{I}|$. The normalized frequency of rank bin $r$ along direction $\ell$ is $p_{r,\ell}=N_{\mathrm{case}}^{-1}\sum_{i\in\mathcal{I}}\mathbf{1}\{R_{i,\ell}=r\}$ for $r=0,1,\ldots,N$.

\noindent The unscaled variance is \mbox{$\operatorname{RHVar}_{\mathrm{raw}}(\ell)=\frac{1}{N+1}\sum_{r=0}^{N}\left(p_{r,\ell}-\frac{1}{N+1}\right)^2$}. Under the uniform-rank null hypothesis and assuming independent rank samples, \mbox{$\bbE[\operatorname{RHVar}_{\mathrm{raw}}(\ell)]=\frac{N}{N_{\mathrm{case}}(N+1)^2}$}.

Because the raw variance depends on both $N$ and $N_{\mathrm{case}}$, we define the scaled rank-histogram variance by
\begin{equation}
\operatorname{RHVar}(\ell)
=\frac{N_{\mathrm{case}}(N+1)^2}{N}\operatorname{RHVar}_{\mathrm{raw}}(\ell)
=\frac{N_{\mathrm{case}}(N+1)}{N}\sum_{r=0}^{N}\left(p_{r,\ell}-\frac{1}{N+1}\right)^2.
\label{eq:rank_hist_variance}
\end{equation}
We report $\operatorname{RHVar}=L^{-1}\sum_{\ell=1}^{L}\operatorname{RHVar}(\ell)$, the average over the state-coordinate directions. This normalization gives $\bbE[\operatorname{RHVar}(\ell)]=1$ under the independent uniform-rank null hypothesis. Thus values near one indicate fluctuations consistent with a calibrated ensemble, while values substantially larger than one indicate rank nonuniformity and poorer marginal calibration.

RHVar is used only as a qualitative and exploratory diagnostic. The preceding expectation assumes independent ranks, whereas ranks pooled over filtering trajectories and assimilation cycles can be serially dependent. If they were independent, $N\operatorname{RHVar}(\ell)$ would be Pearson's goodness-of-fit statistic, whose asymptotic null distribution is $\chi^2$ with $N$ degrees of freedom as the number of independent verification cases grows with $N$ fixed. Asymptotically exact rank-histogram tests under serial dependence have been developed for predictive distributions of future observations \citep{brocker2020stratified}, but applying them here would require an indirect evaluation through predictive observation distributions rather than the filtering distributions studied in this work.

\subsection{Variogram Score}
\label{app_subsec:variogram_score}

For Lorenz~'96, we additionally use the variogram score (VS) \citep{scheuerer2015variogram} to compare the realized truth with the analysis ensemble in the state space, rather than in the observation space. We use variogram order one and uniform normalized weights over all unordered pairs of the $d_v$ state components. For one test trajectory, the score at assimilation cycle $j$ is
\begin{equation}
\operatorname{VS}_j=\binom{d_v}{2}^{-1}\sum_{1\leq k<\ell\leq d_v}\left[\left|v_{j,k}^\dagger-v_{j,\ell}^\dagger\right|-\frac{1}{N}\sum_{n=1}^{N}\left|v_{j,k}^{(n)}-v_{j,\ell}^{(n)}\right|\right]^2.
\label{eq:variogram_score}
\end{equation}
Here $v_{j,k}^\dagger$ and $v_{j,k}^{(n)}$ are the $k$-th components of the truth and the $n$-th analysis ensemble member, respectively. The weight $\binom{d_v}{2}^{-1}=2/[d_v(d_v-1)]$ averages the score over all unordered state-component pairs. The full $d_v=40$ Lorenz~'96 state is used, including both observed and unobserved components.

For each test trajectory, we first average over its $J$ assimilation cycles,
\begin{equation}
\operatorname{VS}=\frac{1}{J}\sum_{j=1}^{J}\operatorname{VS}_j.
\label{eq:variogram_score_time_average}
\end{equation}
We then report the mean of this trajectory-level score over available outputs and its standard deviation only when all $64$ trajectory outputs are available; the trajectory index is omitted above for consistency with the other evaluation metrics. A smaller VS indicates that the ensemble more accurately represents the truth's cross-component difference structure. To reduce peak memory use in the high-dimensional calculation, the implementation processes the state-component pairs in chunks of $4096$; this computational batching does not change the score. VS is reported only for Lorenz~'96.

\section{Implementation Details}
\label{app:implementation_details}
This appendix provides additional implementation details for the numerical experiments in Section~\ref{sec:numerical_exp}. Appendix~\ref{subsec:appendix_architecture_training} summarizes the neural-network architectures and training hyperparameters used by the learning-based filters. Appendix~\ref{app_subsec:grid_search} describes the hyperparameter-selection protocols used to configure the benchmark filters. Appendix~\ref{app_subsec:snr_calibration} explains the observation-noise calibration procedure used to compare different observation maps under matched signal-to-noise ratios. Further implementation details, experiment scripts, and code are available in the project repository\footnote{\url{https://github.com/wispcarey/Proper-Scoring-Ensemble-Filter}}.

\subsection{Neural Network Architecture and Training Hyperparameters}
\label{subsec:appendix_architecture_training}

This subsection summarizes the neural-network configurations and training hyperparameters used for the learning-based filters in the numerical experiments. The two architecture families are those introduced in Subsection~\ref{ssec:methods}: the EnKF correction-terms architecture \citep{bach_learning_2026}, denoted by CorrTerms, and the end-to-end filtering architecture, denoted by EtE. For each model, the same neural architecture is used for the ES-trained and NL2-trained variants; only the training objective is changed. In the linear--Gaussian diagnostic experiment in Subsection~\ref{ssec:linear_gaussian_model}, the reported objective comparison uses the EtE architecture with the ES, L2, and NL2 losses.

Table~\ref{tab:nn_architecture_settings} shows the neural network architecture details. The latent dimension controls the internal set representation, and the output feature dimension gives the size of the set-transformer feature representation passed to the subsequent filtering map. The seed column gives the number of learned seed vectors used by the set transformer. The attention-block column reports the total number of attention blocks, consisting of the self-attention blocks, and one bottleneck cross-attention block.

\begin{table}[htbp]
\centering
\caption{Neural-network architecture settings for the learning-based filters.}
\label{tab:nn_architecture_settings}
\renewcommand{\arraystretch}{1.15}
\small
\resizebox{\textwidth}{!}{%
\begin{tabular}{llrrrrrrr}
    \toprule
    \textbf{Model}
    & \textbf{Architecture}
    & \begin{tabular}[c]{@{}c@{}}\textbf{Total}\\\textbf{params}\end{tabular}
    & \begin{tabular}[c]{@{}c@{}}\textbf{Set-transformer}\\\textbf{params}\end{tabular}
    & \begin{tabular}[c]{@{}c@{}}\textbf{Latent}\\\textbf{dim}\end{tabular}
    & \begin{tabular}[c]{@{}c@{}}\textbf{Output feature}\\\textbf{dim}\end{tabular}
    & \begin{tabular}[c]{@{}c@{}}\textbf{Seed}\\\textbf{vectors}\end{tabular}
    & \begin{tabular}[c]{@{}c@{}}\textbf{Attention}\\\textbf{blocks}\end{tabular}
    & \begin{tabular}[c]{@{}c@{}}\textbf{Attention}\\\textbf{heads}\end{tabular} \\
    \midrule
    \multirow{2}{*}{Linear--Gaussian}
    & CorrTerms & 108845 & \multirow{2}{*}{47152}  & \multirow{2}{*}{16} & \multirow{2}{*}{128} & \multirow{2}{*}{16} & \multirow{2}{*}{7} & \multirow{2}{*}{8} \\
    & EtE       & 75908  &                         &                     &                      &                     &                    &                    \\
    \midrule
    \multirow{2}{*}{Doubling-angle}
    & CorrTerms & 172260 & \multirow{2}{*}{117856} & \multirow{2}{*}{32} & \multirow{2}{*}{128} & \multirow{2}{*}{16} & \multirow{2}{*}{7} & \multirow{2}{*}{8} \\
    & EtE       & 143009 &                         &                     &                      &                     &                    &                    \\
    \midrule
    \multirow{2}{*}{Lorenz~'63}
    & CorrTerms & 389001 & \multirow{2}{*}{334016} & \multirow{2}{*}{64} & \multirow{2}{*}{128} & \multirow{2}{*}{16} & \multirow{2}{*}{7} & \multirow{2}{*}{8} \\
    & EtE       & 359427 &                         &                     &                      &                     &                    &                    \\
    \midrule
    \multirow{2}{*}{Lorenz~'96}
    & CorrTerms & 404463 & \multirow{2}{*}{336960} & \multirow{2}{*}{64} & \multirow{2}{*}{128} & \multirow{2}{*}{16} & \multirow{2}{*}{7} & \multirow{2}{*}{8} \\
    & EtE       & 368296 &                         &                     &                      &                     &                    &                    \\
    \bottomrule
\end{tabular}%
}
\end{table}

The training and evaluation datasets are generated independently for each model setting described in Subsections~\ref{ssec:linear_gaussian_model}--\ref{ssec:L96}. Table~\ref{tab:training_hyperparameters} reports the number of simulated trajectories, trajectory lengths, learning rates, batch sizes, training clamp values, and weight-decay values. All configurations use zero weight decay except the CorrTerms architecture for the Lorenz~'96 experiment, where the weight decay is set to $0.01$.

The default random seed is $42$ unless otherwise stated; the training-seed study in Appendix~\ref{ssec:training_seed_sensitivity} uses $\{0,1,2,42\}$, with evaluation seed fixed at $42$. For the doubling-angle, Lorenz~'63, and Lorenz~'96 models, initial training runs for a fixed $1000$ epochs using truncated backpropagation through time (TBPTT) with a truncation interval of five assimilation cycles. We use the final training checkpoint, with no validation set, validation-based checkpoint selection, or early stopping. Test scores do not determine the stopping time or retained checkpoint. Benchmark-filter hyperparameters are selected separately using test metrics, as described in Appendix~\ref{app_subsec:grid_search}.

Initial training for these three models uses AdamW with learning rate $3\times10^{-4}$ for doubling-angle and $10^{-3}$ for both Lorenz models, and the weight decay specified in Table~\ref{tab:training_hyperparameters}. A piecewise-constant \texttt{LambdaLR} schedule halves the learning rate at epochs $200$, $300$, and $400$, with no learning-rate warm-up.

The public GitHub repository provides instructions and scripts for data generation, training, and evaluation, and will include Google Drive download links for the large datasets and trained checkpoints. After downloading them, users can run the repository's evaluation scripts directly without regenerating the data or retraining the models. The hardware and per-analysis-step timings are reported in Subsection~\ref{ssec:runtime} and Figure~\ref{fig:L96_assim_time_vs_N}.

\begin{table}[htbp]
\centering
\caption{Training and evaluation hyperparameters for the learning-based filters.}
\label{tab:training_hyperparameters}
\renewcommand{\arraystretch}{1.15}
\small
\resizebox{\textwidth}{!}{%
\begin{tabular}{llrrrrrrrr}
    \toprule
    \textbf{Model}
    & \textbf{Architecture}
    & \begin{tabular}[c]{@{}c@{}}\textbf{Training}\\\textbf{trajectories}\end{tabular}
    & \begin{tabular}[c]{@{}c@{}}\textbf{Training}\\\textbf{length}\end{tabular}
    & \begin{tabular}[c]{@{}c@{}}\textbf{Test}\\\textbf{trajectories}\end{tabular}
    & \begin{tabular}[c]{@{}c@{}}\textbf{Test}\\\textbf{length}\end{tabular}
    & \begin{tabular}[c]{@{}c@{}}\textbf{Learning}\\\textbf{rate}\end{tabular}
    & \begin{tabular}[c]{@{}c@{}}\textbf{Batch}\\\textbf{size}\end{tabular}
    & \begin{tabular}[c]{@{}c@{}}\textbf{Training}\\\textbf{clamp}\end{tabular}
    & \begin{tabular}[c]{@{}c@{}}\textbf{Weight}\\\textbf{decay}\end{tabular} \\
    \midrule
    \multirow{2}{*}{Linear--Gaussian}
    & CorrTerms & \multirow{2}{*}{8192} & \multirow{2}{*}{60} & \multirow{2}{*}{64} & \multirow{2}{*}{100}  & \multirow{2}{*}{$1\times 10^{-3}$} & \multirow{2}{*}{1024} & \multirow{2}{*}{--} & \multirow{2}{*}{0} \\
    & EtE       &                       &                     &                    &                       &                                  &                       &                     &                   \\
    \midrule
    \multirow{2}{*}{Doubling-angle}
    & CorrTerms & \multirow{2}{*}{8192} & \multirow{2}{*}{60} & \multirow{2}{*}{64} & \multirow{2}{*}{200}  & \multirow{2}{*}{$3\times 10^{-4}$} & \multirow{2}{*}{1024} & \multirow{2}{*}{2} & \multirow{2}{*}{0} \\
    & EtE       &                       &                     &                    &                       &                                  &                       &                   &                   \\
    \midrule
    \multirow{2}{*}{Lorenz~'63}
    & CorrTerms & \multirow{2}{*}{8192} & \multirow{2}{*}{60} & \multirow{2}{*}{64} & \multirow{2}{*}{500}  & \multirow{2}{*}{$1\times 10^{-3}$} & \multirow{2}{*}{512} & \multirow{2}{*}{60} & \multirow{2}{*}{0} \\
    & EtE       &                       &                     &                    &                       &                                  &                      &                    &                   \\
    \midrule
    \multirow{2}{*}{Lorenz~'96}
    & CorrTerms & \multirow{2}{*}{8192} & \multirow{2}{*}{60} & \multirow{2}{*}{64} & \multirow{2}{*}{1500} & \multirow{2}{*}{$1\times 10^{-3}$} & \multirow{2}{*}{512} & \multirow{2}{*}{20} & $0.01$ \\
    & EtE       &                       &                     &                    &                       &                                  &                      &                    & 0 \\
    \bottomrule
\end{tabular}%
}
\end{table}

The different test lengths are chosen according to model complexity and reference-filter cost. For the linear--Gaussian model, the dynamics and filtering distribution are relatively simple, and the exact Kalman filter is available as a reference; using a much longer test trajectory provides limited additional information, so we use test length $100$. For the doubling-angle and Lorenz~'63 models, the reference filtering distribution is approximated by a particle filter with $10^6$ particles. Since this reference computation becomes expensive for long trajectories, we use test lengths $200$ and $500$, respectively. For the Lorenz~'96 model, the state dimension is $d_v=40$, making a particle filter reference computationally infeasible. We therefore use a longer test trajectory of length $1500$ to obtain more stable test averages for the evaluation metrics.

The training clamp is used only during training. After each analysis step, if the absolute value of a state component exceeds the listed threshold, that component is clipped to the threshold with the same sign. This clipping is introduced to improve stability during the early stage of training and is not used during testing. A dash in the training-clamp column indicates that no clamp is used.

The ensemble size used during training follows the setting specified in the corresponding numerical experiment. The doubling-angle model in Subsection~\ref{ssec:doublin_angle_model} is trained at ensemble size $N=30$ and evaluated at other ensemble sizes without fine-tuning. The Lorenz~'63 experiments in Subsection~\ref{ssec:L63} are trained at $N=10$ and evaluated directly across the tested ensemble sizes. For the Lorenz~'96 experiment in Subsection~\ref{ssec:L96}, the initial training is performed at ensemble size $N=10$. When evaluating at a different ensemble size $N'\neq N$, we freeze the set-transformer parameters during fine-tuning, corresponding to approximately $83.3\%$ of the CorrTerms parameters and $91.5\%$ of the EtE parameters. Fine-tuning runs for $50$ epochs at one tenth of the original initial learning rate, namely $10^{-4}$.

\subsection{Hyperparameter Selection for Benchmark Filters}
\label{app_subsec:grid_search}

Hyperparameters are selected independently for each benchmark filter, model, observation map, and ensemble size. Selection uses the same test trajectories as the reported results rather than a held-out validation set, so the resulting comparisons should be interpreted as favorably tuned benchmarks. In the doubling-angle experiment, the large benchmark ensembles use neither inflation nor localization, and no search is performed. SMF uses a ridge-free, order-$2$ configuration. The Gaussian-RBF width scale is fixed at $\gamma=2.0$, as in the experiments of \citet{spantini2022coupling}.

For the Kalman-type filters, we use exhaustive searches over inflation and, where applicable, Gaspari--Cohn localization radius. The selection objectives are $W_2$ distance to the exact filter for the linear--Gaussian problem, SED to the BPF reference for Lorenz~'63, and relative energy score (RES) for Lorenz~'96. Figures~\ref{fig:lorenz63_grid_search} and \ref{fig:lorenz96_grid_search} show representative searches.

\begin{figure}[htb]
    \centering
    \includegraphics[width=0.94\textwidth]{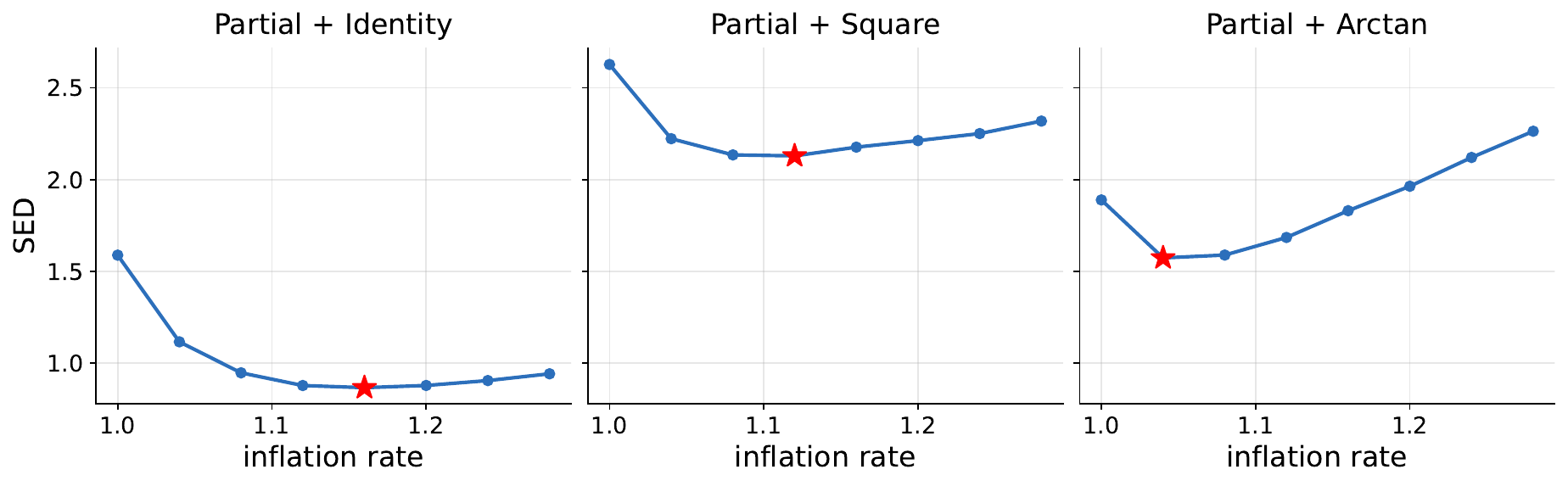}
    \caption{Representative EnKF grid search at $N=10$ for Lorenz~'63. The panels correspond to the partial identity, square, and arctan observations. Only inflation is varied; the vertical axis shows test-set SED, and each star marks the selected value.}
    \label{fig:lorenz63_grid_search}

\end{figure}

\begin{figure}[htb]
    \centering
    \includegraphics[width=0.94\textwidth]{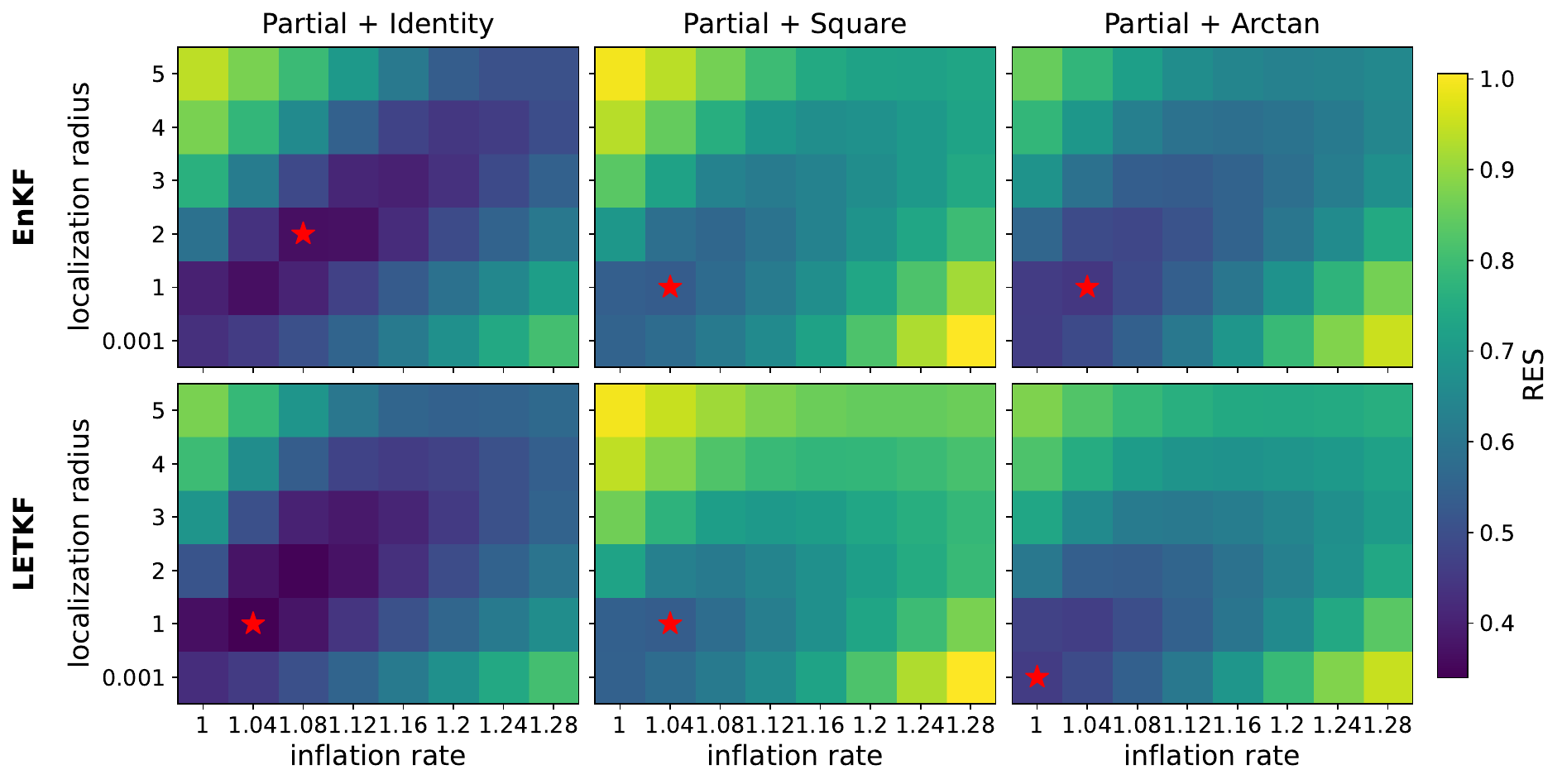}
    \caption{Representative EnKF and LETKF grid searches at $N=10$ for Lorenz~'96. Rows identify the filters and columns the observation maps. Inflation and Gaspari--Cohn localization radius are varied jointly; color shows RES, and each star marks the selected pair.}
    \label{fig:lorenz96_grid_search}
\end{figure}

For SMF in the Lorenz experiments, the Gaussian-RBF width scale $\gamma=2.0$ and map order $2$ remain fixed. The multivariate TPE sampler in Optuna \citep{akiba2019optuna} searches the continuous map-fitting inflation factor $\zeta$ and categorical ridge coefficient $\lambda$ for Lorenz~'63; Lorenz~'96 additionally searches the categorical hard-localization radius $r$. All SMF studies minimize RES, computed by normalizing each trajectory's time-averaged ES by its time-averaged truth norm and averaging over valid trajectories.

\begin{figure}[htb]
    \centering
    \includegraphics[width=\textwidth]{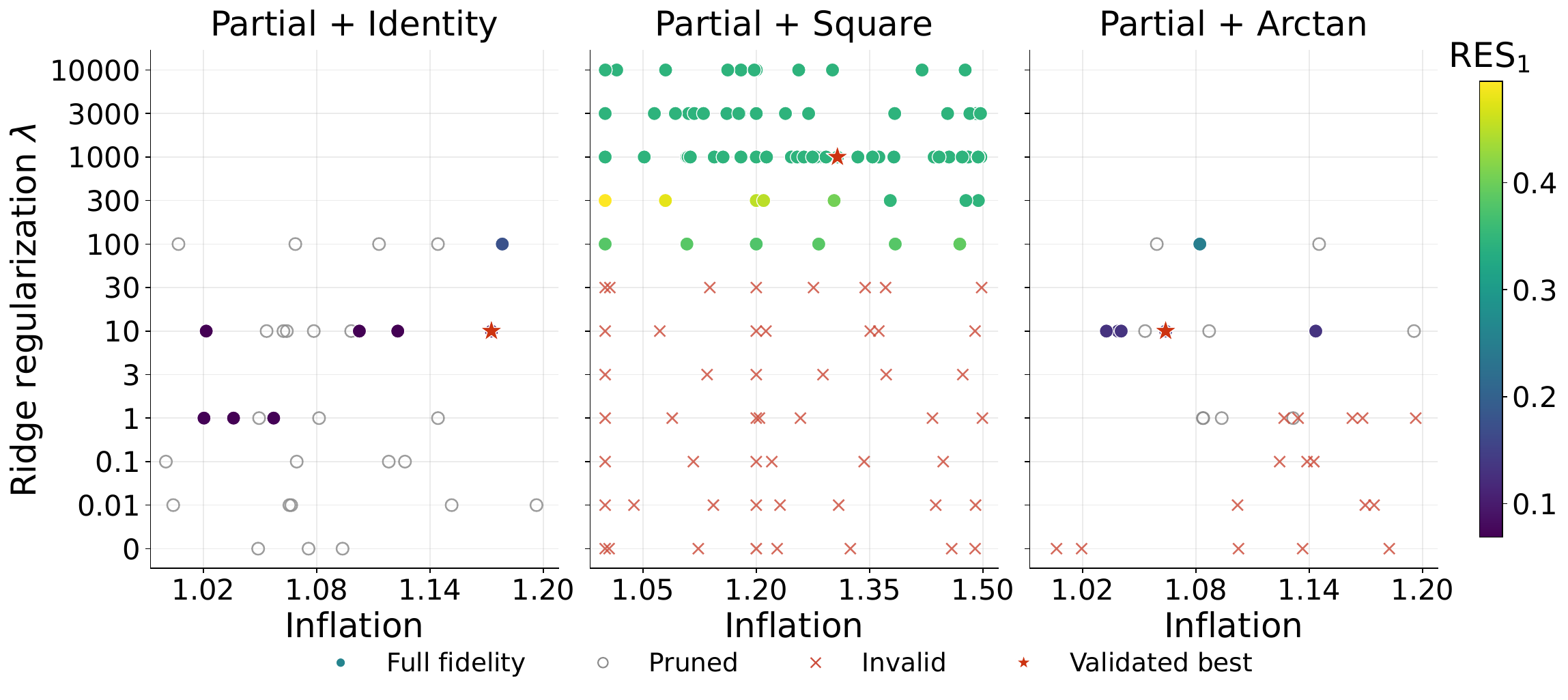}
    \caption{Representative Lorenz~'63 SMF Optuna searches at $N=100$ under the partial identity, square, and arctan observations. Inflation $\zeta$ and ridge coefficient $\lambda$ locate each trial; their irregular placement reflects adaptive TPE proposals rather than a rectangular grid. Color gives RES, while the symbols distinguish full-fidelity, pruned, invalid, and validated-best trials.}
    \label{fig:lorenz63_smf_optuna_search}
\end{figure}

The studies use $32$ trials per configuration, except for Lorenz~'63 square observations, which use $128$. Multi-fidelity evaluation and successive-halving pruning are used except for the full-length Lorenz~'63 square studies; Lorenz~'96 arctan at $N=100$ also disables performance pruning. During hyperparameter selection, every evaluation stage and each candidate-validation seed must retain at least $90\%$ valid trajectories. The best full-length candidates are re-evaluated with four filter seeds, $10042$--$10045$, on the same truth trajectories, and selected by their mean RES.

\subsection{Observation-Noise Calibration by Signal-to-Noise Ratio}
\label{app_subsec:snr_calibration}

This subsection describes how we calibrate the observation-noise standard deviation $\sigma_y$ when comparing different observation maps. The signal-to-noise ratio (SNR) defined here is used only to set the observation-noise level in the experiments. It is not used as an evaluation metric.

The observation maps $h(\cdot)$ are constructed in two steps. We first choose the observed state dimensions, and then apply a scalar observation operation componentwise on the selected state values. For a scalar observed state value $v^o$, we have the default observation maps
\begin{subequations}
\begin{align}
    &\text{Default observation of doubling-angle:} &&h(v^o) = \cos(2\pi v^o).\label{eq:snr_doubling_default_obs}
    \\
    &\text{Default observation of Lorenz '63 \& '96:} &&h(v^o) =v^o.\label{eq:snr_lorenz_default_obs}
\end{align}
\end{subequations}
Since the doubling-angle model is one-dimensional, this is a full observation. For Lorenz~'63 and Lorenz~'96, the default observation is the identity operation on each observed dimension. The nonlinear observation operations considered for Lorenz~'63 and Lorenz~'96 are
\begin{subequations}
\begin{align}
    &\text{Square Observation: }&&h(v^o) = (v^o)^2, \label{eq:snr_square_obs}\\
    &\text{Arctan Observation: }&&h(v^o) = \arctan(v^o). \label{eq:snr_arctan_obs}
\end{align}
\end{subequations}
For vector-valued observations, \eqref{eq:snr_lorenz_default_obs}--\eqref{eq:snr_arctan_obs} are applied componentwise to the observed dimension.

For a fixed observation map $h$, consider the observation model
\begin{equation}
\label{eq:snr_observation_model}
    y_{j+1}^\dagger = h(v_{j+1}^\dagger)+\eta_{j+1}^\dagger, \quad \eta_{j+1}^\dagger \sim \normal(0,\sigma_y^2 I_{d_y}), \quad j\in\bbZ_+,
\end{equation}
where $v_j^\dagger$ is the truth state and $d_y$ is the observation dimension. Since different observation maps can substantially change the empirical scale of the clean observation $h(v_j^\dagger)$, the same nominal value of $\sigma_y$ may correspond to very different effective observation-noise levels. This is particularly important when comparing identity, square, and arctan observations.

Given a representative truth trajectory of length $J$, we compute the time average of the clean observations and the empirical signal strength as
\begin{align}
    S_h = \frac{1}{J}\sum_{j=1}^{J} \left\|h(v_j^\dagger)-\bar h\right\|_2^2,\qquad \bar h = \frac{1}{J}\sum_{j=1}^{J} h(v_j^\dagger) \label{eq:snr_signal_variance}
\end{align}
Because the observation noise covariance is $\sigma_y^2 I_{d_y}$, the total observation-noise variance is $S_\eta = d_y\sigma_y^2.$
We define the empirical variance-ratio SNR by
\begin{equation}
\label{eq:snr_variance_ratio}
    \operatorname{SNR}(\sigma_y) = \frac{S_h}{S_\eta} = \frac{1}{d_y\sigma_y^2}\frac{1}{J}\sum_{j=1}^{J} \left\|h(v_j^\dagger)-\bar h\right\|_2^2.
\end{equation}
This is a variance ratio rather than a standard-deviation ratio. Therefore, the SNR is inversely proportional to $\sigma_y^2$. In practice, we start from a default noise level $\hat{\sigma}_y$ and measure the corresponding default SNR for each observation map $\hat R = \operatorname{SNR}(\hat{\sigma}_y) = \frac{S_h}{d_y\hat{\sigma}_y^2}.$
If the target SNR is $R$, then the noise level that achieves this target is
\begin{equation}
\label{eq:snr_sigma_from_default}
    \sigma_{y,\mathrm{target}} = \hat{\sigma}_y\left(\hat{R}/R\right)^{1/2}.
\end{equation}

For Lorenz~'63 and Lorenz~'96, we use the default-observation SNR as the target SNR for the square and arctan observation experiments. The calibrated value $\sigma_y(h)$ is therefore chosen separately for each dataset and observation map. Table~\ref{tab:snr_default_settings} reports the SNR calibration settings. The left table gives the empirical variance-ratio SNR values at the baseline observation-noise level $\sigma_y$. The right table gives the calibrated values of $\sigma_y(h)$ used for the square and arctan observation experiments on Lorenz~'63 and Lorenz~'96. All numerical values are rounded to two decimal places.

\begin{table}[htbp]
\centering
\caption{SNR calibration settings. Left: empirical variance-ratio SNR values at the baseline observation-noise level $\sigma_y$. Right: calibrated observation-noise standard deviations used to match the default-observation SNR. Empty entries indicate observation maps that are not used for that dataset.}
\label{tab:snr_default_settings}
\renewcommand{\arraystretch}{1.2}
\resizebox{\textwidth}{!}{%
\begin{tabular}[t]{lcccc}
    \multicolumn{5}{c}{\textbf{Baseline SNR}} \\
    \toprule
    \textbf{Dataset} & $\boldsymbol{\sigma_y}$ & \textbf{Default} & \textbf{Square} & \textbf{Arctan} \\
    \midrule
    Doubling-angle & $0.20$ & $12.43$ & -- & -- \\
    Lorenz~'63 & $2.00$ & $15.14$ & $1090.63$ & $0.41$ \\
    Lorenz~'96 & $1.00$ & $13.28$ & $594.67$ & $1.01$ \\
    \bottomrule
\end{tabular}
\qquad
\begin{tabular}[t]{lcc}
    \multicolumn{3}{c}{\textbf{Calibrated $\boldsymbol{\sigma_y}(h)$}} \\
    \toprule
    \textbf{Dataset} & \textbf{Square} & \textbf{Arctan} \\
    \midrule
    Lorenz~'63 & $16.98$ & $0.33$ \\
    Lorenz~'96 & $6.69$ & $0.28$ \\
    \bottomrule
\end{tabular}%
}
\end{table}

For all experiments, filters compared under the same dataset and observation map use the same calibrated value of $\sigma_y(h)$. Therefore, the SNR calibration controls the observation-noise setting, while filtering accuracy and ensemble spread are evaluated separately using the metrics described in Subsection~\ref{ssec:evaluation_metrics}.

\section{Additional Experimental Results}
\label{app:additional_experimental_results}

This appendix presents additional experimental results that complement the numerical experiments in Section~\ref{sec:numerical_exp}. We first compare finite-ensemble energy-score normalizations using test-set performance at successive training checkpoints in the doubling-angle and Lorenz~'63 models. We then report the Lorenz~'63 FED results and visualize non-Gaussian and multimodal filtering distributions on representative test trajectories. Finally, we report trajectory variability, initial-ensemble perturbations, and training-seed sensitivity.

\subsection{Finite-Ensemble Energy-Score Normalization}
\label{app_subsec:es_normalization_comparison}

\begin{figure}[!htbp]
    \centering
    \includegraphics[width=\textwidth]{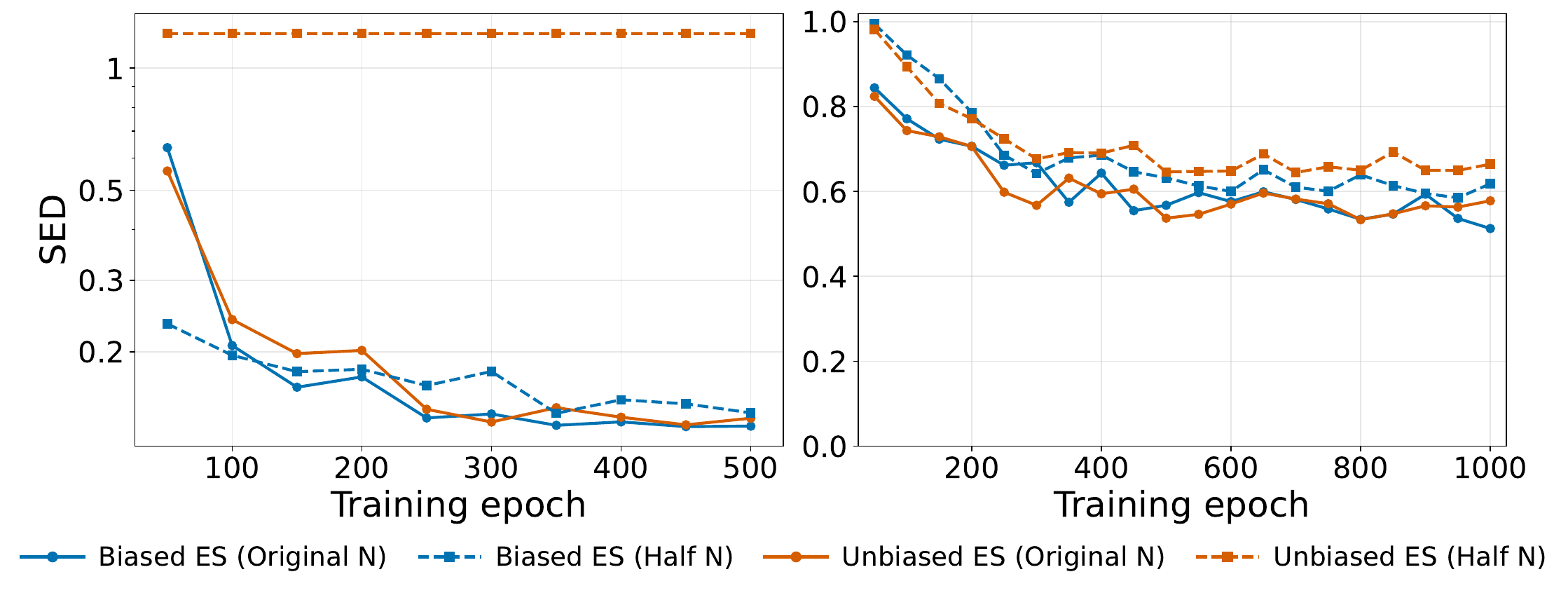}
    \caption{Effect of the finite-ensemble ES normalization on test-set performance across training checkpoints. Each point is the test-set SED of the indicated checkpoint against a $10^6$-particle BPF reference; the plotted quantity is not the training loss. The left panel shows the doubling-angle model, with original $N=30$ and half $N=15$; the right panel shows Lorenz~'63 with the square observation map, with original $N=10$ and half $N=5$. The biased ES uses the pairwise-term coefficient $1/(2N^2)$, whereas the unbiased ES uses $1/[2N(N-1)]$. Smaller is better.}
    \label{fig:es_normalization_comparison}
\end{figure}

The empirical energy score in \eqref{eq:empirical_energy_score} uses the pairwise-term coefficient $1/(2N^2)$. This is the energy score of the realized empirical measure, but it is a biased finite-$N$ estimator of the population energy score when the ensemble members are viewed as i.i.d. draws. We refer to this empirical-measure form as the biased ES normalization. We compare it with the unbiased ES normalization, which uses the off-diagonal pairwise term $[2N(N-1)]^{-1}\sum_{n\ne k}\|v^{(n)}-v^{(k)}\|_2$ and is unbiased for the population pairwise expectation under i.i.d. sampling \citep{fricker_three_2013}. The terms ``biased'' and ``unbiased'' refer to this i.i.d. estimator interpretation; the output particles of the learned filter interact through the ensemble architecture.

The ratio between the unbiased and biased pairwise-term coefficients is $N/(N-1)$, so their difference decreases as $N$ grows. Figure~\ref{fig:es_normalization_comparison} therefore compares the two normalizations both at the original training ensemble sizes used in the main experiments and at half those sizes. For the doubling-angle model, the original and half ensemble sizes are $N=30$ and $N=15$, respectively; for Lorenz~'63 with the square observation map, they are $N=10$ and $N=5$. At each reported training epoch, the corresponding checkpoint is evaluated on the test dataset, and the plotted metric is its test-set SED to the BPF reference filtering distribution computed with $10^6$ particles. Thus, the figure reports an evaluation metric across training checkpoints, not the training loss. Apart from the pairwise-term normalization and the training ensemble size $N$, all settings are identical to those of the corresponding experiments in the main text. The half-$N$ cases are included specifically to make finite-ensemble differences between the two normalizations more visible.

\begin{figure}[!htbp]
    \centering
    \textbf{(a) Assimilation cycle $J=150$}\par\vspace{-0.2em}
    \DoublingAngleGrid{figures/doubling1d_step150}{0.07\textwidth}
    \vspace{0.15em}
    \textbf{(b) Assimilation cycle $J=200$}\par\vspace{-0.2em}
    \DoublingAngleGrid{figures/doubling1d_step200}{0.07\textwidth}
    \includegraphics[width=0.60\textwidth, trim=0cm 1.5cm 0cm 1.5cm, clip]{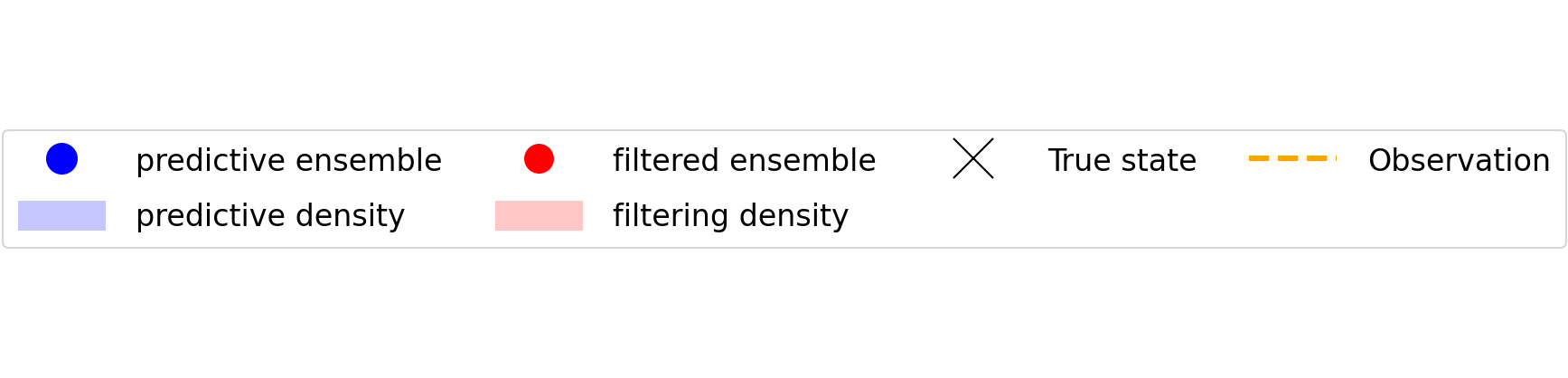}
    \vspace{-0.5cm}
    \caption{Predictive (blue) and filtering (red) densities at (a) $J=150$ and (b) $J=200$ for one doubling-angle test trajectory. The modulo-one state is mapped to the unit circle; the vertical line marks the observed $x$-coordinate. The BPF reference uses $10^6$ particles. Learned methods use $N\in\{10,30,100\}$ (training $N=30$), while the benchmark filters use $N\in\{300,1000,3000\}$; IEnKF at $N=3000$ is omitted because of its computational cost.}
    \label{fig:doubling1d_prior_filtering_steps150_200}
    \label{fig:doubling1d_prior_filtering_step150}
    \label{fig:doubling1d_prior_filtering_step200}
\end{figure}

At the original ensemble sizes, the biased and unbiased normalizations yield comparable test-set SED trajectories across checkpoints in both models. The two Lorenz~'63 runs at half $N$ also attain progressively smaller test-set SED values and remain broadly similar, although their final values are somewhat higher than those at the original $N$. In contrast, the unbiased half-$N$ run for the doubling-angle model shows no improvement in test-set SED: the score remains essentially flat across all evaluated checkpoints and substantially higher than those of the other three doubling-angle runs, which decrease to approximately $0.13$--$0.14$. Thus, the unbiased normalization usually makes little difference in the tested main-text settings, but it can fail to yield improving test-set performance when $N$ is small.

\subsection{Doubling-Angle Model}
\label{app:doubling_angle_additional_results}

Figure~\ref{fig:doubling1d_prior_filtering_steps150_200} supplements the final-cycle illustration in Figure~\ref{fig:doubling1d_filtering} by showing both predictive and filtering densities at $J=150$ and $J=200$ along the same trajectory. The layout allows a direct comparison of each forecast and analysis.

At both cycles, EtE+ES follows the non-Gaussian and multimodal BPF reference most closely. CorrTerms+ES captures part of this structure, whereas the NL2-trained and benchmark filters tend to miss or collapse the separated modes.

\subsection{Lorenz '63}
\label{app:lorenz63_additional_results}

For the Lorenz~'63 experiment in Subsection~\ref{ssec:L63}, we report FED under all three observation maps, followed by filtering-distribution visualizations for the partial identity and partial arctan settings.

\begin{figure}[!htb]
    \centering
    \includegraphics[width=\textwidth]{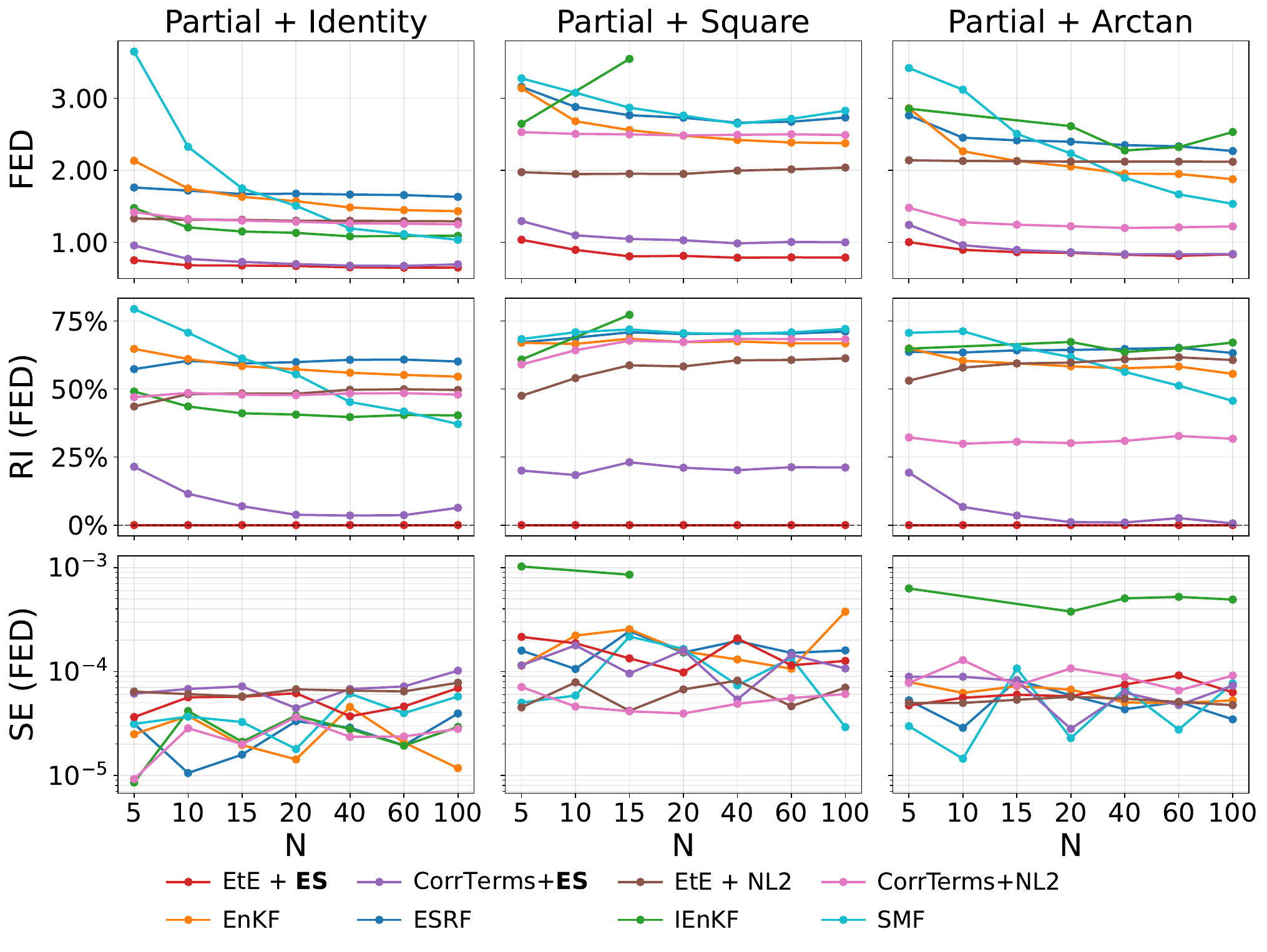}
    \caption{Lorenz~'63 FED results. Columns show the partial identity, square, and arctan observations; rows report FED, RI(FED), and, on a logarithmic scale, the BPF-subsampling standard error in \eqref{eq:fed_subsampling_se}. RI(FED) measures the relative improvement of EtE+ES over each comparison method under the same observation map and at the same ensemble size. The methods and ensemble sizes match Figure~\ref{fig:L63_metric_curves}; missing values indicate filter divergence.}
    \label{fig:L63_fed_curves}
\end{figure}

The absolute SED and FED values are not directly comparable because SED averages one-dimensional distances over prescribed projections, whereas FED uses Euclidean distance in the full three-dimensional state space. Nevertheless, the FED results give closely similar RI patterns and the same substantive conclusion as SED across all three observation maps. EtE+ES has the smallest FED among the tested methods at every reported ensemble size, followed by CorrTerms+ES; their difference becomes small under the partial arctan observation. The NL2-trained and benchmark filters remain farther from the BPF reference, with the largest separation under the partial square observation. Within the benchmark group, SMF improves markedly with $N$ and has the smallest FED at $N=100$ under identity and arctan observations, whereas at $N=100$ its square-observation FED remains larger than those of EnKF and ESRF.

The BPF-subsampling standard errors range from approximately $10^{-5}$ to $10^{-3}$, whereas the FED values are of order one. Although four replicates provide only a coarse standard-error estimate, this three-to-five-order-of-magnitude separation shows that variability due specifically to the choice of the stored $1000$-particle BPF subsamples is negligible on the scale of the FED comparisons. This diagnostic does not measure training-seed or test-trajectory uncertainty and does not make the BPF reference exact. Separately, Figure~\ref{fig:bpf_validity_low_dim} validates the underlying BPF reference across independent BPF seeds and particle counts up to $10^6$.

\begin{figure}[!htb]
    \centering
    \setlength{\tabcolsep}{1.5pt}
    \renewcommand{\arraystretch}{1.05}
    \setlength{\fboxsep}{0pt}
    \setlength{\fboxrule}{0.1pt}

    \resizebox{\textwidth}{!}{%
        \begin{tabular}{>{\centering\arraybackslash}m{0.080\textwidth} *{9}{>{\centering\arraybackslash}m{0.093\textwidth}}}
        \toprule
        & BPF & EtE & {\scriptsize CorrTerms} & EtE & {\scriptsize CorrTerms} & \multirow{2}{*}{ESRF} & \multirow{2}{*}{EnKF} & \multirow{2}{*}{IEnKF} & \multirow{2}{*}{SMF} \\
        & {\scriptsize reference} & + \textbf{ES} & + \textbf{ES} & + NL2 & + NL2 & & & & \\
        \midrule

        \multicolumn{10}{c}{\textbf{Partial identity observation}} \\
        \midrule
        \LorenzGridRows{200}{identity}{0} \\
        \midrule
        \LorenzGridRows{300}{identity}{0} \\

        \bottomrule
        \multicolumn{10}{c}{%
            \includegraphics[width=0.35\textwidth, trim=0.3cm 0.1cm 0.3cm 0.1cm, clip]{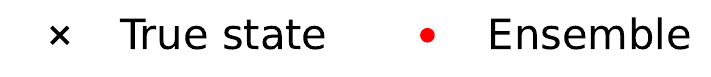}
        }
        \end{tabular}%
    }

    \caption{Lorenz~'63 filtering ensembles at assimilation cycles $200$ and $300$ under the partial identity observation. The two-dimensional projections and evaluation protocol match Figure~\ref{fig:lorenz63_grid}; all learned and benchmark filters use $N=100$.}
    \label{fig:lorenz63_grid_identity_additional}
\end{figure}

Figure~\ref{fig:lorenz63_grid_identity_additional} shows the two-dimensional projected filtering ensembles at assimilation cycles $200$ and $300$ under the partial identity observation setting. Under this observation model, EtE+ES provides the closest visual agreement with the BPF reference. CorrTerms+ES also captures part of the non-Gaussian structure, but its approximation is generally less accurate than that of EtE+ES. The methods trained with the normalized $L^2$ loss, namely EtE+NL2 and CorrTerms+NL2, as well as the benchmark filters, show more substantial deviations from the target filtering distributions. Although the aggregate curves show that SMF improves substantially by $N=100$, these selected snapshots still exhibit visible deviations from the BPF geometry.

\begin{figure}[!htb]
    \centering
    \setlength{\tabcolsep}{1.5pt}
    \renewcommand{\arraystretch}{1.05}
    \setlength{\fboxsep}{0pt}
    \setlength{\fboxrule}{0.1pt}

    \resizebox{\textwidth}{!}{%
        \begin{tabular}{>{\centering\arraybackslash}m{0.080\textwidth} *{9}{>{\centering\arraybackslash}m{0.093\textwidth}}}
        \toprule
        & BPF & EtE & {\scriptsize CorrTerms} & EtE & {\scriptsize CorrTerms} & \multirow{2}{*}{ESRF} & \multirow{2}{*}{EnKF} & \multirow{2}{*}{IEnKF} & \multirow{2}{*}{SMF} \\
        & {\scriptsize reference} & + \textbf{ES} & + \textbf{ES} & + NL2 & + NL2 & & & & \\
        \midrule

        \multicolumn{10}{c}{\textbf{Partial arctan observation}} \\
        \midrule
        \LorenzGridRows{200}{arctan}{0} \\
        \midrule
        \LorenzGridRows{300}{arctan}{0} \\

        \bottomrule
        \multicolumn{10}{c}{%
            \includegraphics[width=0.35\textwidth, trim=0.3cm 0.1cm 0.3cm 0.1cm, clip]{figures/lorenz63_grid/grid_identity_post_adaptive_transposed_legend.pdf}
        }
        \end{tabular}%
    }

    \caption{Lorenz~'63 filtering ensembles at assimilation cycles $200$ and $300$ under the partial arctan observation. The two-dimensional projections and evaluation protocol match Figure~\ref{fig:lorenz63_grid}; all learned and benchmark filters use $N=100$.}
    \label{fig:lorenz63_grid_arctan_additional}
\end{figure}

Figure~\ref{fig:lorenz63_grid_arctan_additional} reports the corresponding results under the partial arctan observation setting. Because this observation model introduces an additional nonlinear transformation, the resulting filtering problem is more challenging. Even in this case, EtE+ES remains the most accurate among the tested methods and reproduces the overall shape of the BPF reference more faithfully than the alternatives. CorrTerms+ES again captures some non-Gaussian structure, whereas the normalized-$L^2$ variants and the benchmark filters exhibit more pronounced discrepancies. SMF is the strongest benchmark filter in the aggregate SED and FED comparisons at $N=100$, but the displayed ensembles remain visibly coarser and less faithful than the BPF reference.

\subsection{Variability across Test Trajectories}
\label{app:trajectory_variability}

For each method $e$ and ensemble size $N$, let $\mathcal{E}_b(e,N)$ be the metric for test trajectory $b\in[B]$, where $B=64$. The score mean at each $N$ omits trajectories on which the filter diverges. We report the sample standard deviation only when all $64$ trajectory outputs are available. We average the available standard deviations over ensemble sizes and refer to the result as the trajectory standard deviation (traj-std):
\begin{equation}
    \operatorname{traj\text{-}std}_{\mathcal{E}}(e)
    =
    \frac{1}{|\mathcal{N}_e|}
    \sum_{N\in\mathcal{N}_e}
    \operatorname{SD}_{b\in[B]}\!\left[\mathcal{E}_b(e,N)\right],
\label{eq:traj_std}
\end{equation}
where $\mathcal{N}_e$ contains only ensemble sizes with all $64$ scores. For the learned methods in the doubling-angle experiment, $\mathcal{N}_e=\{10,30,100,300\}$, while each fixed-size benchmark contributes only the $N$ stated in the figure legend. For Lorenz~'63 and Lorenz~'96, $\mathcal{N}_e$ is a subset of $\{5,10,15,20,40,60,100\}$. IEnKF produces all $64$ scores only at $N=100$ under the Lorenz~'63 identity observation, and at no tested size under the other Lorenz~'63 observations or in Lorenz~'96. Thus, a score mean can be shown even when traj-std is unavailable. These standard deviations are used only in the score curves and traj-std summaries, not in RI or RHVar.

Tables~\ref{tab:doubling_angle_traj_std}--\ref{tab:lorenz96_traj_std} report the trajectory standard deviations defined in \eqref{eq:traj_std}.

\begin{table}[!htbp]
\centering
\caption{SED traj-std for the doubling-angle experiment.}
\label{tab:doubling_angle_traj_std}
\renewcommand{\arraystretch}{1.15}
\resizebox{\textwidth}{!}{%
\begin{tabular}{lrrrrrrrr}
    \toprule
    \textbf{Metric}
    & \shortstack{\textbf{EtE}\\\textbf{+ ES}}
    & \shortstack{\textbf{CorrTerms}\\\textbf{+ ES}}
    & \shortstack{\textbf{EtE}\\\textbf{+ NL2}}
    & \shortstack{\textbf{CorrTerms}\\\textbf{+ NL2}}
    & \shortstack{\textbf{EnKF}\\$N=3000$}
    & \shortstack{\textbf{ESRF}\\$N=3000$}
    & \shortstack{\textbf{IEnKF}\\$N=1000$}
    & \shortstack{\textbf{SMF}\\$N=3000$} \\
    \midrule
    SED
    & $0.0128$
    & $0.0156$
    & $0.0105$
    & $0.0105$
    & $0.0111$
    & $0.0111$
    & $0.0144$
    & $0.0181$ \\
    \bottomrule
\end{tabular}%
}
\end{table}

\begin{table}[!htbp]
\centering
\caption{SED traj-std for the Lorenz~'63 experiment. A dash indicates that no tested $N$ has all $64$ trajectory outputs.}
\label{tab:lorenz63_traj_std}
\renewcommand{\arraystretch}{1.15}
\resizebox{\textwidth}{!}{%
\begin{tabular}{lrrrrrrrr}
    \toprule
    \textbf{Observation}
    & \shortstack{\textbf{EtE}\\\textbf{+ ES}}
    & \shortstack{\textbf{CorrTerms}\\\textbf{+ ES}}
    & \shortstack{\textbf{EtE}\\\textbf{+ NL2}}
    & \shortstack{\textbf{CorrTerms}\\\textbf{+ NL2}}
    & \textbf{EnKF}
    & \textbf{ESRF}
    & \textbf{IEnKF}
    & \textbf{SMF} \\
    \midrule
    Partial + Identity
    & $0.0478$
    & $0.0470$
    & $0.0424$
    & $0.0241$
    & $0.0425$
    & $0.0510$
    & $0.0289$
    & $0.0858$ \\
    Partial + Square
    & $0.1458$
    & $0.1201$
    & $0.0794$
    & $0.0684$
    & $0.0650$
    & $0.0685$
    & --
    & $0.0822$ \\
    Partial + Arctan
    & $0.0984$
    & $0.0802$
    & $0.0698$
    & $0.0821$
    & $0.1242$
    & $0.0998$
    & --
    & $0.1262$ \\
    \bottomrule
\end{tabular}%
}
\end{table}

\begin{table}[!htbp]
\centering
\caption{ES and VS traj-std for the Lorenz~'96 experiment. Dashes indicate that no tested $N$ has all $64$ trajectory outputs.}
\label{tab:lorenz96_traj_std}
\renewcommand{\arraystretch}{1.15}
\resizebox{\textwidth}{!}{%
\begin{tabular}{lrrrrrrrr}
    \toprule
    \textbf{Observation / metric}
    & \shortstack{\textbf{EtE}\\\textbf{+ ES}}
    & \shortstack{\textbf{CorrTerms}\\\textbf{+ ES}}
    & \shortstack{\textbf{EtE}\\\textbf{+ NL2}}
    & \shortstack{\textbf{CorrTerms}\\\textbf{+ NL2}}
    & \textbf{EnKF}
    & \textbf{LETKF}
    & \textbf{IEnKF}
    & \textbf{SMF} \\
    \midrule
    Identity / ES
    & $0.1744$
    & $0.2205$
    & $0.2501$
    & $0.2215$
    & $0.3127$
    & $0.3303$
    & --
    & $0.3349$ \\
    Identity / VS
    & $0.1321$
    & $0.1527$
    & $0.1625$
    & $0.1511$
    & $0.2319$
    & $0.2473$
    & --
    & $0.2461$ \\
    Square / ES
    & $0.1868$
    & $0.1761$
    & $0.2350$
    & $0.1665$
    & $0.2410$
    & $0.2320$
    & --
    & $0.1798$ \\
    Square / VS
    & $0.1446$
    & $0.2399$
    & $0.2054$
    & $0.3169$
    & $0.1755$
    & $0.1725$
    & --
    & $0.1071$ \\
    Arctan / ES
    & $0.1692$
    & $0.2075$
    & $0.1978$
    & $0.2079$
    & $0.2485$
    & $0.2628$
    & --
    & $0.1968$ \\
    Arctan / VS
    & $0.1312$
    & $0.1589$
    & $0.1634$
    & $0.1696$
    & $0.1944$
    & $0.2051$
    & --
    & $0.1428$ \\
    \bottomrule
\end{tabular}%
}
\end{table}

\subsection{Robustness to Initial-Ensemble Misspecification}
\label{sssec:L63_shifted_initial}

The main Lorenz~'63 experiment initializes the ensemble around the realized initial truth. To assess sensitivity to errors in this initialization, for each test trajectory we draw trajectory-specific vectors $a,b\in\bbR^{d_v}$ componentwise and initialize the ensemble as follows (with the trajectory index suppressed):
\begin{equation}
\label{eq:L63_shifted_initialization}
\begin{aligned}
&a_k\sim\operatorname{Unif}(-3,3),\qquad b_k\sim\operatorname{Unif}(1,3),\qquad k=1,\ldots,d_v,\\
&v_0^{(n)}\sim\normal\!\left(v_0^\dagger+a,\mathrm{diag}(b_1^2,\ldots,b_{d_v}^2)\right),\qquad n=1,\ldots,N.
\end{aligned}
\end{equation}
For each test trajectory, the pair $(a,b)$ is sampled once and reused across the partial identity, partial square, and partial arctan observation experiments. The compared methods follow Subsection~\ref{ssec:L63}. We use the same $64$ test trajectories of $500$ assimilation cycles and ensemble sizes $N\in\{5,10,15,20,40,60,100\}$. SED is evaluated against the $10^6$-particle BPF reference.

\begin{figure}[!htbp]
    \centering
    \includegraphics[width=\textwidth]{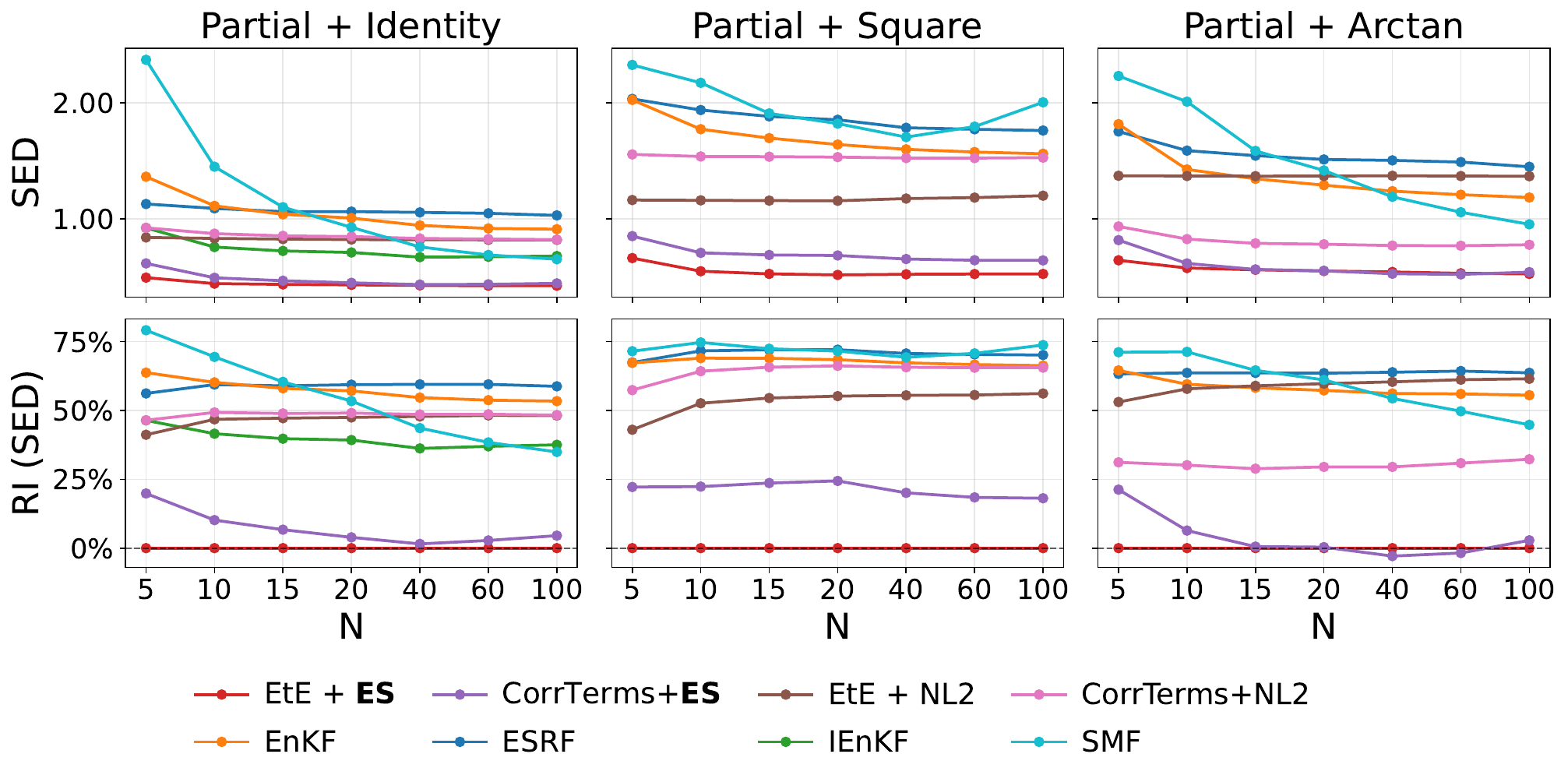}
    \caption{Lorenz~'63 SED (top row) and the RI in SED of EtE+ES over each comparison method at the same $N$ (bottom row) under initial-ensemble misspecification. Columns correspond to the partial identity, partial square, and partial arctan observations. Missing curves indicate filter divergence; IEnKF diverges under the partial square and partial arctan observations throughout the tested ensemble-size range.}
    \label{fig:L63_shifted_initial}
\end{figure}

Figure~\ref{fig:L63_shifted_initial} shows that PSEF (EtE+ES) remains the strongest overall method under the tested initial-ensemble misspecification, with relative improvements over the learning-based and benchmark alternatives similar to those obtained with truth-centered initialization in the main Lorenz~'63 experiment.

\subsection{Sensitivity to the Training Seed}
\label{ssec:training_seed_sensitivity}

We assess the sensitivity of PSEF (EtE+ES) to training randomness on the doubling-angle and Lorenz~'63 models by repeating training with seeds $\{0,1,2,42\}$. Only the training seed varies: the test trajectories are identical across runs, and the testing random seed remains fixed at $42$. All other experimental settings are unchanged. For each ensemble size $N$ and observation map, we first average SED over the fixed test trajectories for each trained model, then compute the mean and standard deviation (std) of these scores across the four training seeds. The coefficient of variation is $\mathrm{CV}=\mathrm{std}/\mathrm{mean}$. Thus, these statistics measure variation across training runs, whereas traj-std in Appendix~\ref{app:trajectory_variability} measures variation across test trajectories.

\begin{table}[!htbp]
\centering
\caption{Doubling-angle SED across training seeds.}
\label{tab:doubling_angle_training_seeds}
\begin{tabular}{rrrr}
    \toprule
    $N$ & Mean & Std & CV \\
    \midrule
    $10$ & $0.1887$ & $0.0078$ & $0.0413$ \\
    $30$ & $0.1298$ & $0.0053$ & $0.0407$ \\
    $100$ & $0.1048$ & $0.0053$ & $0.0503$ \\
    $300$ & $0.0972$ & $0.0057$ & $0.0586$ \\
    \bottomrule
\end{tabular}
\end{table}

\begin{table}[!htbp]
\centering
\caption{Lorenz~'63 SED across training seeds.}
\label{tab:lorenz63_training_seeds}
\small
\setlength{\tabcolsep}{4pt}
\begin{tabular}{r rrr rrr rrr}
    \toprule
    \multirow{2}{*}{$N$}
    & \multicolumn{3}{c}{Identity}
    & \multicolumn{3}{c}{Arctan}
    & \multicolumn{3}{c}{Square} \\
    \cmidrule(lr){2-4}\cmidrule(lr){5-7}\cmidrule(lr){8-10}
    & Mean & Std & CV
    & Mean & Std & CV
    & Mean & Std & CV \\
    \midrule
    $5$ & $0.4555$ & $0.0244$ & $0.0537$ & $0.5926$ & $0.0366$ & $0.0618$ & $0.6278$ & $0.0280$ & $0.0446$ \\
    $10$ & $0.4109$ & $0.0263$ & $0.0640$ & $0.5267$ & $0.0368$ & $0.0699$ & $0.5287$ & $0.0398$ & $0.0753$ \\
    $15$ & $0.3996$ & $0.0270$ & $0.0676$ & $0.5118$ & $0.0362$ & $0.0707$ & $0.4990$ & $0.0379$ & $0.0759$ \\
    $20$ & $0.3932$ & $0.0281$ & $0.0716$ & $0.4997$ & $0.0362$ & $0.0724$ & $0.4974$ & $0.0355$ & $0.0713$ \\
    $40$ & $0.3815$ & $0.0288$ & $0.0754$ & $0.4807$ & $0.0348$ & $0.0723$ & $0.4777$ & $0.0368$ & $0.0770$ \\
    $60$ & $0.3818$ & $0.0275$ & $0.0720$ & $0.4820$ & $0.0292$ & $0.0607$ & $0.4780$ & $0.0317$ & $0.0663$ \\
    $100$ & $0.3804$ & $0.0268$ & $0.0703$ & $0.4861$ & $0.0369$ & $0.0759$ & $0.4812$ & $0.0345$ & $0.0716$ \\
    \bottomrule
\end{tabular}
\end{table}

Tables~\ref{tab:doubling_angle_training_seeds} and~\ref{tab:lorenz63_training_seeds} show that different training seeds produce measurable but modest differences in SED. For the doubling-angle model, CV ranges from $0.0407$ to $0.0586$; for Lorenz~'63, it ranges from $0.0446$ to $0.0770$ across all three observation maps. The largest CV across both models is therefore $0.0770$, with the standard deviation remaining below $8\%$ of the mean in every tested configuration. The mean SED decreases with ensemble size for the doubling-angle model and generally decreases before leveling off for Lorenz~'63. These results indicate limited sensitivity to the training seed in the tested settings under a fixed evaluation protocol.